\documentclass[11pt, onecolumn]{IEEEtran}

\makeatletter
\def\ps@headings{%
\def\@oddhead{\mbox{}\scriptsize\rightmark \hfil \thepage}%
\def\@evenhead{\scriptsize\thepage \hfil \leftmark\mbox{}}%
\def\@oddfoot{}%
\def\@evenfoot{}}
\makeatother
\usepackage{multirow}
\usepackage{xcolor}
\usepackage{graphicx}
\usepackage{amsmath,amsthm,amssymb,amsfonts}
\usepackage{hyperref}
\usepackage{url}
\usepackage{cases}
\usepackage{bbm}
\usepackage{bm}

\usepackage{algorithmic}
\usepackage{algorithm}

\usepackage[square, comma, sort&compress, numbers]{natbib}

\newtheorem{theorem}{Theorem}
 \newtheorem{definition}{Definition}
 \newtheorem{lemma}{Lemma}
 \newtheorem{remark}{Remark}
 \newtheorem{collary}{Corollary}

\DeclareMathOperator*{\argmin}{arg\,min}

\renewcommand\appendixname{Appendix}

\newcommand{\mc}[1]{\mathcal{#1}}
\newcommand{\mb}[1]{\mathbb{#1}}
\newcommand{\V}[1]{{\bm{\mathbf{\MakeLowercase{#1}}}}} 

\usepackage{setspace}

\begin{document}

\title{Low-tubal-rank Tensor Completion using Alternating Minimization}

\author{Xiao-Yang Liu, Shuchin Aeron, Vaneet Aggarwal, and Xiaodong Wang
\IEEEcompsocitemizethanks{\IEEEcompsocthanksitem  X.-Y.~Liu is with the Department of Computer Science and Engineering, Shanghai Jiao Tong University, email: yanglet@sjtu.edu.cn. He is also affiliated with the Department of Electrical Engineering, Columbia University.  
  S.~Aeron is with the Department of Electrical and Computer Engineering, Tufts University, email: shuchin@ece.tufts.edu. V.~Aggarwal is with the School of Industrial Engineering, Purdue University, email: vaneet@purdue.edu. X.~Wang are with the Department of Electrical Engineering, Columbia University, email: xw2008@columbia.edu.
  
  This paper was presented in part at the SPIE Conference on Defense and Security, Apr 2016.
}
}

\date{}
\maketitle

\IEEEdisplaynotcompsoctitleabstractindextext
\IEEEpeerreviewmaketitle


\begin{abstract}

  The low-tubal-rank tensor model has been recently proposed for real-world multidimensional data. In this paper, we study the low-tubal-rank tensor completion problem, i.e., to recover a third-order tensor by observing a subset of its elements selected uniformly at random. We propose a fast iterative algorithm, called {\em Tubal-Alt-Min}, that is inspired by a similar approach for low-rank matrix completion. The unknown low-tubal-rank tensor is represented as the product of two much smaller tensors with the low-tubal-rank property being automatically incorporated, and Tubal-Alt-Min alternates between estimating those two tensors using tensor least squares minimization. First, we note that tensor least squares minimization is different from its matrix counterpart and nontrivial as the circular convolution operator of the low-tubal-rank tensor model is intertwined with the sub-sampling operator.
  Second, the theoretical performance guarantee is challenging since Tubal-Alt-Min is iterative and nonconvex in nature. We prove that 1) Tubal-Alt-Min guarantees exponential convergence to the global optima, and 2) for an $n \times n \times k$ tensor with tubal-rank $r \ll n$, the required sampling complexity is $O(nr^2k \log^3 n)$ and the computational complexity is $O(n^2rk^2 \log^2 n)$. Third, on both synthetic data and real-world video data, evaluation results show that compared with tensor-nuclear norm minimization (TNN-ADMM), Tubal-Alt-Min improves the recovery error dramatically (by orders of magnitude). It is estimated that Tubal-Alt-Min converges at an exponential rate $10^{-0.4423 \text{Iter}}$ where $\text{Iter}$ denotes the number of iterations, which is much faster than TNN-ADMM's $10^{-0.0332 \text{Iter}}$, and the running time can be accelerated by more than $5$ times for a $200 \times 200 \times 20$ tensor.

\end{abstract}


\newpage
\section{Introduction}

   The big data era calls for efficient algorithms to analyze the enormous amount of data generated by high-resolution sensors, mobile devices, online merchants, and social networks \cite{Baraniuk}. Such real-world data/signals\footnote{In the following, we use the words ``signal" and ``data" interchangeably.} are naturally represented as multidimensional arrays \cite{Cichochi2015SPM}, namely, vectors, matrices, high-order tensors or tensor networks. Signal recovery from partial measurements \cite{Terence} by exploiting the redundancy property modeled as sparse or low-rank has received wide attention in various research and engineering communities. We are interested in fast algorithms for multilinear data completion where the measurement procedure is modelled as a simple down-sampling operation. Exemplar applications include MRI imaging \cite{Terence}, signal processing \cite{Cichochi2015SPM}, big data analysis with missing entries \cite{Rubin2014}, data privacy \cite{Liu2015ICDCS}, network engineering \cite{Liu2015TMC,Zhang2009SIGCOMM,Qiu2010MobiCom,Qiu2014MobiCom}, Internet of Things \cite{Liu2013INFOCOM,Liu2014TPDS}, machine learning \cite{ML2014}, computer vision \cite{Chen2004PAMI,Ye2013PAMI}, recommender system \cite{Koren2009}, and system identification \cite{Liu2009}.

   Such diverse applications motivate and justify the developments of compressive sensing (vector case) \cite{Terence,Tao2006ToIT2}, matrix completion and matrix sensing \cite{Candy2009,Tao2010ToIT}, and higher-order tensor completion \cite{Kolda2009,Kilmer2011,Kilmer2013}. Compressive sensing \cite{Terence,Tao2006ToIT2} advocated relaxing the original NP-hard problem to its convex surrogate, i.e., replacing the $\ell_0$-norm with $\ell_1$-norm. Similarly, researchers introduced nuclear norm \cite{Fazel2002} and tensor-nuclear norm \cite{Shuchin2015} to approximate the combinatorial rank function for the low-rank matrix and tensor completion problem\footnote{A vector is a first-order tensor while a matrix is a second-order tensor.}, respectively. Those two relaxation approaches achieve optimal results with high computational cost, mainly because of the time-consuming SVD (singular value decomposition) or tensor-SVD operations \cite{Shuchin2015,Shuchin2014CVPR}.

   Alternating minimization approaches have been proposed for the matrix completion problem \cite{Jain2013STOC,Hardt2014COLT,Hard2014FOCS}. First, it is both computation- and storage-efficient in implementation. The unknown low-rank matrix $M \in \mathbb{R}^{m \times n}$ is factorized into two much smaller matrices $X$ and $Y$ of size $m \times r$ and $n \times r$, respectively, i.e., $M = XY^{\dagger}$, and rank $r \ll \min(m,n)$ implying $(m + n)r \ll mn$, thus requiring much less computation and memory to optimize. Secondly, this factorization approach enables easier modeling. Besides the low-rank property, this factorization approach allows one to impose extra constraints on the target matrix $M$ or factors $(X,Y)$. For example, Sparse PCA \cite{sparsePCA2006} seeks a low-rank $M$ that is the product of {\em sparse} $X$ and $Y$. Thirdly, it converges to the global optima at a geometric rate, and such theoretic results become available only very recently \cite{Jain2013STOC,Hardt2014COLT,Hard2014FOCS}.

   However, extending existing alternating minimization algorithms (originally designed for the matrix case and enjoyed empirical successes) \cite{Zhang2009SIGCOMM,Liu2013INFOCOM,Jain2013STOC,Hardt2014COLT,Hard2014FOCS} to higher-order tensors is impeded by three major challenges: 1) there exist different definitions for tensor operators that lead to different low-rank tensor models, i.e., the CP-rank tensor \cite{Kolda2009}, the Tuker-rank tensor \cite{Kolda2009} and the low-tubal-rank tensor \cite{Kilmer2011,Kilmer2013}; 2) existing approaches would be rather inefficient for higher-order tensors due to the curse of dimensionality; and 3) those algorithms do not guarantee good theoretical performance.

   In this paper, we address these challenges for the third-order tensor completion problem. More specifically, we are interested in the low-tubal-rank tensor model that shares a similar algebraic structure with the low-rank matrix model. Our goal is to design a fast algorithm under the alternating minimization framework, and theoretically assess its performance. We believe this approach would be a breakthrough point for higher-order tensors due to the following three perspectives:
   \begin{itemize}
     \item The low-tubal-rank tensor model \cite{Kilmer2011,Kilmer2013} is recently proposed for modeling multilinear real-world data, such as WiFi fingerprints \cite{Liu2015TMC}, images \cite{Shuchin2014}, videos \cite{Shuchin2014CVPR}, seismic data \cite{Shuchin2015Seimic}, and machine learning \cite{Shuchin2014}. There is a ``spatial-shifting" property in those data, and we believe it is ubiquitous in real-world data arrays. The low-tubal-rank tensor model is ideal for capturing such characteristics.
     \item Although being iterative and nonconvex in nature, the alternating minimization approach can be much faster than convex relaxations of the tensor completion problem. The potential computational efficiency comes from the fact that it automatically incorporates the low-rank property, resulting in massive dimension reduction. Note that computational efficiency is critical for processing big data.
     \item It has been proved that alternating minimization achieves the global optima at an exponential convergence rate for matrix completion \cite{Jain2013STOC,Hardt2014COLT,Hard2014FOCS}. According to similar algebra laws, such performance guarantees are expected to hold for higher-order tensors.
   \end{itemize}

    First, we propose a fast alternating minimization algorithm, {\em Tubal-Alt-Min}, for the low-tubal-rank tensor completion problem. A key novelty is solving a least squares minimization for tensors by defining a new set of operators, which can be of independent interest. Tensor least squares minimization is different from the standard least squares minimization because the circular convolution operator of the low-tubal-rank tensor model is intertwined with the sub-sampling operator. Therefore, the tensor completion problem is essentially different from matrix completion, implying that existing alternating minimization algorithms \cite{Jain2013STOC,Hardt2014COLT,Hard2014FOCS} cannot be extended straightforwardly to our problem.

    Secondly, the proposed alternating minimization-based approach can be much faster than the tensor-nuclear norm minimization with alternating direction method of multipliers (TNN-ADMM) \cite{Shuchin2015,Shuchin2014CVPR}. We prove that 1) the proposed algorithm guarantees convergence to the global optima at an exponential rate, which is much faster than TNN-ADMM; and 2) for a tensor of size $n \times n \times k$ and tubal-rank $r \ll n$, the required sampling complexity is $O(nr^2k \log^3 n)$ and the computational complexity is $O(n^2rk^2 \log^2 n)$. Please note that there is no constraint on the size of the third-dimension. The proof is based on exploiting an injective mapping between the {\em circulant algebra} and the {\em circular matrix space}.

    Thirdly, we evaluate Tubal-Alt-Min on both synthetic data and real-world video data. The performances are measured in terms of recovery error, convergence rate, and running time. Compared with the convex relaxation-based algorithm TNN-ADMM \cite{Shuchin2015,Shuchin2014CVPR}, Tubal-Alt-Min improves the recovery error by one order of magnitude at sampling rate $50\%$ for synthetic data, and three orders of magnitude for the video data. Tubal-Alt-Min converges to the global optima within $O(\log n / \epsilon)$ iterations and the convergence rate is estimated to be $10^{-0.4423 \text{Iter}}$ where $\text{Iter}$ denotes the number of iterations, which is much faster than TNN-ADMM's rate of $10^{-0.0332 \text{Iter}}$. The running time can be accelerated by more than $5$ times for a $200 \times 200 \times 20$ tensor.

    The remainder of the paper is organized as follows. In Section II, we present the low-tubal-rank tensor model and some preliminaries of the circulant algebra. Section III describes the low-tubal-rank tensor completion problem and the proposed {\em Tubal-Alt-Min} algorithm, including a novel routine to solve the key subproblem: tensor least squares minimization. Section IV provides theoretical performance guarantees of the Tubal-Alt-Min algorithm, while detailed proofs are given in the Appendix. In Section V, we evaluate the proposed scheme on both synthetic data and real-world video data. The conclusions and future works are given in Section VI.

\section{Notations and Preliminaries}

   We begin by first outlining the notations, the algebraic models and some useful results for third-order tensors \cite{Braman2010,Kilmer2011,Kilmer2013,Gleich2013}. Two lemmas (Lemma 2 and 4) in this section are new results. Section \ref{model_A} presents the low-tubal-rank tensor model, used for problem formulation and algorithm design in Section \ref{sect:problem_statement}. Section \ref{sect:circular_algebra} presents preliminaries of the circulant algebra, used for performance analysis in Section \ref{analysis} and the Appendix.

   For computational and sampling complexity, the following standard asymptotic notations are used throughout this paper. Given two non-negative functions $f(n)$ and $g(n)$:
   \begin{itemize}
   \item $f(n) = o(g(n))$ means $\lim_{n \rightarrow \infty} \frac{f(n)}{g(n)} =0$, \item $f(n) = O(g(n))$ means $\lim_{n \rightarrow \infty} \frac{f(n)}{g(n)} < \infty$,
   \item $f(n) = \omega(g(n))$ means $\lim_{n \rightarrow \infty} \frac{f(n)}{g(n)} = \infty$,
   \item $f(n) = \Omega(g(n))$\footnote{Note that $\Theta$ and $\Omega$ are re-used later, whose  meanings will be clear from the context.} means $\lim_{n \rightarrow \infty} \frac{f(n)}{g(n)} > 0$,
   \item $f(n) = \Theta(g(n))$ means $f(n) = O(g(n))$ and $g(n) = O(f(n))$.
   \end{itemize}

\subsection{Low-tubal-rank Tensor Model}
\label{model_A}

   Throughout the paper, we will focus on real valued third-order tensors in the space $\mathbb{R}^{m \times n \times k}$. We use $m,~n,~k,~r$ for tensor dimensions, $x,~y \in \mathbb{R}$ for scalar variables, $\textbf{x},~\textbf{y} \in \mathbb{R}^{n}$ for vectors, and $X,~Y \in \mathbb{R}^{m \times n}$ for matrices. Tensors are denoted by calligraphic letters and their corresponding circular matrices (defined in Section \ref{sect:circular_algebra}) are tagged with the superscript $c$, i.e., $\mathcal{T}\in \mathbb{R}^{m \times n \times k},~\mathcal{X} \in \mathbb{R}^{m \times r \times k},~\mathcal{Y} \in \mathbb{R}^{n \times r \times k}$ and $T^{c}\in \mathbb{R}^{mk \times nk},~X^{c} \in \mathbb{R}^{mk \times rk },~Y^{c} \in \mathbb{R}^{nk \times rk}$.

   Let $X^\dag$ denote the transpose of matrix $X$. We use $i,j,\kappa$ to index the first, second and third dimension of a tensor, and $s,t$ for temporary indexing. $[n]$ denotes the set $\{1,2,...,n\}$. Usually, $i \in [m],j \in [n],\kappa \in [k]$ unless otherwise specified. For tensor $\mathcal{T} \in \mathbb{R}^{m \times n \times k}$, the $(i,j,\kappa)$-th entry is $\mathcal{T}(i,j,\kappa)$, or concisely represented as $\mathcal{T}_{ij\kappa}$.
   The $\ell_2$-norm of a vector is defined as $||\textbf{x}||_2 = \sqrt{\sum_{i=1} \textbf{x}_i^2}$, while the Frobenius norm of a matrix $X$ is $||X||_F = \sqrt{\sum_{i=1}^{m} \sum_{j=1}^{n} X_{ij}^2}$ and of a tensor is $||\mathcal{T}||_F = \sqrt{\sum_{i=1}^{m} \sum_{j=1}^{n} \sum_{\kappa=1}^{k} \mathcal{T}_{ij\kappa}^2}$.

   \textbf{Tubes/fibers, and slices of a tensor}: A {\em tube} (also called a fiber) is a 1-D section defined by fixing all indices but one, while a {\em slice} is a 2-D section defined by fixing all but two indices. We use $\mathcal{T}(:,j,\kappa), ~\mathcal{T}(i,:,\kappa), ~\mathcal{T}(i,j,:)$ to denote the mode-$1$, mode-$2$, mode-$3$ tubes, which are vectors, and $\mathcal{T}(:,:,\kappa),~\mathcal{T}(:,j,:),~\mathcal{T}(i,:,:)$ to denote the frontal, lateral, horizontal slices, which are matrices. For easy representation sometimes, we denote $\mathcal{T}^{(\kappa)} = \mathcal{T}(:,:,\kappa)$.

   \textbf{Tensor transpose and frequency domain representation}:
   $\mathcal{T}^{\dag} \in \mathbb{R}^{n \times m \times k}$ is obtained by transposing each of the frontal slices and then reversing the order of transposed frontal slices $2$ through $k$, i.e., for $2 \leq \kappa \leq k$, $\mathcal{T}^{\dag}(:,:,\kappa) =(\mathcal{T}(:,:,k+2-\kappa))^{\dag}$ (the transpose of matrix $\mathcal{T}(:,:,k+2-\kappa)$). For reasons to become clear soon, we define a tensor $\widetilde{\mathcal{T}}$, which is the representation in the frequency domain and is obtained by taking the Fourier transform along the third mode of $\mathcal{T}$, i.e., $\widetilde{\mathcal{T}}(i,j,:) = \text{fft}(\mathcal{T}(i,j,:))$. In MATLAB notation, $\widetilde{\mathcal{T}} = \text{fft}(\mathcal{T},[~],3)$, and one can also compute $\mathcal{T}$ from $\widetilde{\mathcal{T}}$ via $\mathcal{T} = \text{ifft}(\widetilde{\mathcal{T}},[~],3)$.

   We now define the linear algebraic development \cite{Kilmer2013} for the low-tubal-rank tensor model. It rests on defining a tensor-tensor product between two 3-D tensors, referred to as the t-product as defined below. For two tubes (vectors) of the same size, i.e., $\textbf{a},\textbf{~b} \in \mathbb{R}^{k}$, let $\textbf{a}*\textbf{b}$ denote the \emph{circular convolution} between these two tubes, which preserves the size. Next, we give the definition of tensor product and some related definitions.

  \begin{definition}\label{def:tensor_product}
  \cite{Kilmer2011,Kilmer2013}
  \textbf{t-product}. The tensor-product $\mathcal{C} = \mathcal{A} \ast \mathcal{B}$ of $\mathcal{A} \in \mathbb{R}^{n_1 \times n_2 \times k}$ and $\mathcal{B} \in \mathbb{R}^{n_2 \times n_3 \times k}$ is a tensor of size $n_1 \times n_3 \times k$,
  $\mathcal{C}(i,j,:) = \sum\limits_{s=1}^{n_2} \mathcal{A}(i,s,:) \ast \mathcal{B}(s,j,:)$, for $i \in [n_1]$ and $j \in [n_3]$.
  \end{definition}

  A 3-D tensor of size $n_1 \times n_2 \times k$ can be viewed as an $n_1 \times n_2$ matrix of tubes which lie in the third-dimension. So the t-product of two tensors can be regarded as a matrix-matrix multiplication, except that the operation between scalars is replaced by circular convolution between two tubes. Therefore, the two operators element-wise addition and the t-product, and the space $\mathbb{R}^{n \times n \times k}$ together define an {\em Abelian group} \cite{Braman2010}.

  \begin{definition}\cite{Kilmer2011,Kilmer2013}
  \textbf{Identity tensor}.
  The identity tensor $\mathcal{I} \in \mathbb{R}^{n \times n \times k}$ is a tensor whose first frontal slice $\mathcal{I}(:,:,1)$ is the $n \times n$ identity matrix and all other frontal slices $\mathcal{I}^{(i)},~(i=2,...,k)$ are zero matrices.
  \end{definition}

  \begin{definition}\cite{Kilmer2011,Kilmer2013}
  \textbf{Orthogonal tensor}. A tensor $\mathcal{Q} \in \mathbb{R}^{n \times n \times k}$ is orthogonal if it satisfies
  $\mathcal{Q}^{\dag} \ast \mathcal{Q} =  \mathcal{Q} \ast \mathcal{Q}^{\dag} = \mathcal{I}$.
  \end{definition}

  \begin{definition}\label{inverse}\cite{Kilmer2011,Kilmer2013}
  \textbf{Inverse}. The inverse of a tensor $\mathcal{T} \in \mathbb{R}^{n \times n \times k}$ is written as $\mathcal{T}^{-1} \in \mathbb{R}^{n \times n \times k}$ and satisfies
  $\mathcal{T}^{-1} \ast \mathcal{T} = \mathcal{T} \ast \mathcal{T}^{-1} = \mathcal{I}$.
  \end{definition}

  \begin{definition}\label{def:block-diagonal}\cite{Kilmer2011,Kilmer2013}
  \textbf{Block diagonal form of third-order tensor}. Let $\overline{\mathcal{A}}$ denote the block-diagonal matrix representation of the tensor $\mathcal{A}$ in the Fourier domain, i.e.,
  \begin{equation}
  \overline{\mathcal{A}} \triangleq  blkdiag(\widetilde{\mathcal{A}}) \triangleq  \left[
  \begin{array}{cccc}
  \widetilde{\mathcal{A}}^{(1)} &  & &\\
  & \widetilde{\mathcal{A}}^{(2)} & &\\
  & & ... & \\
  & & &\widetilde{\mathcal{A}}^{(k)}
  \end{array}\right] \in \mathbb{C}^{mk \times nk},
  \end{equation}
  where $\mathbb{C}$ denotes the set of complex numbers. It is easy to verify that  $\overline{\mathcal{A}^{\dagger}} = \overline{\mathcal{A}}^{\dagger}$.
  \end{definition}

  \begin{remark}\label{remark:computing_tensor_product}\cite{Kilmer2011,Kilmer2013}
  The following fact will be used throughout the paper for calculating tensor products and also tensor inverse. For tensors $\mathcal{A} \in \mathbb{R}^{n_1 \times n_2 \times k}$ and $\mathcal{B} \in \mathbb{R}^{n_2 \times n_3 \times k}$, we have
  \begin{equation}
  \mathcal{A} * \mathcal{B} = \mathcal{C} \Longleftrightarrow \overline{\mathcal{A}}~  \overline{\mathcal{B}} = \overline{\mathcal{C}}.
  \end{equation}
  \end{remark}

  \begin{definition}\cite{Kilmer2011,Kilmer2013}
  \textbf{f-diagonal tensor}.
  A tensor is called f-diagonal if each frontal slice of the tensor is a diagonal matrix, i.e., $\Theta(i,j,\kappa)=0$ for $i \neq j, \forall \kappa$.
  \end{definition}

  Using the above definitions, one can obtain the t-SVD \cite{Kilmer2011,Kilmer2013} (tensor singular value decomposition) for compressing or denoising third-order data. 

  \begin{definition}\label{tsvd}\cite{Kilmer2011,Kilmer2013}
  \textbf{t-SVD}. The t-SVD of $\mathcal{T} \in \mathbb{R}^{m \times n \times k}$ is given by
   $\mathcal{T} = \mathcal{U} \ast \Theta \ast \mathcal{V}^{\dag}$,
   where $\mathcal{U}$ and $\mathcal{V}$ are orthogonal tensors of sizes $m \times m \times k$ and $n \times n \times k$, respectively, $\Theta$ is a f-diagonal tensor of size $m \times n \times k$ and its tubes are called the eigentubes of $\mathcal{T}$. An algorithm for computing the t-SVD is outlined in Alg. \ref{alg:tSVD}.
   \end{definition}

  \begin{figure}[t]\centering
  \includegraphics[width=0.495\textwidth]{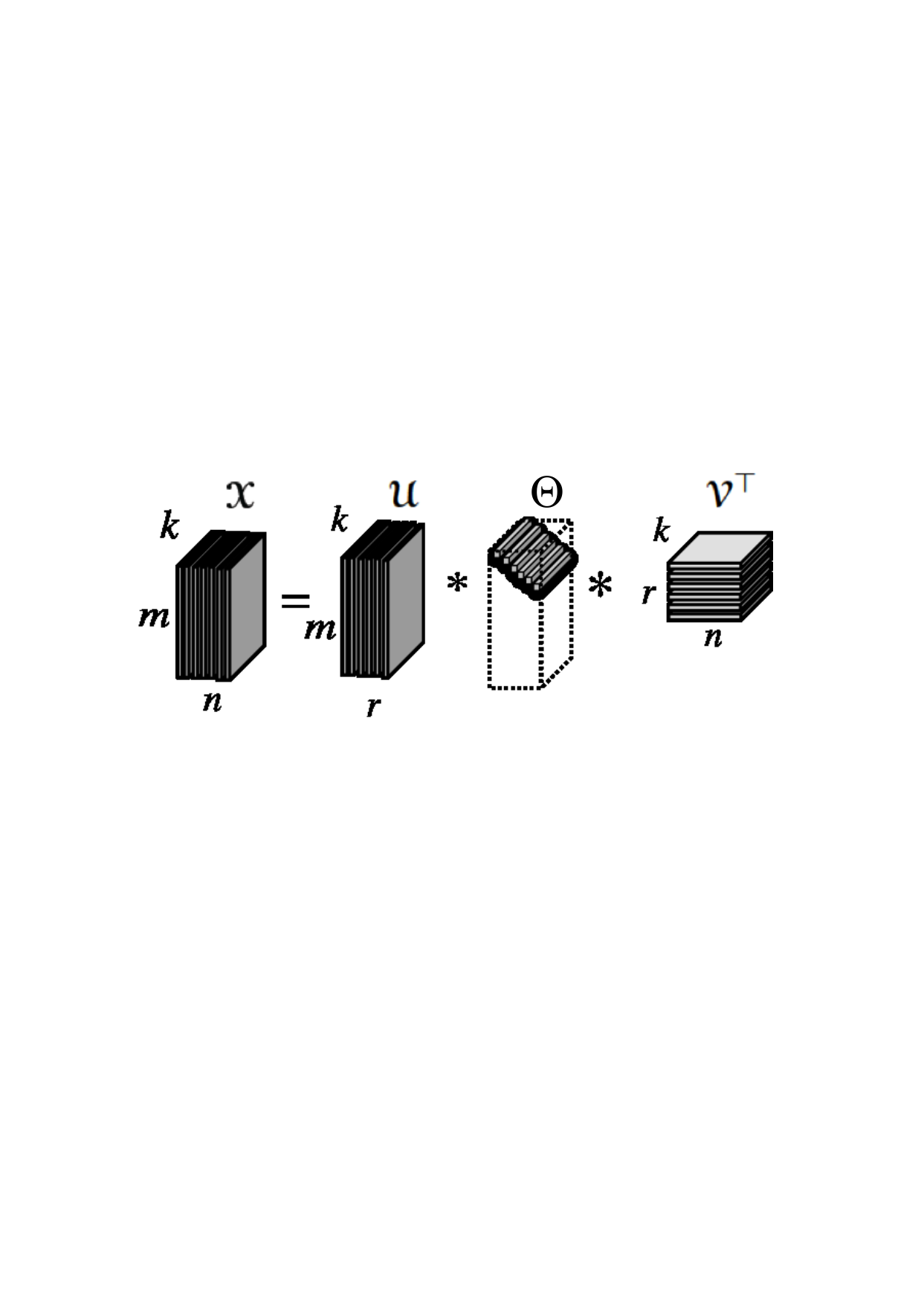}
   \caption{The (reduced) t-SVD of an $m \times n \times k$ tensor of tubal-rank $r$.}
   \label{fig:tsvd}
  \end{figure}

   \begin{figure}[h]
 	\centering
	\vspace{-5mm}
    \begin{minipage}{\textwidth}
  \begin{algorithm}[H]
  \caption{t-SVD \cite{Kilmer2011,Kilmer2013}}
  \begin{algorithmic}
   \label{alg:tSVD}
  \STATE \textbf{Input: } $\mathcal{X} \in \mathbb{C}^{n_1 \times n_2  \times n_3}$
  	\STATE ${\widetilde{\mathcal{X}}} \leftarrow {\tt fft}(\mathcal{X},[\hspace{1mm}],3)$; ~\%Take DFT along the 3rd dimension
  \FOR{$i = 1 \hspace{2mm} \rm{to} \hspace{2mm} n_3$}
  	\STATE $ [\hat{U}, \hat{S}, \hat{V}] = {\tt SVD} (\widetilde{\mathcal{X}}(:,:,i))$;
  	\STATE $ {\widetilde{\mathcal{U}}}^{(i)} = \hat{U};  {\widetilde{\mathcal{S}}}^{(i)} = \hat{S}$; $ \widetilde{\mathcal{V}}^{(i)}= \hat{V}; $
  \ENDFOR

  \STATE $\mathcal{U} \leftarrow {\tt ifft}(\widetilde{\mathcal{U}},[\hspace{1mm}],3);$ $\mathcal{S} \leftarrow {\tt ifft}(\widetilde{\mathcal{S}},[\hspace{1mm}],3);$ $\mathcal{V} \leftarrow {\tt ifft}(\widetilde{\mathcal{V}},[\hspace{1mm}],3)$.
  \end{algorithmic}
\end{algorithm}
\end{minipage}
\vspace{-3mm}
\end{figure}

   \begin{definition}\label{tubal_rank}\cite{Kilmer2011,Kilmer2013}
   \textbf{Tensor tubal-rank}. The tensor tubal-rank of a third-order tensor $\mathcal{T}$ is the number of non-zero tubes of $\Theta$ in the t-SVD, denoted as $r$.
   \end{definition}

   \begin{remark}\label{reduced_t_svd}
   Suppose $\mathcal{T}$ has tubal-rank $r$, then the reduced t-SVD of $\mathcal{T}$ is given by $\mathcal{T} = \mathcal{U} \ast \Theta \ast \mathcal{V}^{\dagger}$,
   where $\mathcal{U} \in \mathbb{R}^{m \times r \times k}$ and $\mathcal{V} \in \mathbb{R}^{n \times r \times k}$ satisfying $\mathcal{U}^\dagger * \mathcal{U} = \mathcal{I}$, $\mathcal{V}^\dagger * \mathcal{V} = \mathcal{I}$, and $\Theta$ is a f-diagonal tensor of size $r \times r \times k$, as illustrated in Fig. \ref{fig:tsvd}. This reduced version of t-SVD will be used throughout the paper unless otherwise noted.
   \end{remark}

   \begin{lemma}\label{best_r_rank}\cite{Kilmer2011,Kilmer2013}
   \textbf{Best rank-$r$ approximation}.
   Let the t-SVD of $\mathcal{T} \in \mathbb{R}^{m \times n \times k}$ be $\mathcal{T} = \mathcal{U} \ast \Theta \ast \mathcal{V}^{\dag}$. For a given positive integer $r$, define $\mathcal{T}_r = \sum_{s=1}^{r} \mathcal{U}(:,s,:) \ast \Theta(s,s,:) \ast \mathcal{V}^{\dag}(:,s,:)$. Then
   $\mathcal{T}_r = \argmin\limits_{\overline{\mathcal{T}} \in \mathbb{T}}||\mathcal{T} - \overline{\mathcal{T}}||_F$,
   where $\mathbb{T} = \{\mathcal{X} \ast \mathcal{Y}^{\dag} | \mathcal{X} \in \mathbb{R}^{m \times r \times k}, \mathcal{Y} \in \mathbb{R}^{n \times
    r \times k} \}$.
   \end{lemma}

   Note that $\Theta$ in t-SVD is organized in a decreasing order, i.e., $||\Theta(1,1,:)||_2 \geq ||\Theta(2,2,:)||_2 \geq ... $, which is implicitly defined in \cite{Kilmer2013}. Therefore, the best rank-$r$ approximation of tensors is similar to PCA (principal component analysis) of matrices.

  We next define the {\em incoherence} of third-order tensors  which is a condition to guarantee unique solutions. The concept of incoherence is first introduced by \cite{Candy2009,Tao2010ToIT} for matrices and is a standard assumption for low-rank matrix/tensor completion problems.

  \begin{definition}\cite{Shuchin2015}
  \textbf{Tensor Incoherence Condition}.
  Given the t-SVD of a tensor $\mc{T} = \mc{U} *\Theta *\mc{V}^{\dagger}$ with tubal-rank $r$, $\mc{T}$ is said to satisfy the tensor incoherent condition, if there exists $\mu_0 > 0$ such that for $\kappa \in [k]$.
  \begin{equation}\label{tensor_incoherency}
  \begin{aligned}
  \text{(Tensor-column incoherence)}~~&\mu(\mathcal{U}) \triangleq \frac{m}{r}\max_{i \in [m]} \left\| \mc{U}^\dagger * \dot{e}_i \right\|_F^2 \leq \mu_0,\\
  \text{(Tensor-row incoherence)}~~&\mu(\mathcal{V}) \triangleq \frac{n}{r} ~\max_{j \in [n]} \left\| \mc{V}^\dagger * \dot{e}_j \right\|_F^2 \leq \mu_0,\\
  \end{aligned}
  \end{equation}
  where $\dot{e}_i$ is the $m \times 1 \times k$ column basis with $\dot{e}_{i11} = 1$  and $\dot{e}_j$ is the $n \times 1 \times k$ column basis with $\mathbf{e}_{j11}=1$.
  \end{definition}

  \begin{remark}
  The smallest $\mu_0$ is $1$ which is achieved by the case when each tensor-column subspace $\mathcal{U}(:,i,:)$ ($i \in [m]$) has elements with magnitude $\frac{1}{\sqrt{mk}}$, or each tensor-column subspace $\mathcal{U}(:,j,:)$ ($j \in [n]$) has elements with magnitude $\frac{1}{\sqrt{nk}}$. The largest possible value of $\mu_0$ is $\min (m,n)/r$ when one of the tensor columns of $\mathcal{U}$ is equal to the standard tensor column basis $\dot{e}_i$. With low $\mu_0$, each element of $\mathcal{T}$ is supposed to play a similar role in recovering  $\mathcal{T}$.
  \end{remark}

\subsection{Circulant Algebra}\label{sect:circular_algebra}

  The circulant algebra is recently introduced to generalize matrix algebra to the third-order tensor case. We borrow some notations and operations from  \cite{Braman2010,Gleich2013}, meanwhile we propose several new definitions to facilitate our analysis in the Appendix.

  Throughout the paper, circulants are denoted by underlined letters. We define {\em tubal scalar}, {\em tubal vector}, and {\em tubal matrix} in the following. Note that they are one-dimension higher than their counterparts in traditional linear algebra. In circulant algebra, a tubal scalar is a vector of length $k$. $\mathbb{K}$ denotes the space of length-$k$ {\em tubal scalars}, $\mathbb{K}^n$ denotes the space of {\em tubal vectors} where each element is a tubal scalar in $\mathbb{K}$, and $\mathbb{K}^{m \times n}$ denotes the space of $m \times n$ {\em tubal matrices} where each element is a tubal scalar in $\mathbb{K}$. We use $\underline{\alpha}, \underline{\beta} \in \mathbb{K}$ for tubal scalars, $\underline{\textbf{x}}, \underline{\textbf{y}} \in \mathbb{K}^n$ for tubal vectors, and $\underline{A},~\underline{B} \in \mathbb{K}^{m \times n}$ for tubal matrices. Their corresponding circular matrices are tagged with the superscript $c$, i.e., $\underline{\alpha}^c, \underline{\beta}^c, \underline{\textbf{x}}^c, \underline{A}^c$.

  D.~Gleich, {\em et al} \cite{Gleich2013} introduced the operator $\text{circ}(\cdot)$ to map circulants to their corresponding circular matrix representations. For tubal scalar $\alpha = \{a_1, a_2, ..., a_k\} \in \mathbb{K}$, tubal vector $\underline{\textbf{x}} \in \mathbb{K}^n$, and tubal matrix $\underline{A} \in \mathbb{K}^{m \times n}$, we use the notation $\leftrightarrow$ to denote this mapping as follows:
   \begin{equation}\label{eq:circ_operator}
   \begin{split}
   \underline{\alpha} ~~~~\leftrightarrow~~~~ \underline{\alpha}^c &= \text{circ}(\underline{\alpha}) =\left[
    \begin{array}{cccc}
    \alpha_1 & \alpha_k & ... & \alpha_2\\
    \alpha_2 & \alpha_1 & ... & ...\\
    ... & ... & ... & \alpha_k\\
    \alpha_k & \alpha_{k-1} & ... & \alpha_1\\
    \end{array}
    \right],~~~~~~
    \underline{\textbf{x}} ~~~~\leftrightarrow~~~~\underline{\textbf{x}}^c = \text{circ}(\underline{\textbf{x}}) = \left[
    \begin{array}{c}
    \text{circ}(\underline{x}_1)\\
    \vdots \\
    \text{circ}(\underline{x}_n)\\
    \end{array}
    \right], \\
    \underline{A} ~~~~\leftrightarrow~~~~ \underline{A}^c &= \text{circ}(\underline{A}) = \left[
    \begin{array}{ccc}
    \text{circ}(\underline{A}_{1,1}) & ... & \text{circ}(\underline{A}_{1,n})\\
    \vdots & \vdots & \vdots \\
    \text{circ}(\underline{A}_{m,1}) & ... & \text{circ}(\underline{A}_{m,n})\\
    \end{array}
    \right]. \\
   \end{split}
   \end{equation}

   \begin{lemma}\cite{Braman2010}
   $(\mathbb{K}^n,\pm,*)$ is a commutative ring with the multiplicative identity $\underline{1} = \{1~0~...~0\}$, where $\pm$ and $*$ denotes the addition/subtraction and circular convolution. We have:
   \begin{equation}
   \begin{split}
   \underline{\alpha} + \underline{\beta} ~~~~&\leftrightarrow~~~~\underline{\alpha}^c + \underline{\beta}^c = \text{circ}(\underline{\alpha}) + \text{circ}(\underline{\beta}), \\
   \underline{\alpha} * \underline{\beta} ~~~~&\leftrightarrow~~~~\underline{\alpha}^c \underline{\beta}^c = \text{circ}(\underline{\alpha})~  \text{circ}(\underline{\beta}), \\
   \underline{\textbf{x}} * \underline{\alpha} ~~~~&\leftrightarrow~~~~\underline{\textbf{x}}^c  \underline{\alpha}^c = \text{circ}(\underline{\textbf{x}})~  \text{circ}(\underline{\alpha}), \\
   \underline{A} * \underline{\textbf{x}}
   ~~~~&\leftrightarrow~~~~\underline{A}^c  \underline{\textbf{x}}^c = \text{circ}(\underline{A})~  \text{circ}(\underline{\textbf{x}}).
   \end{split}
   \end{equation}
   \end{lemma}


  Throughout this paper, we view a tensor in the space $\mathbb{R}^{m \times n \times k}$ as a tubal matrix in the space $\mathbb{K}^{m \times n}$. The tensors $\mathcal{T},~\mathcal{X},~\mathcal{Y}$ have circulant representations $\underline{T},~\underline{X},~\underline{Y}$ and circular matrix representations $T^{c},~X^{c},~Y^{c}$. We define the Frobenius norm of a circulant as follows.
   \begin{lemma}\label{lemma:frob_equality}
   Define the Frobenius norm of a circulant equals to that of its tensor presentation, i.e., $||\underline{T}||_F = ||\mathcal{T}||_F$. Then, according to (\ref{eq:circ_operator}), we have $||\underline{T}||_F = \frac{1}{\sqrt{k}} ||T^{c}||$. If $\mathcal{T}=\mathcal{X} * \mathcal{Y}$, the following equality is used throughout the paper:
   \begin{equation}
   ||\mathcal{T}||_F =||\mathcal{X} * \mathcal{Y}||_F = \frac{1}{\sqrt{k}}||X^{c} Y^{c}||_F.
   \end{equation}
   \end{lemma}

  \begin{definition}\label{def:circulant_tranpose}
  \textbf{Tubal-wise transpose, circulant transpose}. Let $\mathcal{X}^T \in \mathbb{R}^{n \times m \times k}$ denote the tube-wise transpose of $\mathcal{X} \in \mathbb{R}^{m \times n \times k}$, i.e., $\mathcal{X}^T(i,j,:) = \mathcal{X}(j,i,:)$. Similarly, let $\underline{X}^T$ denote the circulant transpose of $\underline{X}$, i.e., $\underline{X}^T(i,j) = \underline{X}(j,i)$, which can be viewed as the transpose of a matrix of vectors.
  \end{definition}

\section{Problem Statement and Proposed Algorithm}\label{sect:problem_statement}

   We first describe the low-tubal-rank tensor completion problem. Then, we present our Tubal-Alt-Min algorithm followed by detailed descriptions of its key components. Finally, we provide a procedure to implement the tensor least squares minimization.

\subsection{Problem Statement}



  We consider the problem of completing a 3-D tensor under the assumption that the 3-D tensor has low-tubal-rank. Specifically, assume that the data tensor $\mathcal{T} \in \mathbb{R}^{m \times n \times k}$ has tubal-rank $r \ll \min(m,n)$. By observing a set $\Omega \subset [m] \times [n] \times [k]$ of $\mathcal{T}$'s elements, our aim is to recover $\mathcal{T}$. That is, knowing the elements $\mathcal{T}_{ij\ell}$ for $(i,j,\ell) \in \Omega$, we want to estimate the elements outside of $\Omega$ as accurately as possible.

  Let $\mathcal{P}_{\Omega}(\cdot)$ denote the projection of a tensor onto the observed set $\Omega$, such that
  $$[\mathcal{P}_{\Omega}(\mathcal{T})]_{ij\ell} =\left\{
  \begin{aligned}
  \mathcal{T}_{ij\ell}, &~~\text{if}~(i,j,\ell) \in \Omega, \\
  0,~~&~~\text{otherwise}. \\
  \end{aligned}
  \right.
  $$
  Since $\mathcal{T}$ is known to be a low-tubal-rank tensor and the estimated $\hat{\mathcal{T}}$ should be close to $\mathcal{T}$ on the observation set $\Omega$, the {\em low-tubal-rank tensor completion problem} is formulated as the following optimization problem:
  \begin{equation}\label{problem_formulation}
  \begin{split}
  \widehat{\mathcal{T}} = &\argmin\limits_{\mathcal{Z} \in \mathbb{R}^{m \times n \times k}}~|| \mathcal{P}_{\Omega}(\mathcal{Z}) - \mathcal{P}_{\Omega}(\mathcal{T}) ||_F\\
  &{\rm s.t.}~~\text{rank}(\mathcal{Z}) \leq r,
  \end{split}
  \end{equation}
  where $\mathcal{Z} \in \mathbb{R}^{m \times n \times k} $ is the decision variable, and the function $\text{rank}(\cdot)$ refers to the tensor tubal-rank. Note that the noisy case is inherently included since the tensor least squares minimization deals with noise.
  Problem (\ref{problem_formulation}) is NP-hard since the function $\text{rank}(\cdot)$ induces combinatorial complexity and existing works \cite{Shuchin2014CVPR,Shuchin2015} seek to relax the rank function to its convex surrogate, namely, the tensor-nuclear norm. In \cite{Shuchin2015}, it was shown that tensor-nuclear norm minimization results in exact recovery under random sampling if the tensors satisfy certain tensor incoherence conditions (\ref{tensor_incoherency}).



   However, the computational cost of the algorithm in \cite{Shuchin2015} is relatively high due to two key factors: 1) each iteration requires computing SVD for large block diagonal matrices; and 2) the iterations are jointly carried out in both time and frequency domains, thus involving frequent and large number of Fourier and inverse Fourier transforms. Therefore, here we will propose an alternating minimization algorithm for solving (\ref{problem_formulation}), inspired by the alternating minimization approach's empirical and theoretical successes in low-rank matrix completion \cite{Jain2013STOC,Hardt2014COLT,Hard2014FOCS}. 


\subsection{The Alternating Minimization Algorithm for Low-tubal-rank Tensor Completion}

  We decompose the target tensor $\widehat{\mathcal{T}} \in \mathbb{R}^{m \times n \times k}$ as  $\widehat{\mathcal{T}} = \mathcal{X} * \mathcal{Y}^\dagger$, $\mathcal{X} \in \mathbb{R}^{m \times r \times k}$, $\mathcal{Y} \in \mathbb{R}^{n \times r \times k}$, and $r$ is the target tubal-rank. With this decomposition, the problem (\ref{problem_formulation}) reduces to

  \begin{equation}\label{problem_approximate}
  \widehat{\mathcal{T}} = \argmin\limits_{\mathcal{X} \in \mathbb{R}^{m \times r \times k},~ \mathcal{Y} \in \mathbb{R}^{n \times r \times k}} || \mathcal{P}_{\Omega}(\mathcal{T}) - \mathcal{P}_{\Omega} (\mathcal{X} * \mathcal{Y}^\dagger)||_F^2,
  \end{equation}
  which finds a target tensor  $\widehat{\mathcal{T}} = \mathcal{X} * \mathcal{Y}^\dagger$. According to Lemma \ref{best_r_rank}, we know that (\ref{problem_approximate}) is equivalent to the original problems  (\ref{problem_formulation}), if there exists a unique tubal-rank-$r$ tensor. 

\begin{algorithm}[t]
  \caption{Alternating Minimization: $\text{Tubal-Alt-Min}(\mathcal{P}_{\Omega}(\mathcal{T}), \Omega, L, r, \epsilon, \mu_0)$}
  \label{alg_AM}
  \begin{algorithmic}
  \STATE \textbf{Input}: Observation set $\Omega \in [m] \times [n] \times [k]$ and the corresponding elements $\mathcal{P}_{\Omega}(\mathcal{T})$, number of iterations $L$, target tubal-rank $r$, parameter $\epsilon > 0$, coherence parameter $\mu_0$.
  \STATE 1:~~~~~~~$(\Omega_0, ~\Omega_+) \leftarrow \text{Split}(\Omega, 2)$,
  \STATE 2:~~~~~~~$(\Omega_1, ...,~\Omega_L) \leftarrow \text{Split}(\Omega_+, L)$,
  \STATE 3:~~~~~~~$\mathcal{X}_0 \leftarrow \text{Initialize}(\mathcal{P}_{\Omega_0}(\mathcal{T}), \Omega_0, r, \mu_0)$,
  \STATE 4:~~~~~~~For $\ell = 1$ to $L$
  \STATE 5:~~~~~~~~~~~~~~ $\mathcal{Y}_{\ell} \leftarrow \text{MedianLS-Y}(\mathcal{P}_{\Omega_{\ell}}(\mathcal{T}), \Omega_{\ell}, \mathcal{X}_{\ell -1}, r)$,
  \STATE 6:~~~~~~~~~~~~~~ $\mathcal{Y}_{\ell} \leftarrow \text{SmoothQR}(\mathcal{Y}_{\ell}, \epsilon, \mu_0)$,
  \STATE 7:~~~~~~~~~~~~~~$\mathcal{X}_{\ell} \leftarrow \text{MedianLS-X}(\mathcal{P}_{\Omega_{\ell}}(\mathcal{T}), \Omega_{\ell}, \mathcal{Y}_{\ell}, r)$,
  \STATE 8:~~~~~~~~~~~~~~$\mathcal{X}_{\ell} \leftarrow \text{SmoothQR}(\mathcal{X}_{\ell}, \epsilon, \mu_0)$,
  \STATE \textbf{Output}: Tensor pair $(\mathcal{X}_{L}, \mathcal{Y}_{L})$.
  \end{algorithmic}
  \vspace{-2pt}
  \end{algorithm}

For an alternating minimization algorithm, there are two key steps. The first is the initialization. The second is to alternate between finding the best $\mathcal{X}$ and the best $\mathcal{Y}$ given the other. Each alternating optimization in isolation is essentially a tensor least squares update which is convex and tractable. For the analysis of the algorithm, we propose a variant of the alternating minimization algorithm, which is a {\em smoothed} alternating minimization, as shown in Alg. \ref{alg_AM}. This algorithm has the following framework: 1) randomizing $\mathcal{X}_0$ (line 3) as the initial input for the iterative loop (line 4-8); 2) fixing $\mathcal{X}_{\ell -1}$ and optimizing $\mathcal{Y}_{\ell}$ (line 5-6); and 3)  fixing $\mathcal{Y}_{\ell}$ and optimizing $\mathcal{X}_{\ell}$ (line 7-8). Different from the standard alternating minimization that alternatively performs a least squares minimization by fixing one and another, Alg. \ref{alg_AM} introduces a median operation over the least squares minimization, and a SmoothQR process. This median operation enables us to get tighter concentration bounds in the proof, while the SmoothQR process guarantees the incoherence of $\mathcal{X}_{\ell}$ and $\mathcal{Y}_{\ell}$ along the iterations.
%


  The general flow of Alg. \ref{alg_AM} is as follows.
  \begin{itemize}
  \item Line 1-2: Throughout the algorithm, the samples in $\Omega$ are utilized for two purposes: to initialize $\mathcal{X}_0$ as a ``good" starting point, and to update $\mathcal{X}$ and $\mathcal{Y}$ in the iterations. The Split function first splits $\Omega$ into two same-sized subsets $\Omega_0, \Omega_+$ and then splits $\Omega_+$ into $L$ subsets of roughly equal size, while preserving the distributional assumption that our theorem uses.
  \item Line 3: Using the samples in $\Omega_0$, the Intialize procedure generates a good starting point that is relatively close to the optimal solution. This is required for analysis purpose, since from this starting point we are able to prove convergence.
  \item Line 4-8: The MedianLS-minimization procedure relies on the {\em tensor least squares minimization} as described in Section \ref{subsect:LS_mini}. The median process is introduced to derive concentration bounds, while the SmoothQR function guarantees that in each iteration $\mathcal{X}_{\ell}$ and $\mathcal{Y}_{\ell}$ satisfy the tensor incoherence condition defined in (\ref{tensor_incoherency}). For general tensors, each iteration includes two consecutive MedianLS minimizations for $\mathcal{X}$ and $\mathcal{Y}$, respectively. Note that $\mathcal{X}$ and $\mathcal{Y}$ need to be treated differently because of the t-product, which is not the case for the matrix completion \cite{Jain2013STOC,Hardt2014COLT,Hard2014FOCS}.
  \end{itemize}

  In Section \ref{analysis}, we will prove that separating the nonconvex minimization (\ref{problem_approximate}) into two consecutive convex minimization subproblmes will also yield the optimal solution of (\ref{problem_formulation}).

  \begin{remark}
  The framework of Alg. \ref{alg_AM} is extended from a similar approach analyzed in \cite{Hard2014FOCS} for matrix completion. Note that Alg. \ref{alg_AM} differs from that of \cite{Hard2014FOCS} in three major aspects: 1) the low-tubal-rank tensor completion problem is essentially different from the matrix completion problem as shown in  Section \ref{whydifferent}, which indicates that matricizing a tensor that leads to a matrix completion problem is not appropriate for (\ref{problem_formulation}); 2) the key routines of Alg. \ref{alg_AM} in Section \ref{key_routines} has new forms, namely, the Initialize procedure in Alg. \ref{alg_initialization} and the tensor least squares minimization in Alg. \ref{alg_LS_update}; and 3) the implementation of {\em tensor least squares minimization} is newly proposed in Section \ref{subsect:LS_mini}.
  \end{remark}

\subsection{Key Routines}
\label{key_routines}

  The key routines of Alg. \ref{alg_AM} include the Split function, the Initialize procedure in Alg. \ref{alg_initialization}, the median least squares minimization in Alg. \ref{alg_median_LS_update}, the tensor least squares minimization in Alg. \ref{alg_LS_update}, and the SmoothQR factorization in Alg. \ref{alg_smoothQR}. In the following, we describe each one in detail.

  \subsubsection{Splitting the Samples}

  The procedure $\text{Split}(\Omega, t)$ takes the sample set $\Omega$ and splits it into $t$ independent subsets $\Omega_1,...,\Omega_t$ that preserve the uniform distribution assumption, e.g., each elment of $\Omega$ belongs to one of the $t$ subsets by sampling with replacement. This Split function is the same as that for the matrix case \cite{Jain2013STOC,Hardt2014COLT,Hard2014FOCS} since essentially they are both set operations.

  \subsubsection{Finding a Good Starting Point}

  \begin{algorithm}[t]
  \caption{Initialization Algorithm: $\text{Initialize}(\mathcal{P}_{\Omega_0}, \Omega_0, r, \mu)$}.
  \label{alg_initialization}
  \begin{algorithmic}
  \STATE \textbf{Input}: observation set $\Omega_0 \in [m] \times [n] \times [k]$ and elements $\mathcal{P}_{\Omega_0}(\mathcal{T})$, target dimension $r$, coherence parameter $\mu \in \mathbb{R}$.
  \STATE ~~~~~~~Compute the first $r$ eigenslices $\mathcal{A} \in \mathbb{R}^{n \times r \times k}$ of $\mathcal{P}_{\Omega_0}(\mathcal{T})$,\\
  \STATE ~~~~~~~$\mathcal{Z} \leftarrow \mathcal{A} * \mathcal{O}$ where $\mathcal{O} \in \mathbb{R}^{r \times r \times k}$ is a random orthonormal tensor,
  \STATE ~~~~~~~$\mathcal{Z}' \leftarrow \text{Truncate}_{\mu'}(\mathcal{Z})$ with $\mu'= \sqrt{8\mu \log n / n}$, where $\text{Truncate}_{\mu'}$ scales tubes $\mathcal{Z}(i,j,:)$ with $||\mathcal{Z}(i,j,:)||_F > \mu'$ by $\mathcal{Z}(i,j,:)/\mu'$,\\
  \STATE ~~~~~~~$\mathcal{X}_0 \leftarrow \text{QR}(\mathcal{Z}')$, where $QR(\cdot)$ is the standard QR factorization that returns an orthogonal subspace.
  \STATE \textbf{Output}: Orthonormal tensor $\mathcal{X}_0 \in \mathbb{R}^{n \times r \times k}$.
  \end{algorithmic}
  \vspace{-2pt}
  \end{algorithm}

   Alg. \ref{alg_initialization} describes the procedure for finding a good starting point. Since it is unclear how well the least squares minimization will converge from a random initial tensor, we start with an initial tensor that has bounded distance from the optimal result as shown in Appendix \ref{sec:initialization}. The algorithm serves as a fast initialization procedure for our main algorithm. It computes the top-$r$ eigenslices of $\mathcal{P}_{\Omega}(\mathcal{T})$, and truncates them in order to ensure incoherence.  Note that the truncation for our tensor problem scales the coefficients of a tube, which is different from the matrix case \cite{Hard2014FOCS} that truncates an element. We use a random orthonormal transformation to spread out the tubes of the eigenslices before truncation.

 \subsubsection{Tensor Least Squares Minimization}

   We describe the MedianLS iteration of $\mathcal{Y}$. Although the iteration for $\mathcal{X}$ is different from $\mathcal{Y}$, essentially it can be computed in a similar way as shown in Section \ref{subsect:LS_mini}. Each MedianLS minimization relies on the basic tensor least squares minimization. Partition $\Omega_+$ into $t = O(\log n)$ subsets, then performing the least squares minimization on each subset and then take the median of the returned tensors.  The median operation is performed in an element-wise manner.

  \begin{algorithm}[t]
  \caption{Median Least Squares:   $\text{MedianLS-Y}(\mathcal{P}_{\Omega}(\mathcal{T}), \Omega, \mathcal{X}, r)$}
  \label{alg_median_LS_update}
  \begin{algorithmic}
  \STATE \textbf{Input}: target tubal-rank $r$, observation set $\Omega \in [m] \times [n] \times [k]$ and elements $\mathcal{P}_{\Omega}(\mathcal{T})$, orthonormal tensor $\mathcal{X} \in \mathbb{R}^{m \times r \times k}$.
  \STATE ~~~~~~~$(\Omega_1,...,\Omega_t) \leftarrow \text{Split}(\Omega,t)~\text{for}~t=3\log_2 n$,
  \STATE ~~~~~~~$\mathcal{Y}_i = \text{LS}(\mathcal{P}_{\Omega_i}(\mathcal{T}), \Omega_i, \mathcal{X}, r)$ for $i \in [t]$,
  \STATE \textbf{Output}: $\text{median}(\mathcal{Y}_1,...,\mathcal{Y}_t)$.
  \end{algorithmic}
  \vspace{-2pt}
\end{algorithm}

\begin{algorithm}[t]
  \caption{Tensor Least Squares Minimization: $\text{LS}(\mathcal{P}_{\Omega}(\mathcal{T}), \Omega, \mathcal{X}, r)$}
  \label{alg_LS_update}
  \begin{algorithmic}
  \STATE \textbf{Input}: target dimension $r$, observation set $\Omega \in [m] \times [n] \times [k]$ and elements $\mathcal{P}_{\Omega}(\mathcal{T})$, orthonormal tensor $\mathcal{X} \in \mathbb{R}^{m \times r \times k}$.
  \STATE ~~~~~~~$ \mathcal{Y} = \argmin_{\mathcal{Y} \in \mathbb{R}^{n \times r \times k} } ||\mathcal{P}_{\Omega} (\mathcal{T} - \mathcal{X} * \mathcal{Y}^\dagger )||_F^2 $,
  \STATE \textbf{Output}: $\mathcal{Y}$.
  \end{algorithmic}
  \vspace{-2pt}
\end{algorithm}

 \subsubsection{Smooth QR}

  For the main theorem to hold, it is required that each iterates $\mathcal{X}_\ell$ and $\mathcal{Y}_\ell$ have coherence less than $\mu$. To achieve this, we adopt the smooth operation, as shown in Alg. \ref{alg_smoothQR}. Note the $\text{QR}(\cdot)$ and $\text{GS}(\cdot)$ operation for third-order tensors are defined in \cite{Kilmer2013}. The $\text{QR}(\cdot)$ process returns the orthogonal projector while $\text{GS}(\cdot)$ make it to be orthonormal tensor-column space. The Gaussian perturbation $\mathcal{H}_{\ell}$ is added to $\mathcal{Y}_{\ell}$ to ensure small coherence. In our context, the tensor $\mathcal{X}_\ell$ and $\mathcal{Y}_\ell$ are the outcomes of a noisy operation $\mathcal{X}_\ell = \mathcal{A}* \mathcal{Y}_{\ell - 1} + \mathcal{G}_{\ell}$ (in Appendix \ref{sec:noisy_subspace_iteration}), and so there is no harm in actually adding a Gaussian noise tensor $\mathcal{H}_{\ell}$ to $\mathcal{X}_\ell$ provided that the norm of that tensor is no larger than that of $\mathcal{G}_{\ell}$.

  \begin{algorithm}[t]
  \caption{Smooth QR factorization: $\text{SmoothQR}(\mathcal{Y}, \epsilon, \mu)$}
  \label{alg_smoothQR}
  \begin{algorithmic}
  \STATE \textbf{Input}: $\mathcal{Y} \in \mathbb{R}^{n \times r \times k}$, parameters $\mu, \epsilon > 0$
  \STATE ~~~~~~~$\mathcal{Z} \leftarrow \text{QR}(\mathcal{Y}), \mathcal{H} \leftarrow 0, \sigma \leftarrow \epsilon ||\mathcal{Y}|| / n$,
  \STATE ~~~~~~~While $\mu(\mathcal{Z}) > \mu$ and $\sigma \leq ||\mathcal{Y}||$
  \STATE ~~~~~~~~~~~$\mathcal{Z} \leftarrow \text{GS}(\mathcal{Y} + \mathcal{H})$ where $\mathcal{H} \backsim \text{N}(0, \sigma^2 / n)$,
  \STATE ~~~~~~~~~~~$\sigma \leftarrow 2 \sigma$,
  \STATE \textbf{Output}: $\mathcal{Z}$.
  \end{algorithmic}
  \vspace{-2pt}
  \end{algorithm}


\subsection{Implementation of Tensor Least Squares Minimization}\label{subsect:LS_mini}

  In this section, we describe the detailed steps to solve the tensor least squares minimization problem in Alg. \ref{alg_LS_update}:
  $ \widehat{\mathcal{Y}} = \argmin\limits_{\mathcal{Y} \in \mathbb{R}^{n \times r \times k} } ||\mathcal{P}_{\Omega} (\mathcal{T} - \mathcal{X} * \mathcal{Y}^\dagger )||_F^2 $. To simplify the description, we define three new products as follows.
  \begin{definition}\label{def:new_products}
  \textbf{Element-wise tensor product, tube-wise circular convolution, frontal-slice-wise tensor product}. The element-wise tensor product $\mathcal{T} = \mathcal{T}_1 \odot \mathcal{T}_2$ operates on two same-sized tensors and results in a same-size tensor, i.e., $T(i,j,\kappa) = \mathcal{T}_1(i,j,\kappa) \mathcal{T}_{2}(i,j,\kappa)$ for $\mathcal{T}, \mathcal{T}_1, \mathcal{T}_2 \in \mathbb{R}^{m \times n \times k}$. The tube-wise circular convolution  $\mathcal{T} = \mathcal{T}_1 \cdot\otimes \mathcal{T}_2$ operates on two same-sized tensors and results in a same-size tensor, i.e., $\mathcal{T}(i,j,:) = \mathcal{T}_1(i,j,:) * \mathcal{T}_2(i,j,:)$ for $\mathcal{T}, \mathcal{T}_1, \mathcal{T}_2 \in \mathbb{R}^{m \times n \times k}$. The frontal-slice-wise tensor product $\mathcal{T} = \mathcal{T}_1 \cdot\S \mathcal{T}_2$ performs matrix multiplication on the corresponding frontal slices of two tensors, i.e., $\mathcal{T}(:,:,\kappa) = \mathcal{T}_1(:,:,\kappa) \mathcal{T}_2(:,:,\kappa)$ for $\mathcal{T} \in \mathbb{R}^{m \times m \times k}, \mathcal{T}_1 \in \mathbb{R}^{m \times n \times k}, \mathcal{T}_2 \in \mathbb{R}^{n \times n \times k}$.
  \end{definition}

  Here, the frontal-slice-wise tensor product is introduced to have a concise third-order tensor representation of Remark \ref{remark:computing_tensor_product}, avoiding the block diagonal form representations as in Definition \ref{def:block-diagonal}. The operator $\cdot \S$ is introduced to replace the block diagonal matrix in \cite{Shuchin2015}\cite{Shuchin2014CVPR} since we want to preserve the three-way data array structure.

  For simplicity, denote $\mathcal{T}_{\Omega} = \mathcal{P}_{\Omega}(\mathcal{T})$, then we have $\mathcal{T}_{\Omega}  = \mathcal{P}_{\Omega} \odot \mathcal{T}$. According to the {\em Convolution Theorem}, we can transform the least squares minimization to the following frequency domain version:
  \begin{equation}\label{LS_transform}
  \widehat{\mathcal{Y}} = \argmin_{\widetilde{\mathcal{Y}} \in \mathbb{R}^{r \times n \times k}} ||\widetilde{\mathcal{T}}_{\Omega} - \widetilde{\mathcal{P}}_{\Omega} \cdot\otimes (\widetilde{\mathcal{X}}~\cdot\S~ \widetilde{\mathcal{Y}}) ||_F^2.
  \end{equation}
  We first transform (\ref{LS_transform}) into $n$ separate standard least squares minimization subproblems:
  \begin{equation}\label{LS_transform_sub}
  \widehat{\mathcal{Y}}(:,j,:) = \argmin_{\widetilde{\mathcal{Y}}(:,j,:) \in \mathbb{R}^{r \times 1 \times k}} ||\widetilde{\mathcal{T}}_{\Omega}(:,j,:) - \widetilde{\mathcal{P}}_{\Omega}(:,j,:) \cdot\otimes (\widetilde{\mathcal{X}}~\cdot\S~ \widetilde{\mathcal{Y}}(:,j,:)) ||_F^2,
  \end{equation}
  where each subproblem corresponds to estimating a lateral slice $\widetilde{\mathcal{Y}}(:,j,:),~j \in [n]$. One can solve it by performing the following steps.
\begin{enumerate}
\item A lateral slice, $\widetilde{\mathcal{T}}_{\Omega}(:,j,:)$ of size $n \times 1 \times k$, is squeezed into a vector $b$ of size $nk \times 1$ in the following way:
  \begin{equation}
   b= [ \text{squeeze}( \widetilde{\mathcal{T}}_{\Omega}(1,j,:)); \text{squeeze}( \widetilde{\mathcal{T}}_{\Omega}(2,j,:)); ...; \text{squeeze}( \widetilde{\mathcal{T}}_{\Omega}(n,j,:))],
  \end{equation}
  where $\text{squeeze}( \widetilde{\mathcal{T}}_{\Omega}(i,j,:))$ squeezes the $i$-th tube of the $j$-th lateral slice of $\widetilde{\mathcal{T}}_{\Omega}$ into a vector of size $k \times 1$. Similarly when estimating $\mathcal{Y}$, $\widetilde{\mathcal{Y}}_\Omega(:,j,:)$ is transformed into a vector $\textbf{x}$ of size $rk \times 1$:
  \begin{equation}
  \textbf{x}= [ \text{squeeze}( \widetilde{\mathcal{Y}}_{\Omega}(1,j,:)); \text{squeeze}( \widetilde{\mathcal{Y}}_{\Omega}(2,j,:)); ...; \text{squeeze}( \widetilde{\mathcal{Y}}_{\Omega}(n,j,:))];.
  \end{equation}
\item  $\widetilde{\mathcal{X}}$ is transformed into a block diagonal matrix of size $nk \times rk$,

    \begin{equation}
    A_1 =\left[
    \begin{array}{cccc}
    \widetilde{\mathcal{X}}(:,:,1)&  &  &   \\
  & \widetilde{\mathcal{X}}(:,:,2) &  &  \\
     &   & ... &  \\
     & &   & \widetilde{\mathcal{X}}(:,:,k)\\
    \end{array}
    \right].
    \end{equation}

\item The $j$-th lateral slice $\widetilde{\mathcal{P}}_{\Omega}(:,j,:)$ is transformed into a tensor $\mathcal{A}_2$ of size $k \times k \times n $ first, and then into a matrix $A_3$ of size $nk \times nk$.

  \begin{equation}
  \mathcal{A}_2(:,:,\ell) = \text{circ}(\widetilde{\mathcal{P}}_{\Omega}(\ell,j,:)),~\ell \in [n],
  \end{equation}
  \begin{equation}
  A_3 = \left[
    \begin{array}{cccc}
    \text{diag}(\mathcal{A}_2(1,1,:))& \text{diag}(\mathcal{A}_2(1,2,:)) & \hdots & \text{diag}(\mathcal{A}_2(1,k,:)) \\
    \text{diag}(\mathcal{A}_2(2,1,:)) & \text{diag}(\mathcal{A}_2(2,2,:)) & \hdots & \vdots\\
    \vdots & \vdots & \vdots & \vdots\\
    \text{diag}(\mathcal{A}_2(k,1,:)) & \hdots & \hdots & \text{diag}(\mathcal{A}_2(k,k,:))\\
    \end{array}
    \right],
  \end{equation}
  where the operator $\text{diag}(\cdot)$ transform a tube into a diagonal matrix by putting the elements in the diagonal.
\end{enumerate}

  Therefore, the problem in (\ref{LS_transform_sub}) becomes the following standard LS problem:
  \begin{equation}\label{eq:standart_ls}
  \widehat{\textbf{x}} = \argmin_{\textbf{x} \in \mathbb{R}^{rk \times 1}} || b - A_3 A_1 \textbf{x}||_F^2.
  \end{equation}

  \noindent Similarly, we can estimate $\widehat{\mathcal{X}}$, using similar steps as that for $\widehat{\mathcal{Y}}$:
  \begin{equation}
  \widehat{\mathcal{X}} = \argmin_{\mathcal{X} \in \mathbb{R}^{n \times r \times k}} ||\widetilde{\mathcal{T}}_{\Omega}^T - \widetilde{\mathcal{P}}_{\Omega}^T \cdot\otimes (\widetilde{\mathcal{Y}}^T ~\cdot\S~  \widetilde{\mathcal{X}}^T ) ||_F^2,
  \end{equation}
  where $\mathcal{X}^T$ denotes the tube-wise transpose (the transpose of a matrix of vectors).


\section{Performance of the Proposed Algorithm}\label{analysis}

  We first describe a counter example to show that the low-tubal-rank tensor completion problem is essentially different from the conventional matrix completion problem. Then, we present the analytical results for the performance guarantees.

\subsection{Why is Low-tubal-rank Tensor Completion Different from Matrix Completion}\label{whydifferent}

   One would naturally ask if tensor completion is in essence equivalent to matrix completion. It appears to be true but in fact wrong. Therefore, a tensor completion problem should be treated differently from a matrix completion problem. On one hand, in Section \ref{sect:circular_algebra} we introduce the operation $\text{circ}(\cdot)$ to establish a mapping between tensor product and matrix multiplication. However, such a mapping is injective, and we use it for easier understanding and it does not mean equivalence. On the other hand, in Appendix \ref{LS_to_NSI} we express the least squares update step as $\mathcal{Y} = \mathcal{T} * \mathcal{X} + \mathcal{G}$ and then transform tensor operations to matrix operations on the corresponding circular matrices. Since the mapping is injective, this transformation is for analysis purpose, which holds only if the operations in the tensor forms hold. Still it is a necessary condition (not sufficient condition) resulting from the injective mapping between tensor product and its circular matrix multiplication.

   The projection $\mathcal{P}_{\Omega}(\mathcal{T})$ can be viewed as an element-wise multiplication between two third-order tensors, i.e., $\mathcal{P}_{\Omega}(\mathcal{T}) =  \mathcal{P}_{\Omega} \odot \mathcal{T}$.
   We define a corresponding projection $\mathcal{P}_{\Omega'}(\cdot)$ for $T^c$ as follows: $\Omega' = \text{circ}(\Omega), P_{\Omega'} = \text{circ}(\mathcal{P}_{\Omega})$, i.e., $\mathcal{P}_{\Omega'}(T^c) = P_{\Omega'} \odot T^c$.

   For the least squares update $\mathcal{Y} = \argmin_{\mathcal{Y} \in \mathbb{R}^{n \times r \times k} } ||\mathcal{P}_{\Omega} (\mathcal{T} - \mathcal{X} * \mathcal{Y}^\dagger)||_F^2$, its circular form would be $Y = \argmin_{Y \in \mathbb{R}^{nk \times rk }} \frac{1}{\sqrt{k}}||\mathcal{P}_{\Omega'} (T^{c} - X^{c}Y^\dagger )||_F^2$. If $Y$ is circular we can transform $Y$ back to a tensor, then these two problems are equivalent. Therefore, the original question becomes the following one: will this circular least squares minimization output a circular matrix $Y$?
   In the following, we give a negative answer by presenting a counter example.

   Assume that $\Omega = [n] \times [n] \times [k]$, then the circular least squares minimization is equivalent to the following optimization problem:
   \begin{equation}\label{circular_LS}
   \min_{Y}~~ ||G||_F^2, ~~~~{\rm s.t.}~~ \mathcal{P}_{\Omega'}(T^{c}) = \mathcal{P}_{\Omega'}(X^{c} Y + G),
   \end{equation}
   where $G$ is a noise term. Without loss of generality, considering the following simple example with $\Omega$ being the whole set:
   \begin{equation} \label{circular_LS_example}
   \min~~ ||G||_F^2,~~~~
   {\rm s.t.}~~
   \begin{bmatrix}
   T_1 & T_2 \\
   T_2 & T_1 \\
   \end{bmatrix} =
    \begin{bmatrix}
   X_1 & X_2 \\
   X_2 & X_1 \\
   \end{bmatrix}
   \begin{bmatrix}
   Y_{11} & Y_{21} \\
   Y_{12} & Y_{22} \\
   \end{bmatrix} +
   \begin{bmatrix}
   G_{11} & G_{12} \\
   G_{21} & G_{22} \\
   \end{bmatrix},
   \end{equation}
   where $\mathcal{T}, \mathcal{X} \in \mathbb{R}^{1 \times 1 \times 2}$ with $T^{c}, X^{c} \in \mathbb{R}^{2 \times 2}$, and $Y, G \in \mathbb{R}^{2 \times 2}$.
   The constraint in (\ref{circular_LS_example}) can be transformed to the following four linear equations:
   \begin{align}
   T_1 = & X_1 Y_{11} + X_2 Y_{12} + G_{11} \label{eqn_t1},\\
   T_1 = & X_1 Y_{22} + X_2 Y_{21} + G_{22} \label{eqn_t2},\\
   T_2 = & X_1 Y_{21} + X_2 Y_{22} + G_{12} \label{eqn_t3},\\
   T_2 = & X_1 Y_{12} + X_2 Y_{11} + G_{21} \label{eqn_t4}.
   \end{align}

   Considering the first two equations (\ref{eqn_t1}) and (\ref{eqn_t2}), if $G_{11} = G_{22}$ (and also $G_{12} = G_{21}$ in (\ref{eqn_t3}) and (\ref{eqn_t4})), then the solution $Y$ is a circular matrix. Given $G_{11} = - G_{22}$ that ensuring $G_{11}^2 = G_{22}^2$,  then the solution $Y$ is not circular matrix. Therefore, the problem in (\ref{circular_LS}) does not guarantee to output a circular matrix.

\subsection{Sampling Complexity and Computational Complexity}\label{subsection:recovery_error}

   \subsubsection{Recovery Error}

   We assume that the unknown noisy tensor $\mathcal{T} \in \mathbb{R}^{m \times n \times k}$ is approximately low-tubal-rank in Frobenius norm. According to Lemma \ref{best_r_rank}, $\mathcal{T}$ has the form $\mathcal{T} = \mathcal{M} + \mathcal{N}$ where $\mathcal{M} = \mathcal{U} * \Theta * \mathcal{V}^\dagger$ has tubal-rank $r$ and eigentubes $\Theta(1,1,:), \Theta(2,3,:), ..., \Theta(r,r,:)$ such that $||\Theta(1,1,:)||_F \geq ||\Theta(2,2,:)||_F \geq ...\geq ||\Theta(r,r,:)||_F$, $\mathcal{U} \in \mathbb{R}^{m \times r \times k}, \mathcal{V} \in \mathbb{R}^{n \times r \times k}, \Theta \in \mathbb{R}^{r \times r \times k}$, and $\mathcal{N} = (\mathcal{I} - \mathcal{U}*\mathcal{U}^\dagger) * \mathcal{T} $ is the part of $\mathcal{T}$ not captured by the tensor-column subspace $\mathcal{U}$ and the dominant eigentubes. Here, we assume that $\mathcal{N}$ can be an arbitrary deterministic tensor that satisfies the following constraints:
   \begin{equation}\label{condition_N}
   \max_{i \in [n]} ||\dot{e}_i^\dagger * \mathcal{N} ||_F^2 \leq \frac{\mu_N}{n} \overline{\sigma}_{rk}^2, ~~\text{and}~~\max_{i,j, \kappa} |\mathcal{N}_{ij\kappa}| \leq \frac{\mu_N}{\max(m,n)}||\mathcal{T}||_F,
   \end{equation}
   where $\dot{e}_i$ denotes the $i$-th tensor-column basis so that $|| \dot{e}_i^\dagger * \mathcal{N}||_F$ is the Frobenius norm of the $i$-th horizontal slice of $\mathcal{N}$, and  $\overline{\sigma}_{rk}$ is the $rk$-th singular value of the block diagonal matrix $\overline{\mathcal{T}}$. The block diagonal matrix $\overline{\mathcal{T}}$ has (approximately) tubal-rank $rk$ and we denote those $rk$ singular values as $\overline{\sigma}_1, \overline{\sigma}_2,...,\overline{\sigma}_{rk}$ such that $\overline{\sigma}_1 \geq \overline{\sigma}_2 \geq ... \geq \overline{\sigma}_{rk}$, and $\overline{\sigma}_{rk+1} > 0$, where $\overline{\sigma}_{rk}$ is the smallest singular value of $\overline{\mathcal{M}}$ and  $\overline{\sigma}_{rk+1}$ is the largest singular value of $\overline{\mathcal{N}}$. These constraints state that no element and no horizontal slice of $\mathcal{N}$ should be too large compared with the Frobenius norm of $\mathcal{N}$. One can think of the parameter $\mu_N$ as an analog to the coherence parameter $\mu(\mathcal{U})$ in (\ref{tensor_incoherency}).

   Let $\mu^* = \max\{\mu(\mathcal{U}), \mu_N, \log k\max(m,n) \}$ and $\gamma_{rk} = 1 - \overline{\sigma}_{rk + 1}/ \overline{\sigma}_{rk}$, then we have the following theorem.

   \begin{theorem}\label{theorem:noisy_case}
   Given a sample set $\Omega$ of size $O(pmnk)$ with each element randomly drawn from an unknown $m \times n \times k$ tensor $\mathcal{T} = \mathcal{M} + \mathcal{N}$, where $ \mathcal{M} = \mathcal{U} * \Theta * \mathcal{V}^\dagger$ has tubal-rank $r$ and $\mathcal{N} = (\mathcal{I} - \mathcal{U}*\mathcal{U}^\dagger)*\mathcal{T}$ satisfies condition (\ref{condition_N}), then with probability at least $1 - \Theta\left(\frac{\gamma_{rk}^{-1}\log(\max(m,n)/\epsilon)}{\max(m,n)^2}\right)$, Alg. \ref{alg_AM} outputs $(\widehat{\mathcal{X}}, \widehat{\mathcal{Y}})$ such that 
   $||\widehat{\mathcal{X}} * \widehat{\mathcal{Y}}^\dagger - \mathcal{M}||_F \leq \frac{3\epsilon}{2} ||\mathcal{T}||_F$, provided that
   \begin{equation}
     p = O\left(\frac{r^2 \mu^* \left( ||\mathcal{M}||_F^2 + ||\mathcal{N}||_F^2/\epsilon^2 \right) \log(\max(m,n)/\epsilon) \log^2 (k\max(m,n)) }{\gamma_{rk}^5 \overline{\sigma}_{rk}^2 \max(m,n)}\right).
   \end{equation}
   \end{theorem}

   \begin{proof}
   This theorem is a direct result of Theorem \ref{main_theorm}, and the detailed proof is presented in Appendix \ref{proof:theorem}. Here, we briefly describe the high-level proof structure for easy understanding.

   Initialization is analyzed in Appendix \ref{sec:initialization}. We prove in Lemma \ref{proof:initialization_p} that initializing $\mathcal{X}_0$ as the top-$r$ left orthogonal eigenslices of $\frac{1}{p} \mathcal{P}_{\Omega}(\mathcal{T})$ will result in bounded distance to the optimum tensor $\mathcal{T}$, i.e.,  $||( \mathcal{I} - \mathcal{U} * \mathcal{U}^\dagger) * \mathcal{X}_0|| \leq 1/4$ with probability at least $1 - 1/\max(m,n)^2$. This is necessary because the key step in Alg. \ref{alg_AM}, iterating the tensor least squares minimization, does not provide information about how well it converges from a random initial tensor.

  With the ``good" initialization, our analysis proceeds as follows:
  \begin{itemize}
  \item Express the least squares update step as an update step of the noisy tensor-column subspace iteration which has the form $\mathcal{Y} = \mathcal{T} * \mathcal{X} + \mathcal{G}$.
  \item Analyze the local convergence of  Alg. \ref{alg_LS_update} by proving that the noisy tensor-column subspace iteration will converge at an exponential rate in Appendix \ref{sec:noisy_subspace_iteration}.
  \item Prove that the Frobenius norm of the error term $\mathcal{G}$, i.e., $||\mathcal{G}||$ in spectrum norm, is bounded in Lemma \ref{lemma:bounded}.
  \end{itemize}

  The smoothed alternating tensor least squares minimization is analyzed in Appendix \ref{LS_to_NSI}. The key step of Alg. \ref{alg_AM} is a tensor least squares minimization $\mathcal{Y}_\ell = \argmin_{\mathcal{Y} \in \mathbb{R}^{n \times r \times k}} || \mathcal{P}_{\Omega} (\mathcal{T} - \mathcal{X}_{\ell-1} * \mathcal{Y}^\dagger) ||_F^2$.
  We will show that the solution to this has the form of $\mathcal{Y}_\ell = \mathcal{T} * \mathcal{X}_{\ell-1} + \mathcal{G}_\ell$. The error term $||\mathcal{G}_\ell||$ in spectrum norm depends on the quantity $||( \mathcal{I} - \mathcal{U} * \mathcal{U}^\dagger) * \mathcal{X}_{\ell-1}||$ which coincides with the sine of the largest principal angle between $\mathcal{U}$ and $\mathcal{X}_{\ell-1}$. As the algorithm converges, the spectrum norm of the error term diminishes.
  \end{proof}


   \begin{remark}
   The low-tubal-rank tensor completion problem generalizes the low-rank matrix completion problem. One can verify that the result in Theorem \ref{theorem:noisy_case} reduces to Theorem I.2 in \cite{Hard2014FOCS} when setting $k=1$.
   \end{remark}

   \begin{remark}
   As compared to the TNN-ADMM's \cite{Shuchin2015} sampling complexity of $O\left(\max(m,n) k r \log \max(m,n) \right)$, Tubal-Alt-Min requires an extra $r$ factor. Even though the theoretical results need an extra $r$ factor, the numerical results in Section \ref{sec:evaluation} depict lower error using fewer number of samples.
   \end{remark}

   \subsubsection{Computational Complexity}

   \begin{theorem}
   The Tubal-Alt-Min algorithm in Alg. \ref{alg_AM} has a computational complexity $O(mnrk^2 \log n \log(n/\epsilon))$ if $k \geq \frac{n}{r \log n \log(n/\epsilon)}$, and $O(mn^2k)$ otherwise.
   \end{theorem}

   \begin{proof}
   We characterize the computational complexity by counting the multiplications involved in Alg. \ref{alg_AM}, which comes from the initialization process (line 1-3) and the for-loop (line 4-8).

   The initialization includes two steps: the t-SVD for computing the first $r$ eigenslices and the DFT along the third dimension. The t-SVD in Alg. \ref{alg:tSVD} has complexity $O(mn^2k)$, while taking the DFT over $mn$ tubes of size $k$ has complexity $O(mn k \log k)$. Therefore, the initialization process has total complexity $O(mn^2k + mn k\log k) = O(mn^2k)$ since $\log k \ll n$.

   Note that in Alg. \ref{alg_AM} there are $L=\Theta(\log(n/\epsilon))$  iterations, and each median operation is taken over $t=O(\log n)$ tensor least squares minimizations. The complexity of the tensor least squares minimization is dominated by (\ref{eq:standart_ls}), which has complexity $O(nrk^2)$. More specifically, calculating $A_3 A_1$ of (\ref{eq:standart_ls}) has complexity $O(rk)$ exploiting the sparseness of both $A_3$ and $A_1$, and the least squares minimization  (\ref{eq:standart_ls}) requires $O(nrk^2)$ multiplications which also takes into consideration the diagonal structure of $A_3A_1$. Note that each lateral slice is treated separately as shown in (\ref{LS_transform_sub}), thus a tensor least squares minimization has complexity $O(mnrk^2)$.


   Therefore, the total complexity of Alg. \ref{alg_AM} is $O(Lmnrk^2t + mn^2k)$ where $L=\Theta(\log(n/\epsilon))$ and $t=O(\log n)$. In particular, the complexity is $O(mnrk^2 \log n \log(n/\epsilon))$, if $k \geq \frac{n}{r \log n \log(n/\epsilon)}$, and if $k < \frac{n}{r \log n \log(n/\epsilon)}$, the total complexity is $O(mn^2k)$, dominated by the SVD.
   \end{proof}

   \begin{collary}
   The recovery error of the algorithm $||\widehat{\mathcal{X}} * \widehat{\mathcal{Y}}^\dagger - \mathcal{M}||_F$ decreases exponentially with the number of iterations $L$. However, the sampling complexity increases by a constant factor with $L$.
   \end{collary}
   \begin{proof}
   	In $L = \Theta( \log(n /\epsilon))$ iterations,  $||\widehat{\mathcal{X}} * \widehat{\mathcal{Y}}^\dagger - \mathcal{M}||_F \leq \frac{3\epsilon}{2} ||\mathcal{T}||_F$. Since $L = \Theta( \log(n /\epsilon))$, $\epsilon = \Theta(\exp{(-C L)})$ for some constant $C$ thus showing that the error $||\widehat{\mathcal{X}} * \widehat{\mathcal{Y}}^\dagger - \mathcal{M}||_F$ decreases exponentially with $L$.
   \end{proof}

   \begin{remark}
   	The number of update steps ($L = \Theta(\log(n /\epsilon))$) enters the sample complexity since we assume (as in \cite{Hard2014FOCS}) that fresh samples are used in each step.
   	\end{remark}

\section{Evaluation}\label{sec:evaluation}

   We evaluate our alternating minimization algorithm (Tubal-Alt-Min) on both synthetic and real video data. The synthetic data, generated according to our low-tubal-rank tensor model, serves as well-controlled inputs for testing and understanding Tubal-Alt-Min's performance over the convex algorithm TNN-ADMM. The real video data tests the applicability of our low-tubal-rank tensor model, compared with other tensor models.

   Please note that for actual implementation, we simplify Alg. \ref{alg_AM} by removing the Split operation (we use the whole observation set $\Omega$ in all subroutines of Alg. \ref{alg_AM}), the median operation, and the SmoothQR function. These three techniques are introduced only for the purpose of obtaining theoretical performance guarantees. The detailed steps of the simplified algorithm are in Algorithm \ref{alg_SAMI}. Note that random initialization is sufficient in practice as we observed the same performance with Alg. \ref{alg_initialization} in all of our testing cases.

\begin{algorithm}[h]
	\caption{Simplified Tubal Alternating Minimization}
	\label{alg_SAMI}
	\begin{algorithmic}
		\STATE \textbf{Input}: Observation set $\Omega \in [m] \times [n] \times [k]$ and the corresponding elements $\mathcal{P}_{\Omega}(\mathcal{T})$, number of iterations $L$, target tubal-rank $r$.
		\STATE 1:~~~~~~~$\mathcal{X}_0 \leftarrow \text{Initialize}(\mathcal{P}_{\Omega}(\mathcal{T}), \Omega, r)$,
		\STATE 2:~~~~~~~For $\ell = 1$ to $L$
		\STATE 3:~~~~~~~~~~~~~~  $\mathcal{Y}_{\ell} \leftarrow \text{LS}(\mathcal{P}_{\Omega}(\mathcal{T}), \Omega, \mathcal{X}_{\ell -1}, r)$,
		\STATE 4:~~~~~~~~~~~~~~~$\mathcal{X}_{\ell} \leftarrow \text{LS}(\mathcal{P}_{\Omega}(\mathcal{T}), \Omega, \mathcal{Y}_{\ell}, r)$,
		\STATE \textbf{Output}: Pair of tensors $(\mathcal{X}_{L}, \mathcal{Y}_{L})$.
	\end{algorithmic}
	\vspace{-2pt}
\end{algorithm}

\subsection{Experiment Setup}

   We use Matlab installed on a server with Linux operating system. The parameters of the server is: Intel$\circledR$ Xeon$^{\circledR}$ Processor E5-2650 v3, $2.3$ GHz clock speed, $2$ CPU each having $10$ physical cores, virtually maximum $40$ threads, $25$ MB cache, and $64$ GB memory.

   For synthetic data, we compare our Tubal-Alt-Min algorithm with tensor-nuclear norm (TNN-ADMM) \cite{Shuchin2014CVPR}, since both are designed for low-tubal-rank tensors, to show the advantages of our non-convex approach over its convex counterpart. We conduct experiments to recover third-order tensors of different sizes $m \times n \times k$ and tubal-ranks $r$, from observed elements in the subset $\Omega$. Three metrics are adopted for comparison, e.g., the recovery error, the running time, and the convergence speed.
   \begin{itemize}
   \item For recovery error, we adopt the relative square error metric, defined as RSE$=||\widehat{\mathcal{T}} - \mathcal{T}||_F / ||\mathcal{T}||_F$.
   \item For running time, varying the tensor size and fixing other parameters, we measure CPU time in seconds.
   \item For convergence speed, we measure the decreasing rate of the RSE across the iterations by linearly fitting the measured RSEs (in log scale). We include those plots due to three reasons: 1) both algorithms are iterative; 2) our theoretical analysis predicts exponential convergence; and 3) the decreasing speed of the RSE provides explanations for the observed performance of the recovery error and the running time.
   \end{itemize}

   For real dataset, we choose a basketball video of size $144 \times 256 \times 40$ (source: YouTube, as used in \cite{Shuchin2014CVPR}), with a non-stationary panning camera moving from left to right horizontally following the running players. Besides the low-tubal-rank tensor model, there are two widely used tensor models: low CP-rank tensor \cite{Kolda2009} and low Tuker-rank tensor \cite{Ye2013PAMI}. The compared algorithms are briefly described as follows:
   \begin{itemize}
   \item TNN-ADMM: \cite{Shuchin2014CVPR} introduced a convex norm, tensor-nuclear norm (TNN) \cite{Shuchin2015}, to approximate the tubal-rank and proposed an Alternating Direction Method of Multipliers (ADMM) algorithm. It is a convex approach and shares the alternating and iterative features. We include TNN-ADMM \cite{Shuchin2014CVPR} to show that the predicted advantages on the synthetic data in Section \ref{sect:synthetic} also hold for real datasets.
   \item The alternating minimization algorithm (CP-Altmin) \cite{Jain2014NIPS} under CANDECOMP/PARAFAC decomposition: The CP decomposition models a tensor as the outer product of $r$ (the CP-rank) vector components. The CP-Altmin algorithm \cite{Jain2014NIPS} alternates among the least squares minimization subproblems for each component. 
   \item The gradient-type algorithm Tuker-Gradient under Tuker decomposition: Tuker decomposition matricizes a tensor from $n$ modes, while the Tuker-Gradient algorithm \cite{Tuker} iteratively estimate each matricized component using the conjugate gradient decent. Essentially, Tuker-Gradient does not fall into the alternating minimization approach. However, we feel it is necessary to compare with it since both Tubal-Alt-Min and Tuker-Gradient iteratively estimate each factor and are nonconvex.
   \end{itemize}

\subsection{Synthetic Data}\label{sect:synthetic}

  \begin{figure}[t]\centering
  \includegraphics[width=0.495\textwidth]{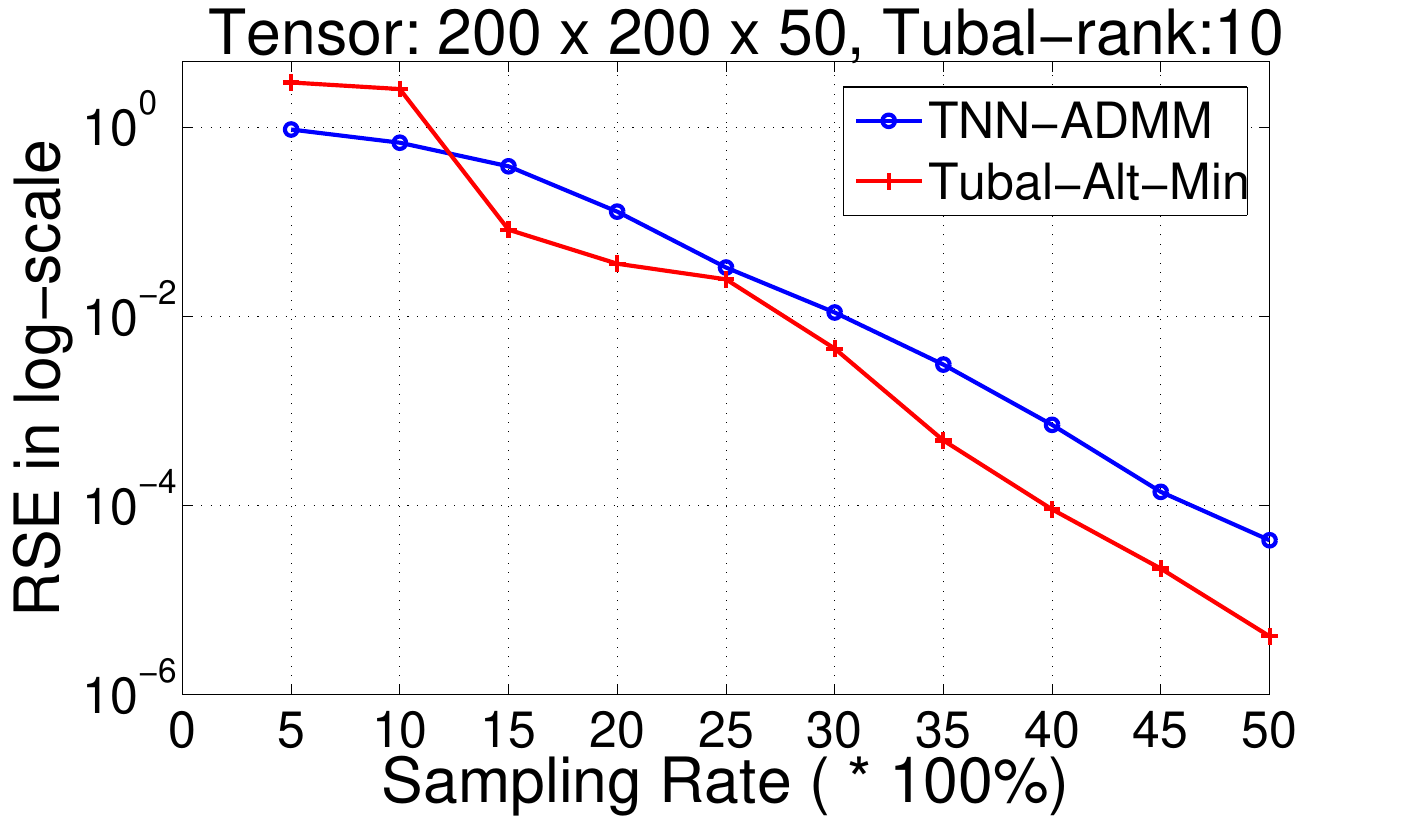}
  \includegraphics[width=0.495\textwidth]{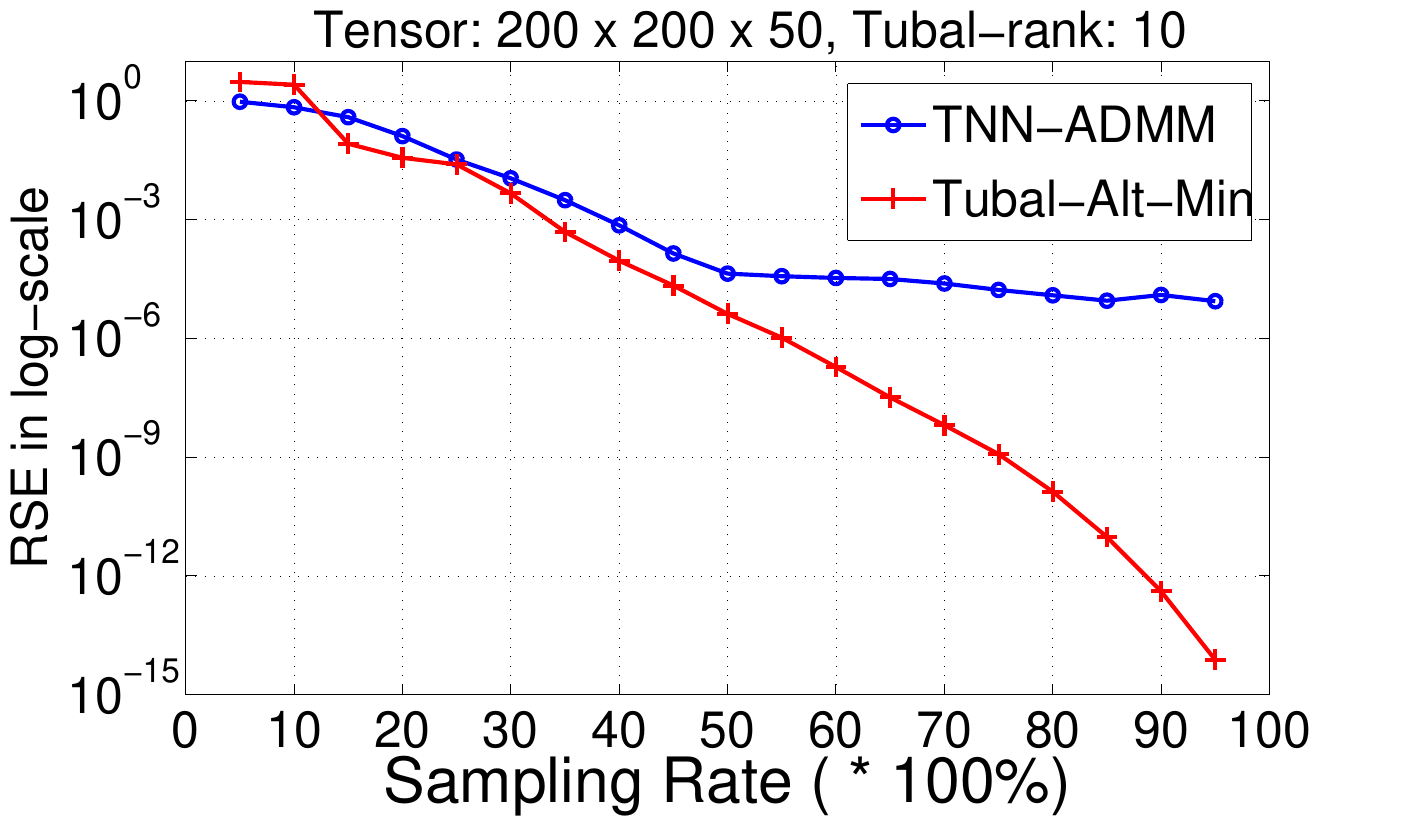}
  \caption{Recovery error RSE in $\log$-scale for different sampling rates.}\label{fig:recovery_error}
  \end{figure}

  \begin{figure}[t]\centering
  \includegraphics[width=0.495\textwidth]{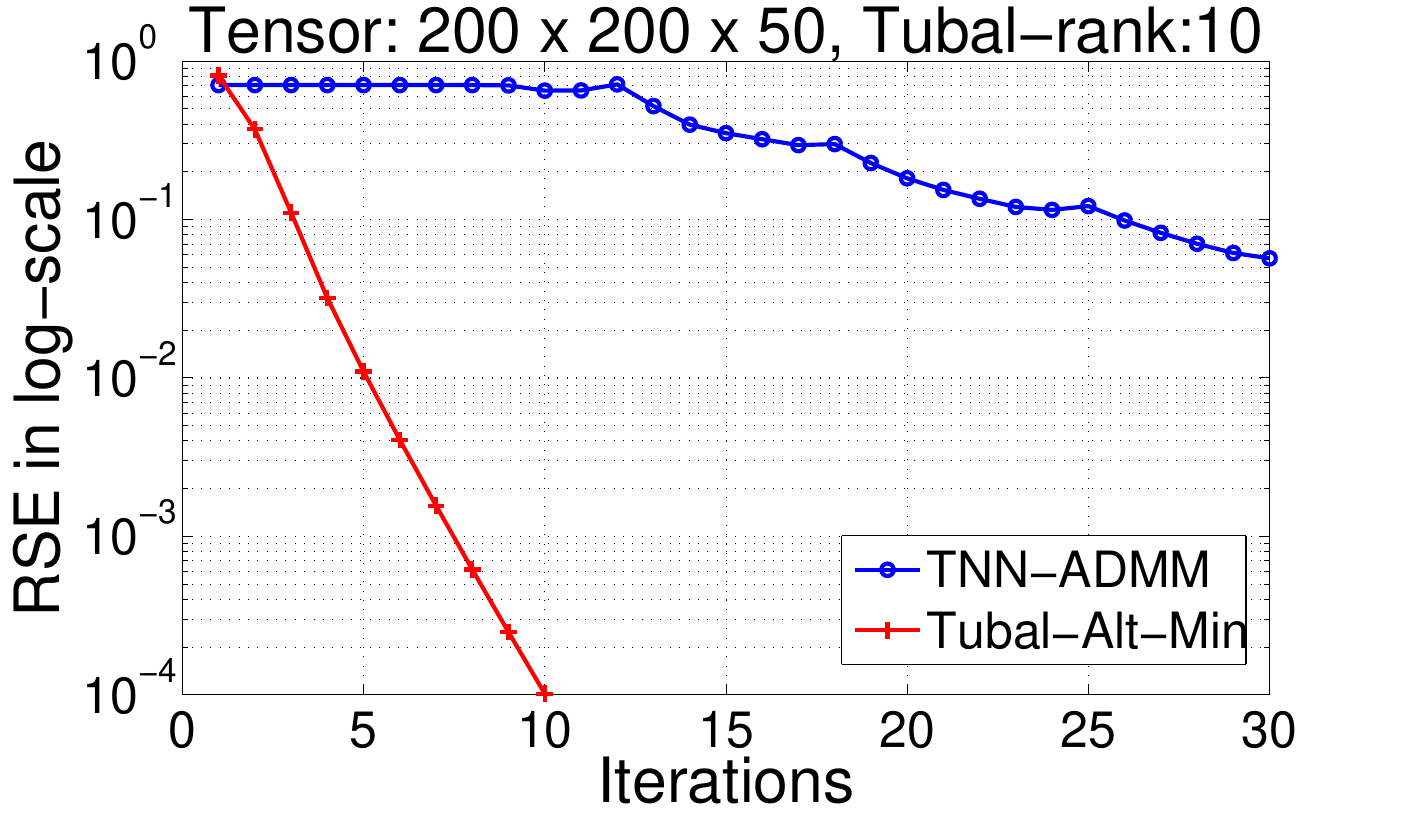}
  \includegraphics[width=0.495\textwidth]{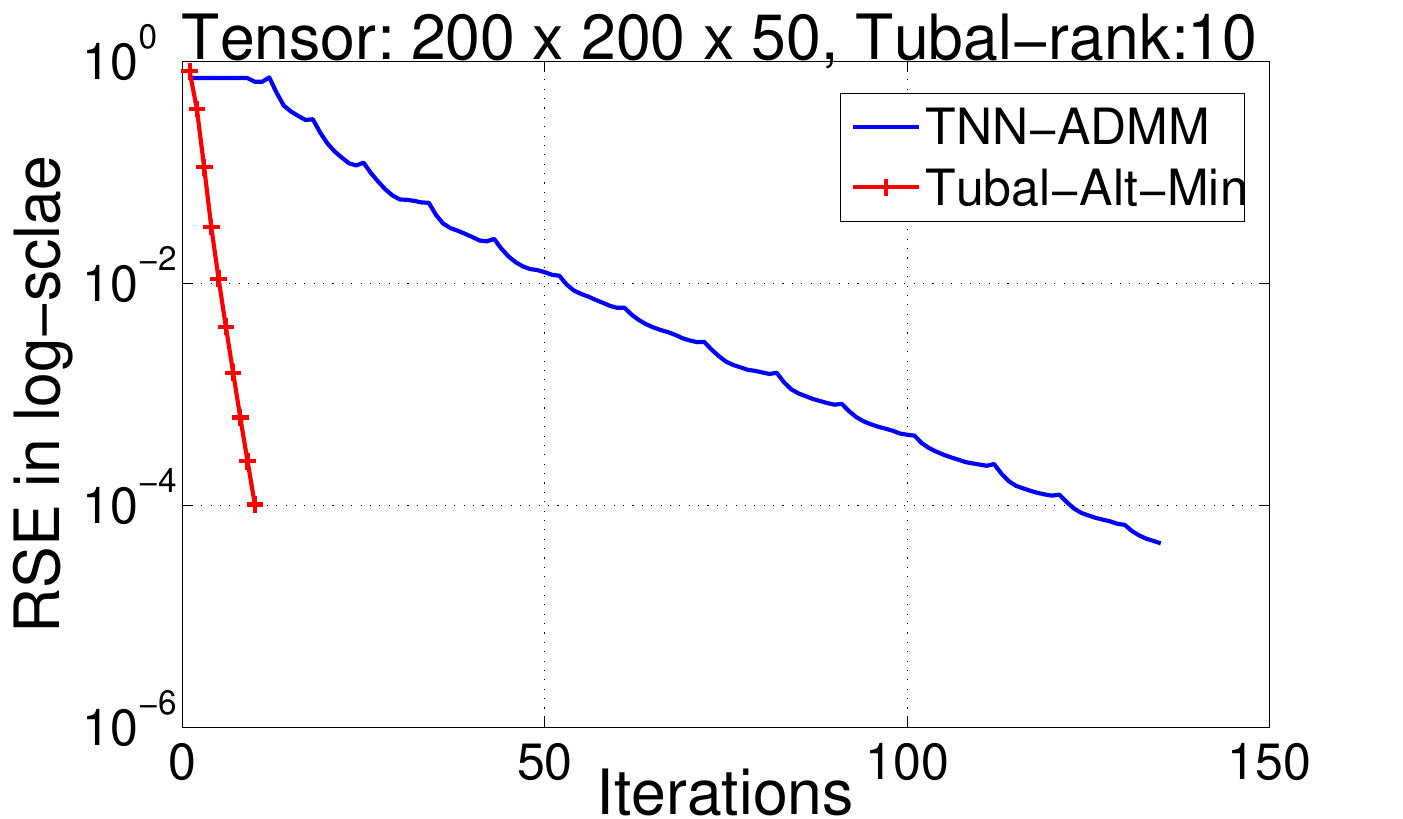}
  \caption{Convergence speed for the two algorithms on synthetic data.}\label{fig:convergence_speed}
  \end{figure}

  \begin{figure}[t]\centering
  \includegraphics[width=0.495\textwidth]{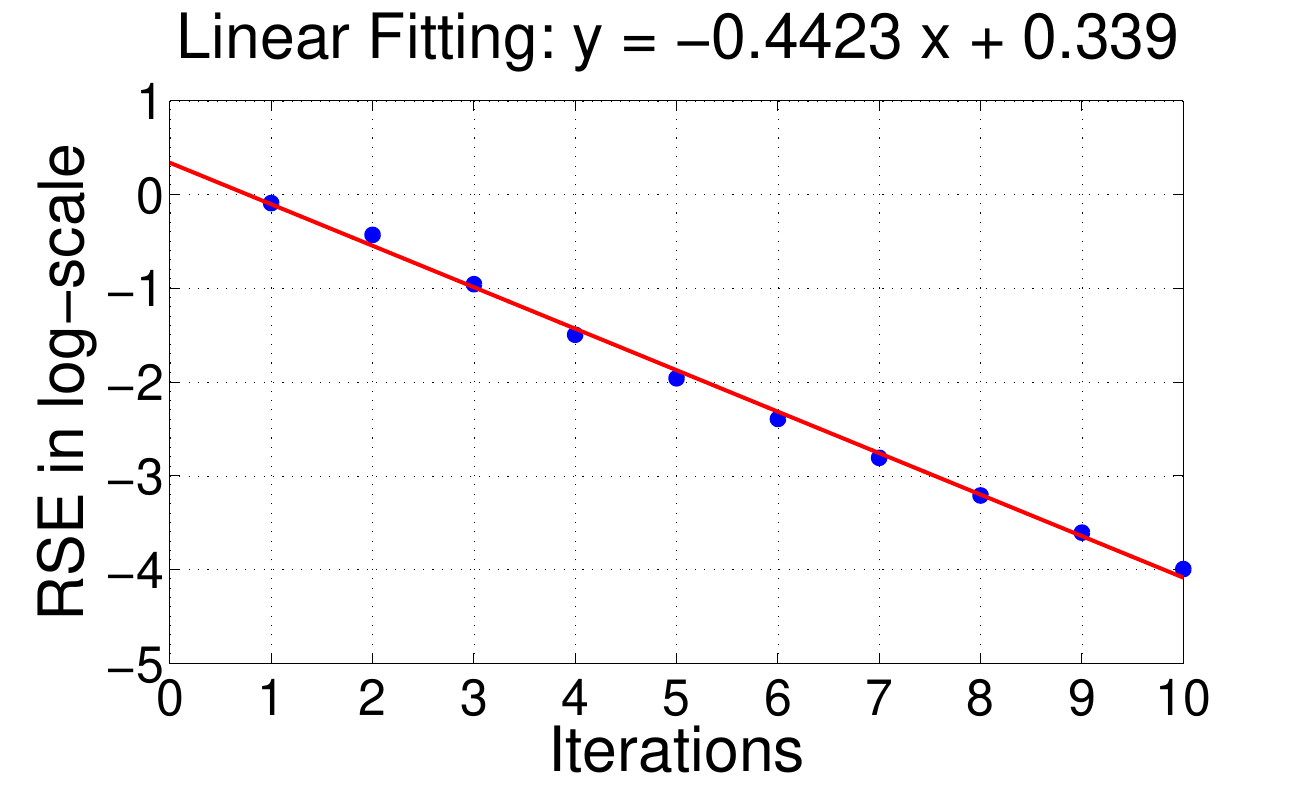}
  \includegraphics[width=0.495\textwidth]{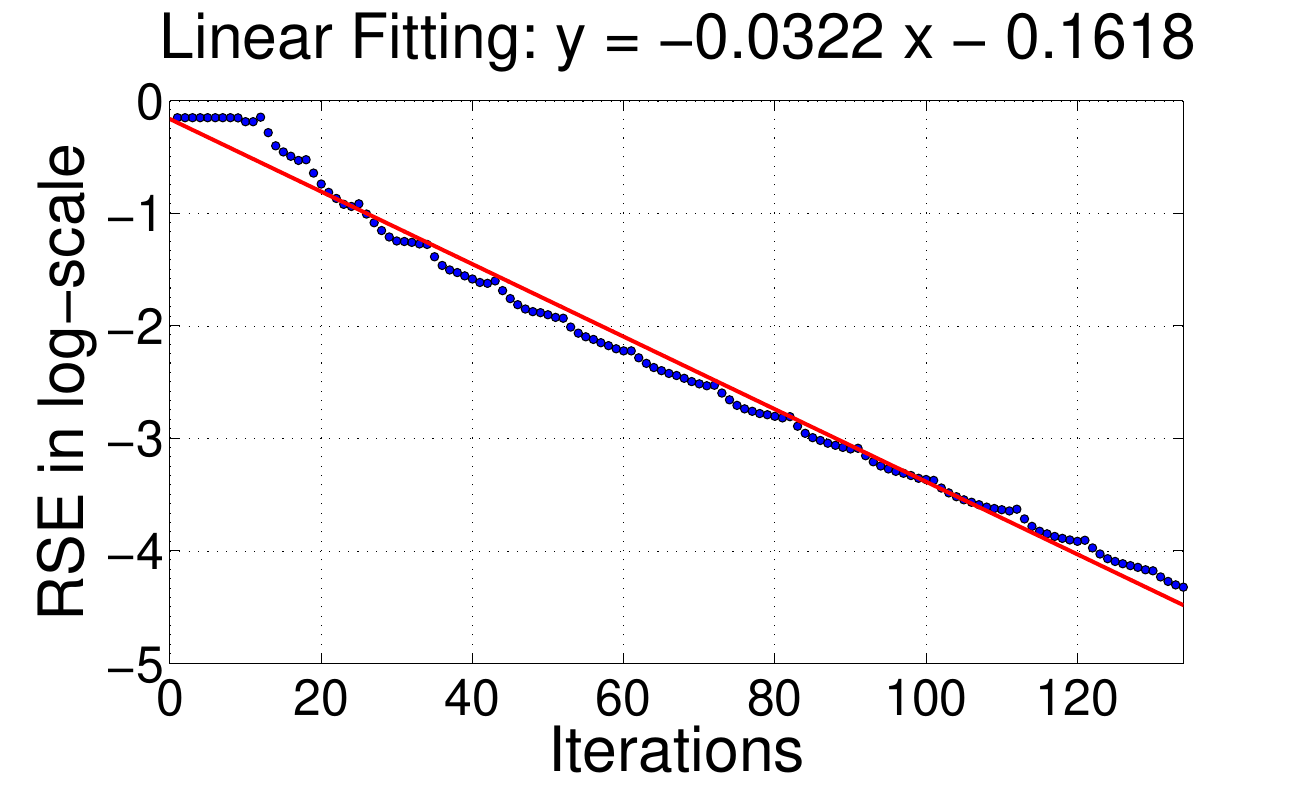}
  \caption{Using linear fitting to estimate the convergence rate. Left: Tubal-Alt-Min. Right: TNN-ADMM.}\label{fig:convergence_rate_fitting}
  \end{figure}

  For recovery error, our input is a low-tubal-rank tensor of size $200 \times 200 \times 50$ and tubal-rank $10$. We first generate two Guassian random tensors of sizes $200 \times 10 \times 50$ and $10 \times 200 \times 50$ and then perform tensor product in Definition \ref{def:tensor_product} to get the input tensor of size $200 \times 200 \times 50$. We set the maximum iteration number to be $10$ for Tubal-Alt-Min and $500$ for TNN-ADMM, respectively. Varying the sampling rate as $5\%,~10\%,...,95\%$ by uniformly selecting entries, we test each sampling rate $5$ times and then plot the average results.

  Fig. \ref{fig:recovery_error} shows the recovery error performance (RSE in $\log$ scale) of Tubal-Alt-Min and TNN-ADMM for varying sampling rates. For a clear comparison, we draw two plots for sampling rate $\leq 50\%$ and sampling rates $5\% \sim 95\%$, respectively. When the sampling rate is higher than $50\%$, the RSE of Tubal-Alt-Min is orders of magnitude lower than that of TNN-ADMM. For sampling rates $15\% \sim 45\%$, the RSE of Tubal-Alt-Min is approximately half of or one order lower than that of TNN-ADMM except for an abnormal case at sampling rate $25\%$. However, for sampling rates $5\% \sim 10\%$, for Tubal-Alt-Min behaves badly since with insufficient samples, Tubal-Alt-Min switches between the factor tensors that are both not well determined. The possible reason that TNN-ADMM has much higher RSE than Tubal-Alt-Min is that TNN-ADMM does not have the exact tubal-rank value $r$ and keeps a higher tubal-rank value to avoid the risk of experiencing higher RSE.

  For convergence speed, the input is a low-tubal-rank tensor of size $200 \times 200 \times 50$ and tubal-rank $10$. We set the maximum iteration number to be $10$ for Tubal-Alt-Min and $500$ for TNN-ADMM. We fix the sampling rate to be $50\%$ and record the RSE in each iteration. TNN-ADMM terminates at the $134$-th iteration because the algorithm detected that the decrease of RSE is lower than a preset threshold.

  Fig. \ref{fig:convergence_speed} shows the decreasing RSEs across the iterations for Tubal-Alt-Min and TNN-ADMM. For a clear comparison, we draw two plots with $30$ and $140$ iterations for TNN-ADMM, respectively.  Clearly, Tubal-Alt-Min decreases much faster than TNN-ADMM. Interestingly, TNN-ADMM behaves very badly during the first $11$ iterations.

  Fig. \ref{fig:convergence_rate_fitting} shows the fitting results for the convergence rates of Tubal-Alt-Min and TNN-ADMM. We use linear functions to fit the observed RSEs (in $\log$-scale) across the iterations. For Tubal-Alt-Min, the fitted function is $y = -0.4423 x + 0.339$ and the estimated convergence rate is $10^{-0.4423 \text{Iter}}$ where $\text{Iter}$ denotes the number of iterations. For TNN-ADMM, the fitted function is $y = -0.0322 x - 0.1618$ and the estimated convergence rate is $10^{-0.0332 \text{Iter}}$. Therefore, Tubal-Alt-Min converges much faster than TNN-ADMM.

  \begin{figure}[t]\centering
  \includegraphics[width=0.5\textwidth]{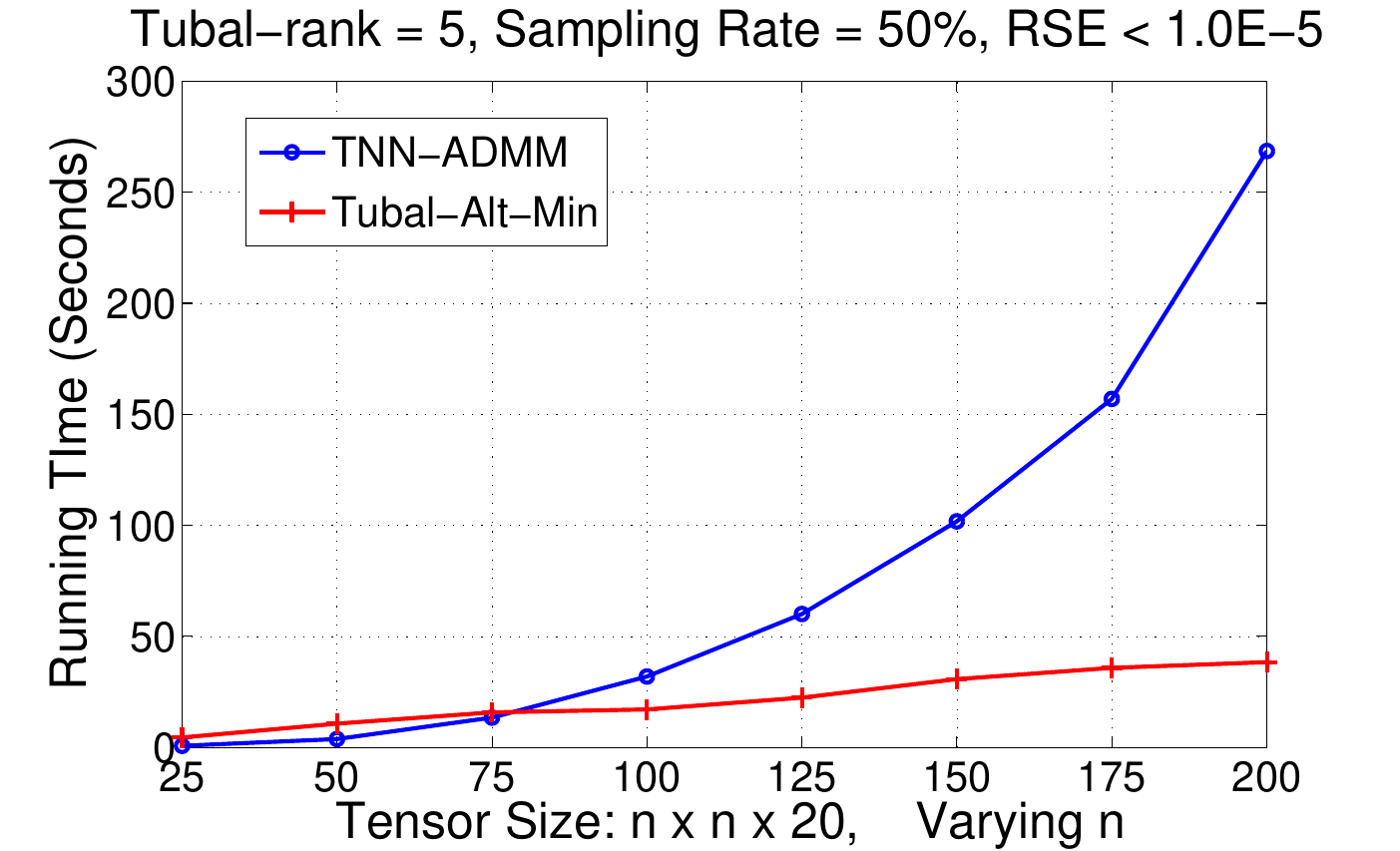}
  \caption{Comparison of running time for varying tensor size.}\label{fig:running_time}
  \end{figure}

  For running time, we measure the CPU time in seconds. We vary the tensor size from $25 \times 25 \times 20$ to $200 \times 200 \times 20$. Note that the tubal-rank tensor model views a tensor as a matrix of vectors, the third-demention $k$ acts as a linear scaling factor, therefore, to test large tensor cases we set $k=20$ that is smaller than $m$ and $n$ to avoid the ERR: out of memory. The tubal-rank is set to be $5$ for all cases, the sampling rate is $50\%$, while the target RSE less than $10^{-5}$. For fairness, we preset a threshold $10^{-5}$ and measure the running time that the algorithms need to reach an RSE less than this threshold.

  Fig. \ref{fig:running_time} shows the running time curves of Tubal-Alt-Min and TNN-ADMM. For tensors larger than $75 \times 75 \times 20$, our Tubal-Alt-Min algorithm beats TNN-ADMM. The acceleration ratio is about $2$ times for tensors of size $125 \times 125\times 20$ and $5$ times for tensors of size $200 \times 200 \times 20$. However, since our implementation of Tubal-Alt-Min in Section \ref{subsect:LS_mini} introduces intermediate matrices of size $nk \times nk \times n$, we encounter the ERR: out of memory and thus we are not able to test larger tensor sizes.

\subsection{Real Data}\label{sect:realdata}

  \begin{figure}[t]\centering
  \includegraphics[width=0.47\textwidth]{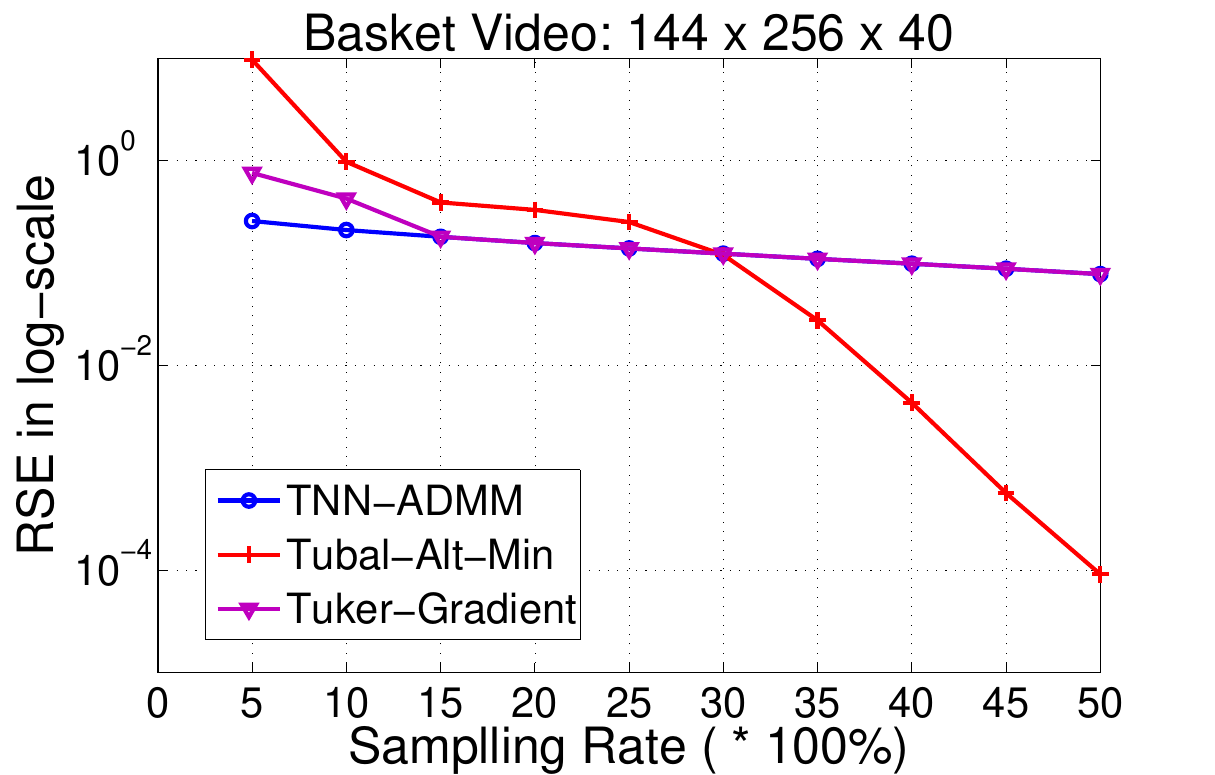}
  \includegraphics[width=0.505\textwidth]{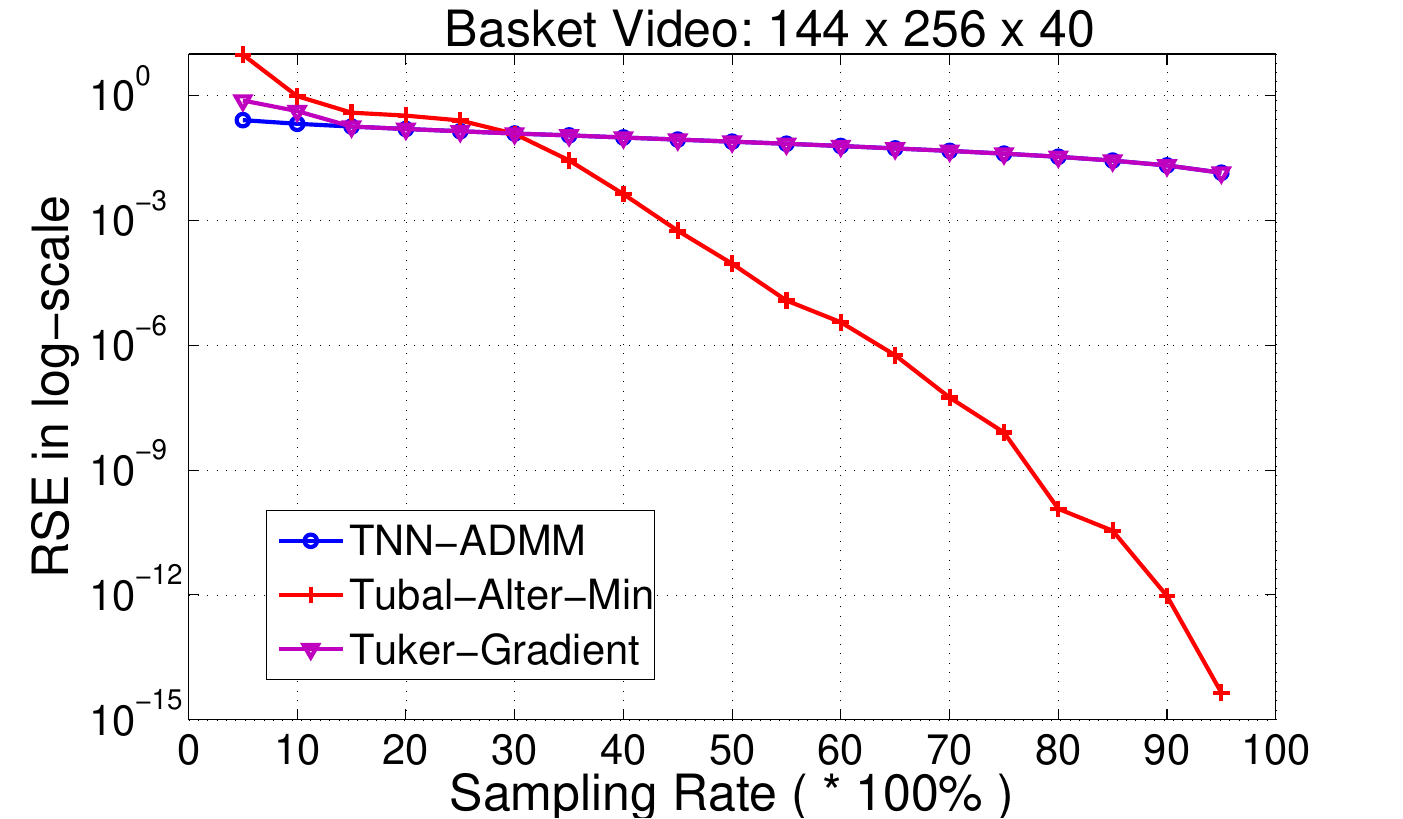}
  \caption{Recovery error RSE in $\log$-scale for different sampling rates.}\label{fig:recovery_error_video}
  \end{figure}

  For the basket video, we only compare the recovery error. One can visually see the improvements of Tubal-Alt-Min by downloading the video results \cite{video}. Since we are interested in fast algorithms and we think it is not fair to compare the running time and convergence speed of our Tubal-Alt-Min algorithm with the Tuker-Gradient \cite{Tuker}\footnote{The Tuker-Gradient \cite{Tuker} implementation tested tensor of size $10,000 \times 10,000 \times 10,000$, which is not possible in our implementation due to the ERR: out of memory for our implementation in Section \ref{subsect:LS_mini}.} that has much higher recovery error.

  Fig. \ref{fig:recovery_error_video} shows the recovery error performance (RSE in $\log$-scale) of Tubal-Alt-Min, TNN-ADMM and Tuker-Gradient for varying sampling rates. For a clear comparison, we draw two plots for sampling rate $\leq 50\%$ and sampling rates $5\% \sim 95\%$, respectively. Tubal-Alt-Min achieves relative lower error by orders of magnitude for sampling rates higher than $35\%$, while the RSE does not decrease much for TNN-ADMM. Here, Tubal-Alt-Min behaves worse than that in Fig. \ref{fig:recovery_error} for sampling rates $5\% \sim 30\%$, and much better for sampling rates $35\% \sim 95\%$. Comparing with Fig. \ref{fig:recovery_error}, Tubal-Alt-Min has similar recovery error for both synthetic data and real video data, while TNN-ADMM's performance for video data is worse than that for synthetic data.  The reasons are: 1) the synthetic input tensor has tubal-rank $10$, while the basket video data has larger tubal-rank (approximately $30$) with the residual error modelled as noise, and 2) Tubal-Alt-Min is more capable of dealing with noise in the real-world video data. Note that in $\log$-scale, the performance of TNN-ADMM and Tuker-Gradient are in the same order and thus indistinguishable.

\section{Conclusions}

   Alternating minimization provides an empirically appealing approach to solve low-rank matrix and tensor completion problems. We provide the first theoretical guarantees on the global optimality for the low-tubal-rank tensor completion problem. In particular, this paper proposes a fast iterative algorithm for the low-tubal-rank tensor completion problem, called {\em Tubal-Alt-Min}, which is based on the alternating minimization approach. The unknown low-tubal-rank tensor is represented as the product of two much smaller tensors with the low-tubal-rank property being automatically incorporated, and Tubal-Alt-Min alternates between estimating those two tensors using tensor least squares minimization. We have obtained theoretical performance guarantees on the success of the proposed algorithm under the tensor incoherency conditions. Based on extensive evaluations and comparisons on both synthetic and real-world data, the proposed algorithm is seen to be fast and accurate.


\appendix

  \begin{figure}[t]\centering
  \includegraphics[width=0.96\textwidth]{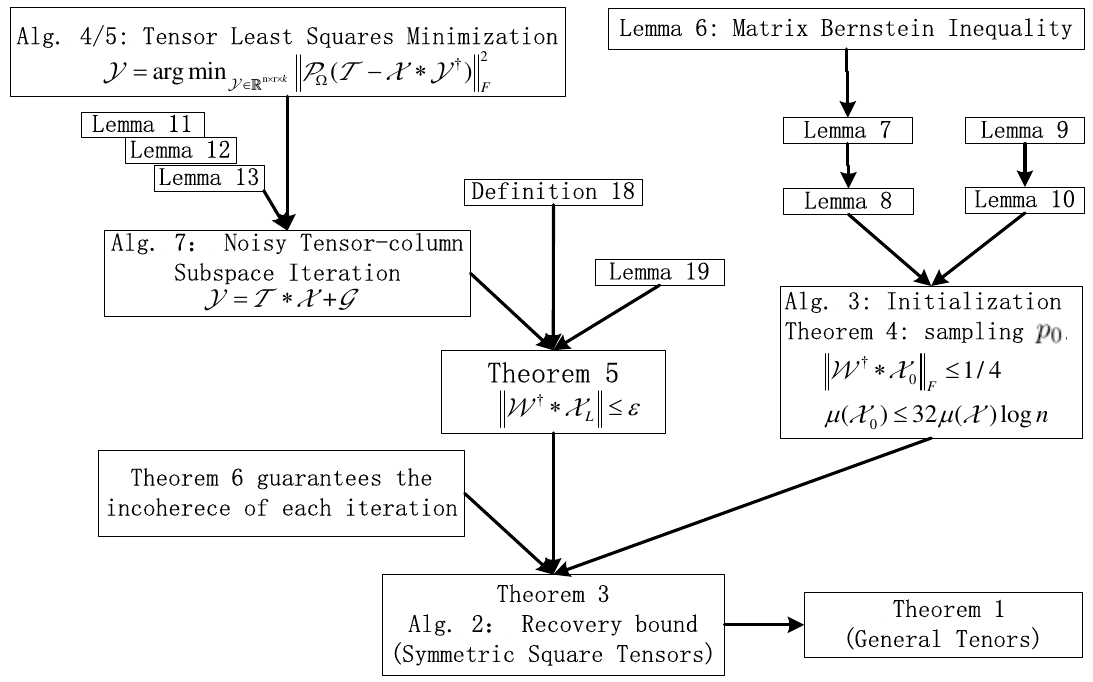}
  \caption{Overview of the proof flow of Theorem \ref{theorem:noisy_case}.}\label{fig:flow_proof}
  \end{figure}

   Fig. \ref{fig:flow_proof} shows the flow chart of our proof for Theorem \ref{theorem:noisy_case}. To guarantee the global optima of the iterative algorithm Alg. \ref{alg_AM}, we follow a two-stage approach. In the first stage, we prove in Theorem \ref{proof:initialization} that Alg. \ref{alg_initialization} (the first step of Alg. \ref{alg_AM}) returns a good initial point that is relatively close to the global optima, while in the second stage, Theorem \ref{theorem:addmissible_noisy_tensors} states that Alg. \ref{alg_median_LS_update} (the second step of Alg. \ref{alg_AM}) converges locally.

   Theorem \ref{proof:initialization} in Appendix \ref{sec:initialization} states that $\mathcal{X}_0$ approximates the tensor-column subspace $\mathcal{U}$ well such that $||(\mathcal{I} - \mathcal{U} * \mathcal{U}^\dagger) * \mathcal{X}_0||_F \leq 1/4$ with probability at least $1 - 1/n^2$, if the elements are included in $\Omega_0$ with probability $p_0 \geq \frac{6144r^2 \mu(\mathcal{U}) (||\mathcal{T}||_F / \gamma_{rk} \overline{\sigma}_{rk})^2 \log n}{n} + \frac{64r^{3/2} \mu(\mathcal{U}) (||\mathcal{T}||_F / \gamma_{rk} \overline{\sigma}_{rk}) \log n}{n}$ where $\overline{\sigma}_{rk}$ denotes the $rk$-th singular value of the block diagonal matrix $\overline{\mathcal{T}}$, and $\gamma_{rk} = 1 - \overline{\sigma}_{rk+1} / \overline{\sigma}_{rk}$.

   The convergence of Alg. \ref{alg_AM} relies on the fact that alternately calling Alg. \ref{alg_median_LS_update} will lead to local convergence. First, we express each iteration as a noisy tensor-column subspace iteration in Lemma \ref{lemma:least_squares} as $\mathcal{Y} = \mathcal{T}*\mathcal{X} + \mathcal{G}$ where $\mathcal{G}$ is the noise/error term. Second, we show that if the error term $\mathcal{G}$ is bounded as in Lemma \ref{lemma:bounded}, then Alg. \ref{alg_AM} converges at an exponential rate as in Lemma \ref{lemma:noisy_tensor}. Third, to guarantee $\mathcal{G}$ is bounded in Frobenius norm, we obtain a requirement that the elements are included in $\Omega_+$ with at least probability $p_{+} = O\left(\frac{r\mu(\mathcal{X})\log{nk}}{\delta^2 n}\right)$ in Lemma \ref{lemma:probability_p}.

   Therefore, the sampling complexity is the sum of two terms: the samples required by the initialization step, and the samples required to bound the error term in the tensor least squares minimization steps, i.e., $p = p_0 + p_+$. Theorem \ref{main_theorm} shows that if the sampling complexity is larger than the sum of these two terms, Alg. \ref{alg_AM} converges to the true unknown tensor rapidly.

\subsection{Additional Definitions and Lemmas}

  For easy description of the performance analysis, we use symmetric square tensor for analysis in the following, i.e., $\mathcal{T} \in \mathbb{R}^{n \times n \times k}$.

  \begin{definition}\label{def:symmetric_tensor}
  \textbf{Square tensor, rectangular tensor, symmetric square tensor}. A tensor $\mathcal{T} \in \mathbb{R}^{n \times n \times k} $ is a {\em square tensor}, and a tensor $\mathcal{X} \in \mathbb{R}^{n \times r \times k}$ is a {\em rectangular tensor}. A {\em symmetric square tensor} $\mathcal{T}$ is a tensor whose frontal slices are symmetric square matrices, i.e., $\mathcal{T}(:,:,i) = \mathcal{T}(:,:,i)^\dagger$. (Note that for a symmetric square tensor $\mathcal{T}$, generally $\mathcal{T} \neq \mathcal{T}^{\dagger}$).
  \end{definition}

  Therefore, a symmetric square tensor $T$ has the t-SVD decomposition $\mathcal{T} = \mathcal{U} * \Theta * \mathcal{U}^{\dagger}$. A tensor-column subspace of $\mathcal{T}$ is the space spanned by the lateral slices of $\mc{U}$ under the t-product, i.e.,
  $\text{t-span}(\mc{U}) = \{ \mathcal{X} = \sum_{s = 1}^{r} \mc{U}(:,s,:) \ast \V{c}_s \in \mathbb{R}^{n \times 1 \times k},~\V{c}_s \in \mb{R}^{1 \times 1 \times k} \}$, where $r$ denotes the tensor tubal-rank.

  \begin{definition}
  \textbf{Tensor basis and the corresponding decomposition}. We introduce two tensor bases \cite{Shuchin2015}. The first one is called {\em column basis} $\dot{e}_i$ of size $n \times 1 \times k$ with only one entry equal to $1$ and the rest equal to $0$. Note that this nonzero entry $1$ will only appear at the $i$-th entry of the first frontal slice of $\dot{e}_i$. Naturally, its transpose $\dot{e}_i^{\dagger}$ is called {\em row basis}. The second tensor basis is called {\em tubal basis} $e_i$ of size $1 \times 1 \times k$ with one entry equal to $1$ and rest equal to $0$. Fig. \ref{fig:two_basis} illustrates these two bases.
  \end{definition}

  \begin{figure}[t]\centering \label{tensor_basis}
  \includegraphics[width=0.3\textwidth]{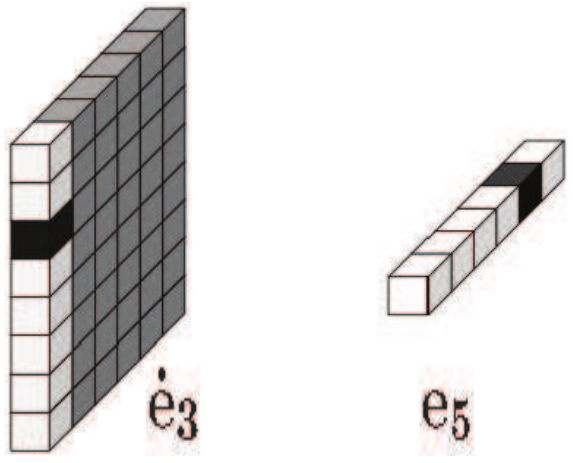}
   \caption{The column basis $\dot{e}_3$ and tubal basis $e_5$. The black entries are $1$, gray and white entries are $0$. The white entries are those that could be $1$.}
   \label{fig:two_basis}
  \end{figure}

  With the above two bases, one can obtain a unit tensor $\mathcal{E}$ with only non-zero entry $\mathcal{E}_{ij\kappa}=1$ as follows:
  \begin{equation}
  \mathcal{E} = \dot{e_i} * e_j * \dot{e}_\kappa^\dagger.
  \end{equation}
  Given any third order tensor $\mathcal{T} \in \mathbb{R}^{m \times n \times k}$, we have the following decomposition
  \begin{equation}
  \mathcal{X} = \sum_{i=1}^{m} \sum_{j=1}^{n} \sum_{\kappa=1}^{k} \mathcal{X}_{ij\kappa} \dot{e_i} * e_j * \dot{e}_\kappa^\dagger.
  \end{equation}

  \begin{definition}\cite{Shuchin2015}
  \textbf{$\ell_{2^*}$, $\ell_{\infty, 2^*}$-norm of tensor}. Let $\mathcal{X}$ be an $ m \times n \times k$ tensor. We define an $\ell_{2^*}$-norm of its $j$-th lateral slice $\mathcal{X}(:,j,:)$ as follows
  \begin{equation}
  ||\mathcal{X}(:,j,:)||_{2^*} = \sqrt{\sum_{i=1}^{n} \sum_{\kappa=1}^{k} \mathcal{X}_{ij\kappa}^2 },~~~ ||\mathcal{X}||_{\infty, 2^*} = \max_{j \in [n]} ||\mathcal{X}(:,j,:)||_{2^*}.
  \end{equation}
  Moreover, we have the following relationship between the $\ell_{2^*}$ norm of $\mathcal{X}(:,j,:)$ and its FFT along the third dimension $\mathcal{X}(:,j,:)$,
  \begin{equation}
  ||\mathcal{X}(:,j,:)||_{2^*} = \frac{1}{\sqrt{k}}||\widetilde{\mathcal{X}}(:,j,:)||_{2^*}.
  \end{equation}
  \end{definition}

  \begin{definition}\label{def:tensor_spectral}\cite{Shuchin2015}
  \textbf{Tensor spectral norm}. The tensor spectral norm $||\mathcal{X}||$ of a third-order tensor $\mathcal{X} \in \mathbb{R}^{m \times n \times k}$ is defined as the largest singular value of $\mathcal{X}$
  \begin{equation}
  ||\mathcal{X}|| = \sup_{\mathcal{L} \in \mathbb{R}^{n \times n \times k},~||\mathcal{L}||_F \leq 1} ||\mathcal{X}*\mathcal{L}||_F.
  \end{equation}
  \end{definition}

  \begin{lemma}\label{lemma:spectral_norm}\cite{Shuchin2015}
  The tensor spectral norm of $\mathcal{X}$ equals to the matrix spectral norm of $\overline{\mathcal{X}}$, i.e.,
  \begin{equation}
  ||\mathcal{X}|| = ||\overline{\mathcal{X}}||.
  \end{equation}
  Note that it also equal to the spectral norm of the circular matrix $X^c$ (defined in Section \ref{sect:circular_algebra}), i.e., $||\mathcal{X}||= ||X^c||=\lambda_{\max}(X^c)$ where $\lambda_{\max}(X^c)$ denotes the largest singular value of $X^c$.
  \end{lemma}

  \begin{definition}
  \textbf{Tensor infinity norm}. The tensor infinity norm $||\mathcal{A}||_{\infty}$ is defined as the largest absolute value of any of its entry, i.e.,
  \begin{equation}
  ||\mathcal{A}||_{\infty} = \max_{i,j,\kappa} |\mathcal{A}_{ij\kappa}|.
  \end{equation}
  \end{definition}

  Note that the tensor-column subspace iteration in Alg. \ref{alg_noisy_subspace_iteration} relies on definitions of inverse, angle function, inner products, norm, conjugate, and also the circulant Fourier transform \cite{Gleich2013}. We describe the inverse and the circulant Fourier transform while omit the rest since they are implicitly used in our paper.  The inverse of $\underline{\alpha} \in \mathbb{K}$ is $\underline{\alpha}^{-1} ~\leftrightarrow~ \text{circ}(\underline{\alpha})^{-1}$, where $\text{circ}(\underline{\alpha})^{-1}$ is also a circulant. The circulant Fourier transforms, $\text{cft}:\underline{\alpha} \in \mathbb{K} \mapsto \mathbb{C}^{k \times k}$ and its inverse $\text{icft}: \mathbb{C}^{k \times k} \mapsto \mathbb{K}$, are defined as follows:
   \begin{equation}
   \text{cft}(\underline{\alpha}) \triangleq\left[
    \begin{array}{ccc}
    \hat{\alpha}_1 &  & \\
    & ... & \\
    & &  \hat{\alpha}_k\\
    \end{array}
    \right] \leftrightarrow \textbf{F}^*\text{circ}(\underline{\alpha})\textbf{F},~~~~~~
    \text{icft}\left( \left[
    \begin{array}{ccc}
     \hat{\alpha}_1 &  & \\
    & ... & \\
    & &  \hat{\alpha}_k\\
    \end{array}
    \right]\right) \triangleq \underline{\alpha} \leftrightarrow \textbf{F}\text{cft}(\underline{\alpha}) \textbf{F}^*,
   \end{equation}
   where $\hat{\alpha}_{\ell}$ are the eigenvalues of $\text{circ}(\underline{\alpha})$ as produced in the Fourier transform order, $\textbf{F}$ is the $k \times k$ discrete Fourier transform matrix, and $\textbf{F}^*$ denotes the (circulant) conjugate \cite{Gleich2013}.

   \textbf{Eigentubes and Eigenslices}: As in \cite{Gleich2013}, we describe the eigentubes and eigenslices. The existence of an eigentube $\underline{\lambda} \in \mathbb{K}_k$ implies the existence of a corresponding eigenslice $\underline{A} \in \mathbb{K}_k^{n \times n}$, satisfying $\underline{A} * \underline{x} =  \underline{x} * \underline{\lambda}$. The corresponding Fourier transforms $\text{cft}(\underline{A}), \text{cft}(\underline{x}), \text{cft}(\underline{\lambda})$ satisfy:
  \begin{equation}\label{eqn_to_freq}
  \begin{split}
  \text{cft}(\underline{A} * \underline{x}) &= \text{cft}(\underline{x} * \underline{\lambda}) \\
  \text{cft}(\underline{A}) \text{cft}(\underline{x}) &= \text{cft}(\underline{x}) \text{cft}(\underline{\lambda}).
  \end{split}
  \end{equation}

   From here on we will always assume that the symmetric square tensor $\mathcal{T}$ has the decomposition $\mathcal{T} = \mathcal{U} * \Theta * \mathcal{U}^{\dagger} + \mathcal{W} * \Theta_{\mathcal{W}} * \mathcal{W}^{\dagger} = \mathcal{M} + \mathcal{N} $, where $\mathcal{M} = \mathcal{U} * \Theta * \mathcal{U}^{\dagger}$ and $\mathcal{N} = \mathcal{W} * \Theta_{\mathcal{W}} * \mathcal{W}^{\dagger} $,  $\mathcal{U} \in \mathbb{R}^{n \times r \times k}, \mathcal{W} \in \mathbb{R}^{n \times (n - r) \times k}$ corresponding to the first $r$ and last $(n-r)$ eigenslices respectively, and $\Theta, \Theta_{\mathcal{W}}$ corresponding to the first $r$ and last $(n - r)$ eigentubes of $\mathcal{T}$. Note that we represent $\mathcal{T}$ as $\mathcal{T} = \mathcal{X} * \mathcal{Y}^\dagger$ in Alg \ref{alg_AM}. The following equalities are frequently used in the proof:
   \begin{equation}
   \begin{split}
   &\mathcal{N} = (\mathcal{I} - \mathcal{U}* \mathcal{U}^{\dagger})* \mathcal{T} = \mathcal{W} * \Theta_{\mathcal{W}} * \mathcal{W}^{\dagger},\\
   &||\mathcal{W}^{\dagger} * \mathcal{X}||= ||(\mathcal{I} - \mathcal{U} * \mathcal{U}^\dagger)*\mathcal{X}||.
   \end{split}
   \end{equation}

\subsection{Proof of Theorem 1}\label{proof:theorem}
   We first prove the theoretical result for a symmetric tensor, and then extend it to Theorem 1 following the argument in Remark \ref{rm_extd}.

   \begin{theorem}\label{main_theorm}
   Suppose that we have a sample set $\Omega$ of size $O(pn^2k)$ with each element randomly drawn from an unknown symmetric square tensor $\mathcal{T}$ of size $n \times n \times k$, where   $\mathcal{T} = \mathcal{M} + \mathcal{N} $,  $ \mathcal{M} = \mathcal{U} * \Theta * \mathcal{U}^\dagger$ has tubal-rank $r$, and $\mathcal{N}=(\mathcal{I} - \mathcal{U}* \mathcal{U}^{\dagger})* \mathcal{T}$ satisfies condition (\ref{condition_N}). Let $\gamma_{rk} = 1 - \overline{\sigma}_{rk+1}/\overline{\sigma}_{rk}$ where $\overline{\sigma}_{rk}$ is the smallest singular value of $\overline{\mathcal{M}}$ and  $\overline{\sigma}_{rk+1}$ is the largest singular value of $\overline{\mathcal{N}}$ ($\overline{\sigma}_{rk+1} = 0, \gamma_{rk} = 1$ for the exact tensor completion problem). Then, there exist parameters $\mu = \Theta(\gamma_{rk}^{-2}r(\mu^* + \log nk))$ and $L = \Theta(\gamma_{rk}^{-1}\log(n/\epsilon))$ such that
    Alg. \ref{alg_AM} will output $(\widehat{\mathcal{X}}, \widehat{\mathcal{Y}})$ with probability at least $1 - \Theta\left(\frac{\gamma_{rk}^{-1}\log(n/\epsilon)}{n^2}\right)$, $\widehat{\mathcal{X}}$ is an orthogonal $n \times r \times k$ tensor that approximates the tensor-column subspace $\mathcal{U}$ as $|| (\mathcal{I} - \mathcal{U} * \mathcal{U}^{\dag}) * \widehat{\mathcal{X}}|| \leq \epsilon $, provided that
    \begin{equation}
    p = O\left(\frac{r^2 \mu^* \left( ||\mathcal{M}||_F^2 + ||\mathcal{N}||_F^2/\epsilon^2 \right) \log(n/\epsilon) \log^2 nk }{\gamma_{rk}^5 \overline{\sigma}_{rk}^2 n}\right).
    \end{equation}
   \end{theorem}

   Before the proof, we state the following remark on the reconstruction error in the Frobenius norm, and also the way of extending the above theoretical results for the symmetric square tensor case to general tensor cases.
   \begin{collary}
   Under the assumptions of Theorem \ref{main_theorm}, the output $(\widehat{\mathcal{X}}, \widehat{\mathcal{Y}})$ of Alg. \ref{alg_AM} satisfies $||\mathcal{M} - \widehat{\mathcal{X}} * \widehat{\mathcal{Y}}^\dagger||_F \leq \frac{3\epsilon}{2} ||\mathcal{T}||_F$.
   \end{collary}
   \begin{proof}
   Let $(\widehat{X}, \widehat{Y})$ be the tensors output by Alg. \ref{alg_AM} when invoked with error parameter $\epsilon$. By Theorem \ref{main_theorm} and Definition \ref{def:tensor_spectral}, we have $||\mathcal{U}*\mathcal{U}^\dagger - \widehat{\mathcal{X}}*\widehat{\mathcal{X}}^\dagger || = || (\mathcal{I} - \mathcal{U} * \mathcal{U}^{\dag}) * \widehat{\mathcal{X}}|| \leq \epsilon$. Lemma \ref{lemma:least_squares} shows that $\widehat{\mathcal{Y}} = \mathcal{T}*\widehat{\mathcal{X}} + \mathcal{G}$ (note that $\widehat{X} = \mathcal{X}_L$), and in the proof of Theorem \ref{main_theorm} we will verify that $\{(\mathcal{X}_{\ell},\mathcal{G})\}_{\ell = 1}^{L}$ is $(\epsilon/4)$-tensor-admissible (Definition \ref{def:tensor_admissible}). Therefore we have $||\mathcal{G}||_F \leq \sqrt{r}\epsilon \overline{\sigma}_{rk}/2$ since $\mathcal{G} \in \mathbb{R}^{n \times r \times k}$ with $\text{rank}(\mathcal{G})=r$ and
   \begin{equation}
  ||\mathcal{G}||_F \leq \sqrt{r}||\mathcal{G}|| \leq \frac{\sqrt{r}}{32} \gamma_{rk} \overline{\sigma}_{rk} ||\mathcal{W}^{\dagger} * \mathcal{X}_{L}|| + \frac{\sqrt{r}\epsilon/4}{32}\gamma_{rk} \overline{\sigma}_{rk} \leq \sqrt{r}\epsilon \overline{\sigma}_{rk}/2,
  \end{equation}
  where $||\mathcal{W}^{\dagger} * \mathcal{X}_{L}|| \leq \epsilon /4$ (plugging $\epsilon /4$ into (\ref{epsilon_4})), and $\gamma_{rk} \leq 1$.

   Therefore, we have
   \begin{equation}
   \begin{split}
   ||\mathcal{M} - \widehat{\mathcal{X}} * \widehat{\mathcal{Y}}^\dagger||_F = ||\mathcal{M} - \widehat{\mathcal{X}} * \widehat{\mathcal{X}}^\dagger * \mathcal{T} - \widehat{\mathcal{X}} * \mathcal{G}||_F & \leq ||\mathcal{U}*\mathcal{U}^\dagger * \mathcal{T} - \widehat{\mathcal{X}} * \widehat{\mathcal{X}}^\dagger * \mathcal{T} ||_F + ||\widehat{\mathcal{X}} * \mathcal{G}||_F \\
   & \leq ||\mathcal{U}*\mathcal{U}^\dagger  - \widehat{\mathcal{X}} * \widehat{\mathcal{X}}^\dagger || || \mathcal{T}||_F + || \mathcal{G}||_F \\
   & \leq \epsilon || \mathcal{T}||_F + \frac{\sqrt{r}\epsilon}{2} \overline{\sigma}_{rk} \leq \frac{3\epsilon}{2} ||T||_F,
   \end{split}
   \end{equation}
   where in the second inequality, we use the relationship that for all tensors $\mathcal{P}$ and $\mathcal{Q}$, we have $||\mathcal{P}*\mathcal{Q}||_F \leq ||\mathcal{P}||||\mathcal{Q}||_F$ (which can be derived from the definition of the spectrum norm in Definition \ref{def:tensor_spectral}).
   \end{proof}

   \begin{remark}\label{rm_extd}
   The result in Theorem \ref{main_theorm} for symmetric square tensors can be directly extended to general tensors as follows. For a general tensors $\mathcal{T} \in \mathbb{R}^{m \times n \times k}$, we can construct a symmetric square tensor
   \begin{equation}
   \mathcal{T}' =\begin{bmatrix}
   0 & \mathcal{T} \\
   \mathcal{T}^\dagger & 0 \\
   \end{bmatrix} \in \mathbb{R}^{(m + n ) \times (m + n) \times k}.
   \end{equation}
   This new tensor $\mathcal{T}'$ has the following property: $\mathcal{T}'$ has tubal-rank $2\text{rank}(\mathcal{T})$ and each eigentube $\Theta(\mathcal{T})(1,1,:), ...,\Theta(\mathcal{T})(r,r,:)$ occurs twice for $\mathcal{T}'$.
   The eigenslices corresponding to a eigentube $\Theta(\mathcal{T})(i,i,:)$ are spanned by the eigenslices $\{(\mathcal{U}(:,i,:),0_{n \times r \times k}), (0_{m \times r \times k}, \mathcal{V}(:,i,:)) \}$. Therefore, an algorithm outputs a tubal-rank $2r$ estimate to $\mathcal{T}'$, which is also a tubal-rank $2r$ estimate to $\mathcal{T}$ with the same recovery error level. Moreover, let $\mathcal{U}'$ denote the tensor-column subspace spanned by the top $2r$ eigenslices of $\mathcal{T}'$, and let $\mathcal{U}, \mathcal{V}$ denote the tensor-column, tensor-row subspaces spanned by the top $r$ left, right eigenslices of $\mathcal{T}$. Then, we have $\mu(\mathcal{U}') \leq \frac{m+n}{2k} \left( \frac{\mu(\mathcal{U})r}{m} + \frac{\mu(\mathcal{V})r}{n} \right) \leq \frac{n+m}{\max(m,n)}\max\{\mu(\mathcal{U}), \mu(\mathcal{V})\}$.

   Therefore, to recover a tubal-rank $r$ general tensor $\mathcal{T}$, one can invoke Alg \ref{alg_AM} with parameters $r/2, \\ \frac{n+m}{\max(m,n)}\mu_0$ for this induced symmetric tensor $\mathcal{T}'$ and keep the other parameters unchanged. Moreover, in the order of sampling complexity, one should change $n$ to $\max(m,n)$ correspondingly.
   \end{remark}

   \begin{proof} [Proof of Theorem 3] \\
   First, Theorem \ref{proof:initialization} concludes that with probability at least $1- 1/n^2$, the initial tensor $\mathcal{X}_0$ satisfies that $|| \mathcal{W}^{\dagger} * \mathcal{X}_0 || \leq 1/4$ and $\mu(\mathcal{X}_0) \leq 32 \mu(\mathcal{U}) \log n$. Assume that this condition holds, then our goal is to apply Theorem \ref{theorem:addmissible_noisy_tensors}, leading to our final bound of recovery error.

   Consider the sequence of tensors $\{(\mathcal{X}_{\ell - 1}, \widetilde{\mathcal{G}}_{\ell}) \}_{\ell = 1}^{L}$ obtained along the iterations of Alg. \ref{alg_AM}. Let $\widetilde{\mathcal{G}}_{\ell} = \mathcal{G}_{\ell} + \mathcal{H}_{\ell}$ where $\mathcal{G}_{\ell}$ is the error term corresponding to the $\ell$-step of MedianLS, and $\mathcal{H}_{\ell}$ is the error term induced by the SmoothQR algorithm in Alg. \ref{alg_smoothQR} at step $\ell$. To apply Theorem \ref{theorem:addmissible_noisy_tensors}, we need to show that this sequence of tensors is $\epsilon/2$-admissible (defined in Definition \ref{def:tensor_admissible}) for Noisy Tensor-Column Subspace Iteration. Then, this theorem directly indicates that $||\mathcal{W}^{\dagger} * \mathcal{X}_L|| \leq \epsilon$ and this would conclude our proof.

   (The $\epsilon/2$-admissible requirement). Let $\tau = \frac{\gamma_{rk}}{128}$, $\hat{\mu} = \frac{C}{\tau^2}(20 \mu^* + \log n)$, and $\mu$ be any number satisfying $\mu \geq \hat{\mu}$. Since $\hat{\mu} = \theta(\gamma_{rk}^{-2}r(\mu^* + \log n))$, it satisfies the requirement of Theorem \ref{theorem:addmissible_noisy_tensors}. We prove that with probability $1 - 1/n^2$, the following three claims hold:
   \begin{itemize}
   \item $\{ (\mathcal{X}_{\ell-1}, \mathcal{G}_{\ell}) \}_{\ell = 1}^{L}$ is $\epsilon/4$-admissible,
   \item $\{ (\mathcal{X}_{\ell-1}, \mathcal{H}_{\ell}) \}_{\ell = 1}^{L}$ is $\epsilon/4$-admissible,
   \item for $\ell \in [L]$, we have $\mu(\mathcal{X}_{\ell}) \leq \mu$.
   \end{itemize}
   If the above three claims hold, then it implies that $\widetilde{\mathcal{G}}_{\ell}$ is $\epsilon/2$-admissible, using a triangle inequality as $\widetilde{\mathcal{G}}_{\ell} = \mathcal{G}_{\ell} + \mathcal{H}_{\ell}$.

   To prove these three claims, we apply a mutual induction approach. For $\ell = 0$, it only requires to check the third claim which follows from Theorem \ref{proof:initialization} that $\mathcal{X}_0$ satisfies the incoherence bound. Now let us assume that all three claims hold at step $\ell -1$, our goal is to argue that with probability $1-1/n^2$, all three claims will hold at step $\ell$.

   The first claim will hold from Lemma \ref{lemma:bounded} using the induction hypothesis of the third claim that $\mu(\mathcal{X}_{\ell}) \leq \mu$. Specifically, the parameters should be set properly as follows. Let $\delta = c \min\{ \frac{\gamma_{rk}\overline{\sigma}_{rk}}{||\mathcal{M}||_F}, \frac{\epsilon \gamma_{rk}\overline{\sigma}_{rk}}{||\mathcal{N}||_F} \}$ for sufficiently small constant $c > 0$. The lemma requires the lower bound $ p_{\ell} \geq \frac{r\mu \log nk}{\delta^2 n}$, and $ p_{\ell} = \frac{r^2 \mu^* \log^2 nk}{\delta^2 n \gamma_{rk}^2}$ when $\mu = \Theta(\gamma_{rk}^{-2}r(\mu^* + \log nk))$ and $\mu^* = \max\{\mu(\mathcal{U}), \mu_N, \log nk \}$. Therefore, Lemma \ref{lemma:bounded} states that with probability $1 - 1/n^3$, the upper bound $||\mathcal{G}_{\ell}||_F \leq \frac{1}{4}\left( \frac{1}{32} \gamma_{rk} \overline{\sigma}_{rk} ||\mathcal{W}^{\dagger} * \mathcal{X}_{\ell-1}|| + \frac{\epsilon}{32}\gamma_{rk} \overline{\sigma}_{rk} \right)$, satisfying the $\epsilon/4$-admissible condition. This results in the probability $p_{+}$ satisfying:
   \begin{equation}
   p_{+} = \sum_{\ell =1}^{L} p_{\ell} = O\left(\frac{r^2 \mu^* \left( ||\mathcal{M}||_F^2 + ||\mathcal{N}||_F^2/\epsilon^2 \right) \log(n/\epsilon) \log^2 nk }{\gamma_{rk}^5 \overline{\sigma}_{rk}^2 n}\right),
   \end{equation}
   where $L = \Theta(\gamma_{rk}^{-1} \log n/\epsilon)$.

   The remaining two lemmas follow from Theorem \ref{thorem_incoherence}. We will apply this theorem to $\mathcal{T}*\mathcal{X}_{\ell} + \mathcal{G}_{\ell}$ with $\upsilon = \overline{\sigma}_{rk}(||\mathcal{W}^{\dagger} * \mathcal{X}_{\ell-1}|| + \epsilon)$ and $\tau$ as above. Since $||\mathcal{N}*\mathcal{X}_{\ell - 1}|| \leq \overline{\sigma}_{rk} ||\mathcal{W}^{\dagger} * \mathcal{X}_{\ell}||$, it holds that $\upsilon \geq \max\{ ||\mathcal{G}_{\ell}||, ||\mathcal{N}*\mathcal{X}_{\ell - 1}||  \}$ as required by Theorem \ref{thorem_incoherence}. and it also requires a lower bound $\mu$.  To satisfy the lower bound, we combing Lemma \ref{VI_3} states that with probability $1 - 1/n^2$, we have $\frac{1}{\upsilon}(\rho(\mathcal{G}) + \rho(\mathcal{N} * \mathcal{X})) \leq 10 \mu^*$. The SmoothQR process produces with probability at least $1 - 1/n^4$ a tensor $\mathcal{H}_{\ell}$ such that $||\mathcal{H}_{\ell}|| \leq \tau\upsilon \leq \frac{\gamma_{rk}\upsilon}{128} \leq \frac{1}{4}\left( \frac{1}{32} \gamma_{rk} \overline{\sigma}_{rk} ||\mathcal{W}^{\dagger} * \mathcal{X}_{\ell-1}|| + \frac{\epsilon}{32}\gamma_{rk} \overline{\sigma}_{rk} \right)$, satisfying the $\epsilon/4$-admissible condition. Therefore, the second and third claim hold.

   Note that all error probabilities that incurred were less than $1/n^2$, then we sum up the error probabilities over all $L = \Theta(\gamma_{rk}^{-1}\log(n/\epsilon))$ steps, getting a probability at least $1 - \Theta\left(\frac{\gamma_{rk}^{-1}\log(n/\epsilon)}{n^2}\right)$.

   The resulting probability $p =  p_0 + p_+$ would be
   \begin{equation}
   \begin{split}
   p_0 \geq &~ \frac{6144r^2 \mu(\mathcal{U}) (||\mathcal{T}||_F / \gamma_{rk} \overline{\sigma}_{rk})^2 \log n}{n} + \frac{64r^{3/2} \mu(\mathcal{U}) (||\mathcal{T}||_F / \gamma_{rk} \overline{\sigma}_{rk}) \log n}{n}, \\
   p_0 = &~ O\left( \frac{r^2 \mu(\mathcal{U}) (||\mathcal{T}||_F / \gamma_{rk} \overline{\sigma}_{rk})^2 \log n}{n}\right), \\
   p_+ = &~ O\left(\frac{r^2 \mu^* \left( ||\mathcal{M}||_F^2 + ||\mathcal{N}||_F^2/\epsilon^2 \right) \log(n/\epsilon) \log^2 nk }{\gamma_{rk}^5 \overline{\sigma}_{rk}^2 n}\right), \\
   p = &~ O\left(\frac{r^2 \mu^* \left( ||\mathcal{M}||_F^2 + ||\mathcal{N}||_F^2/\epsilon^2 \right) \log(n/\epsilon) \log^2 nk }{\gamma_{rk}^5 \overline{\sigma}_{rk}^2 n}\right),
   \end{split}
   \end{equation}
   where $\mu^* = \max\{\mu(\mathcal{U}), \mu_N, \log nk \}$.

   The proof of Theorem 3 concludes.

   \end{proof}

   \begin{lemma}\label{VI_3}
   Under the assumptions of Theorem \ref{main_theorm}, we have for every $\ell \in [L]$ and $\upsilon = \frac{\overline{\sigma}_{rk}}{32}\left( ||\mathcal{W}^\dagger * \mathcal{X}_{\ell-1}|| + \epsilon \right)$ with probability $1 - 1/n^2$, $\frac{1}{\upsilon}(\rho(\mathcal{G}) + \rho(\mathcal{N} * \mathcal{X}_{\ell})) \leq \mu^*$.
   \end{lemma}
   \begin{proof}
   Given the lower bound on $p$ in Theorem \ref{main_theorm}, we apply Lemma \ref{lemma:bounded} to conclude that $||\dot{e}_i^\dagger * \mathcal{G}_{\ell}^{\mathcal{M}}|| \leq \sqrt{r\mu(\mathcal{U})/n}~\upsilon$ and $||\dot{e}_i^\dagger * \mathcal{G}_{\ell}^{\mathcal{N}}|| \leq \sqrt{\mu^*/n}~ \upsilon$. Thus, $\rho(\mathcal{G}_{\ell})/\upsilon^2 \leq \mu^*$.

   Further, we claim that $||\dot{e}_i^\dagger * \mathcal{N} * \mathcal{X}||^2 \leq (\mu^* /n) \overline{\sigma}_{rk}||\mathcal{W}*\mathcal{U}||$ for all $i \in [n]$, since
   \begin{equation}
   ||\dot{e}_i^\dagger * \mathcal{N} * \mathcal{X}|| \leq ||\dot{e}_i^\dagger * \mathcal{W}*\Theta_{\mathcal{W}}|| ~ || \mathcal{W}^{\dagger} * \mathcal{X}_{\ell} || = ||\dot{e}_i^\dagger \mathcal{N}|| ~ || \mathcal{W}^{\dagger} * \mathcal{X}_{\ell} ||,
   \end{equation}
   where we used the following fact that
   \begin{equation}
   ||\dot{e}_i^\dagger * \mathcal{N}||^2 = ||\dot{e}_i^\dagger * \mathcal{N} * \mathcal{W}||^2 + ||\dot{e}_i^\dagger * \mathcal{N} * \mathcal{U}||^2 = ||\dot{e}_i^\dagger * \mathcal{N} * \mathcal{W}||^2 = ||\dot{e}_i^\dagger * \mathcal{W}*\Theta_{\mathcal{W}}|| ~ || \mathcal{W}^{\dagger} * \mathcal{X}_{\ell} ||^2.
   \end{equation}
   Combining (\ref{condition_N}), this shows that $\rho(\mathcal{N} * \mathcal{X}_{\ell-1})/\upsilon^2 \leq \mu^*$, and conclude the proof.

   \end{proof}

\subsection{Initialization}
\label{sec:initialization}

  Alg. \ref{alg_initialization} computes the top-$r$ eigenslices of $\mathcal{P}_{\Omega_0}(\mathcal{T})$, and truncates them in a ``scaling manner" to ensure the incoherence. This initialization procedure serves as an acceleration of our main algorithm Alg. \ref{alg_AM}. We analyze Alg. \ref{alg_initialization} and derive the required sampling probability $p_0$ to get a good starting point in Theorem \ref{proof:initialization}. The corresponding proofs relies mainly on the matrix Bernstein inequality in Lemma \ref{Bernstein} and the Davis-Kahan $\sin \theta$-theorem \cite{DK} in Lemma \ref{proof:initialization_p}.

  \begin{theorem}\label{proof:initialization}
  Let $\mathcal{T} \in \mathbb{R}^{n \times n \times k}$ be a symmetric square tensor with tubal-rank $r$. Assume that each element is included in $\Omega$ independently with probability
  \begin{equation}\label{p_0_initial}
  p_0 \geq \frac{6144r^2 \mu(\mathcal{U}) (||\mathcal{T}||_F / \overline{\gamma}_{rk} \overline{\sigma}_{rk})^2 \log n}{n} + \frac{64r^{3/2} \mu(\mathcal{U}) (||\mathcal{T}||_F / \overline{\gamma}_{rk} \overline{\sigma}_{rk}) \log n}{n}
  \end{equation}
  where $\overline{\sigma}_{rk}$ denotes the $rk$-th singular value of the block diagonal matrix $\overline{\mathcal{T}}$, and $\overline{\gamma}_{rk} = 1 - \overline{\sigma}_{rk+1} / \overline{\sigma}_{rk}$. Then, Alg. \ref{alg_initialization} returns an orthonormal tensor $\mathcal{X}_0 \in \mathbb{R}^{n \times r \times k}$ such that with probability at least $1 - 1/n^2$,
  we have
  \begin{equation}
  ||\mathcal{W}^{\dagger} * \mathcal{X}_0||_F \leq 1/4, ~~~\text{and}~~~\mu(\mathcal{X}_0) \leq 32 \mu(\mathcal{U}) \log n.
  \end{equation}
  \end{theorem}

  \begin{proof}
  The proof follows directly from Lemma \ref{lemma:probability}, Lemma \ref{proof:initialization_p}, Lemma \ref{Trancation}, and Lemma \ref{lemma:incoherence}.
  \end{proof}

  \begin{lemma}\label{Bernstein} \cite{Berstein}
  (Matrix Bernstein Inequality) Consider a finite sequence $\{ Z_i \}$ of independent random matrices with dimensions $d_1 \times d_2$. Assume that each random matrix satisfies $\mathbb{E} Z_i = 0$ and $|| Z_i || \leq R$ almost surely. Define $\zeta^2 \triangleq \max \{ ||\sum_{i} \mathbb{E} Z_i Z_i^\dagger ||, ||\sum_{i} \mathbb{E} Z_i^\dagger Z_i || \}$. Then, for all $u \geq 0$,
  \begin{equation}
  \mathbb{P}\left\{ \left\| \sum_{i} Z_i \right\| \geq u \right\} \leq (d_1 + d_2) \exp\left\{\frac{- u^2 /2}{\zeta^2 + Ru/3}\right\}.
  \end{equation}
  \end{lemma}

  \begin{lemma}\label{lemma:probability}
  Suppose that $\mathcal{T} \in \mathbb{R}^{n \times n \times k}$ and let $\Omega \in [n] \times [n] \times [k]$ be a random subset with each entry being included independently with probability $p_0$. Then
  \begin{equation}\label{probability_no}
  \mathbb{P}\left\{ ||\mathcal{P}_{\Omega}(\mathcal{T}) - \mathcal{T}|| > u\right\} \leq n \exp\left\{ \frac{- u^2/2}{\zeta^2 + \frac{u}{3p_0} ||\mathcal{T}||_{\infty}} \right\}
  \end{equation}
  where $\zeta^2 = 1/p_0 \max\left\{ ||\mathcal{T}||_{\infty,2^*}^2 , ||\mathcal{T}||_{\infty}\right\}$.
  \end{lemma}

  \begin{proof}
  Define a random variable $\xi_{ij\ell} = 1_{(i,j,\ell) \in \Omega}$ where $1_{(\cdot)}$ is the indicator function. Consider the sum of independent random tensors $\mathcal{P}_{\Omega}(\mathcal{T}) - \mathcal{T} = \sum_{i,j,\ell} (\frac{\xi_{ij\ell}}{p_0} - 1) \mathcal{T}_{ij\ell} \dot{e}_i * e_{\ell} * \dot{e}_j^{\dagger}$, where $\dot{e}_i \in \mathbb{R}^{n \times 1 \times k}$ is the column basis with $\dot{e}_i(i,1,1) = 1$, and $e_{\ell} \in \mathbb{R}^{1 \times 1 \times k}$ is the tube basis with $e_{\ell} (1,1,\ell) =1$.

  In the following, we borrow idea from \cite{Shuchin2015} (Appendix C Proof of Proposition 4.1 Condition 2) to get the intermediate results need by Lemma \ref{Bernstein}. Define $\mathcal{E}_{ij\ell} = (\frac{\xi_{ij\ell}}{p_0} - 1) \mathcal{T}_{ij\ell} \dot{e}_i * e_{\ell} * \dot{e}_j^{\dagger}$, and $\overline{\mathcal{E}_{ij\ell}} = (\frac{\xi_{ij\ell}}{p_0} - 1) \mathcal{T}_{ij\ell} \overline{\dot{e}}_i \overline{e}_{\ell}  \overline{\dot{e}}_j^{\dagger}$. Notice that $\mathbb{E}[\overline{\mathcal{E}}_{ij\ell}] = 0$ and $|| \overline{\mathcal{E}}_{ij\ell} || \leq \frac{1}{p_0} ||\mathcal{T}||_{\infty}$.
  \begin{equation}
  \begin{split}
  &\left\|\mathbb{E}\left[\sum_{i,j,\ell} \overline{\mathcal{E}}_{ij\ell}^\dagger \overline{\mathcal{E}}_{ij\ell} \right] \right\| = \left\|\mathbb{E}\left[\sum_{i,j,\ell} \mathcal{E}_{ij\ell}^\dagger * \mathcal{E}_{ij\ell} \right] \right\| \\
  &= \left\| \sum_{ij\ell} \mathcal{T}_{ij\ell}^2 \dot{e}_j * \dot{e}_j^\dagger \mathbb{E}\left(\frac{1}{p_0}\xi_{ij\ell} - 1\right)^2 \right\| \\
  &= \left\| \frac{1-p_0}{p_0} \sum_{ij\ell} \mathcal{T}_{ij\ell}^2 \dot{e}_j * \dot{e}_j^\dagger \right\|
  \end{split}
  \end{equation}
  since $\dot{e}_j * \dot{e}_j^\dagger$ will return a zero tensor except for $(j,j,1)$-th entry equaling $1$, we have
  \begin{equation}
  \left\| \mathbb{E}\left[\sum_{i,j,\ell} \overline{\mathcal{E}}_{ij\ell}^\dagger \overline{\mathcal{E}}_{ij\ell} \right] \right\| = \frac{1-p_0}{p_0} \max_{j} \left|\sum_{i,\ell} \mathcal{T}_{ij\ell}^2 \right| \leq \frac{1}{p_0}||\mathcal{T}||_{\infty,2^*}^2,
  \end{equation}
  And similarly, $\left\|\mathbb{E}\left[\sum_{i,j,\ell} \overline{\mathcal{E}}_{ij\ell} \overline{\mathcal{E}}_{ij\ell}^\dagger \right] \right\| \leq \frac{1}{p_0}||\mathcal{T}||_{\infty}$ which is bounded. Then pluging into Lemma \ref{Bernstein} concludes the proof.

  \end{proof}

  \begin{lemma}\label{proof:initialization_p}
  To assure that $||\mathcal{P}_{\Omega}(\mathcal{T}) - \mathcal{T}|| \leq \frac{\overline{\gamma}_{rk} \overline{\sigma}_{rk}}{32\sqrt{r}}$ holds with high probability at least $1- \frac{1}{n^2}$, it requires that
  \begin{equation}\label{p_0_initial}
  p_0 \geq \frac{6144r^2 \mu(\mathcal{U}) (||\mathcal{T}||_F / \overline{\gamma}_{rk} \overline{\sigma}_{rk})^2 \log n}{n} + \frac{64r^{3/2} \mu(\mathcal{U}) (||\mathcal{T}||_F / \overline{\gamma}_{rk} \overline{\sigma}_{rk}) \log n}{n}
  \end{equation}
  where $\overline{\sigma}_{rk}$ denotes the $rk$-th singular value of the block diagonal matrix $\overline{\mathcal{T}}$, and $\overline{\gamma}_{rk} = 1 - \overline{\sigma}_{rk+1} / \overline{\sigma}_{rk}$. Then, we have
  \begin{equation}
  \mathbb{P}\left\{ ||\mathcal{W}^\dagger * \mathcal{A}|| \leq \frac{1}{16 \sqrt{r}} \right\} > 1 - 1/n^2,
  \end{equation}
  where $\mathcal{A}$ is the tensor in Alg. \ref{alg_initialization}.
  \end{lemma}

  \begin{proof}
  We have that
  \begin{equation}
  \begin{split}
  ||\mathcal{T}||_{\infty,2^*}^2 &= \frac{r\mu(\mathcal{U})}{n} ||\mathcal{T}||_F^2, \\
  ||\mathcal{T}||_{\infty} &= \frac{r\mu(\mathcal{U})}{n} ||\mathcal{T}||_F.
  \end{split}
  \end{equation}
  Then $\zeta^2$ in Lemma \ref{lemma:probability} becomes
  \begin{equation}
  \zeta^2 = \frac{1}{p_0} \max \left\{ \frac{r \mu(\mathcal{U})}{n} ||\mathcal{T}||_F^2, \frac{r\mu(\mathcal{U})}{n} ||\mathcal{T}||_F \right\} = \frac{r \mu(\mathcal{U})}{p_0n} ||\mathcal{T}||_F^2.
  \end{equation}

  Set the right hand side of (\ref{probability_no}) to be $\leq n^{-(C-1)}$, then taking log-function we have:
  \begin{equation}
  \begin{split}
  &-\frac{u^2}{2} \leq -C \left(\zeta^2 + \frac{u}{3p_0} ||\mathcal{T}||_{\infty}\right) \log n = -C \left(\frac{r \mu(\mathcal{U})}{p_0n} ||\mathcal{T}||_F^2 + \frac{u}{3p_0} \frac{r\mu(\mathcal{U})}{n} ||\mathcal{T}||_F\right) \log n,\\
  & u^2 - 2C\log n \frac{r\mu(\mathcal{U})}{3p_0n} ||\mathcal{T}||_F ~u - 2C\log n \frac{r \mu(\mathcal{U})}{p_0n} ||\mathcal{T}||_F^2 \geq 0.
  \end{split}
  \end{equation}
  This can be re-arranged to get
  \begin{equation}
  p_0 \geq  \frac{2Cr\mu(\mathcal{U}) \log n}{n} ~\frac{||\mathcal{T}||_F^2}{u^2} +  \frac{2Cr\mu(\mathcal{U}) \log n}{n} ~\frac{||\mathcal{T}||_F}{u}.
  \end{equation}

  Set $u = \frac{\gamma_{rk}^c \sigma_{rk}^c}{32\sqrt{r}}$ ($\gamma_{rk}^c$ and $\sigma_{rk}^c$ are introduced for reasons to be clear in (\ref{sin_inequelity})),
  leading to the condition that $p_0 \geq \frac{2048Cr^2 \mu(\mathcal{U}) (||\mathcal{T}||_F / \gamma_{rk}^c \sigma_{rk}^c)^2 \log n}{n} + \frac{64Cr^{3/2} \mu(\mathcal{U}) (||\mathcal{T}||_F / \gamma_{rk}^c \sigma_{rk}^c) \log n}{3n}$. Let $C=3$,
  plug in the above parameters into Lemma \ref{lemma:probability}, we get:
  \begin{equation}
  \mathbb{P}\left\{ ||\mathcal{P}_{\Omega}(\mathcal{T}) - \mathcal{T}|| > \frac{\gamma_{rk}^c \sigma_{rk}^c}{32\sqrt{r}} \right\} \leq \frac{1}{n^2}.
  \end{equation}

  Let $\mathcal{W}$ be the top $r$ eigenslices of $\mathcal{P}_{\Omega}(\mathcal{T})$, $\sigma_{rk}^c$ denote the $rk$-th singular value of the circular matrix $T^c$ and define $\gamma_{rk}^c = 1 - \sigma_{rk+1}^c / \sigma_{rk}^c $. Now let us assume that $||\mathcal{T} - \mathcal{P}_{\Omega}(\mathcal{T})|| \leq u$, then $||T^c - P_{\Omega'}(T^c)|| \leq u$, thus $\sigma_{rk}^c(P_{\Omega'}(T^c)) > \sigma_{rk}^c(T^c) - u > \sigma_{rk}^c - \gamma_{rk}^c \sigma_{rk}^c/2$, and $\sigma_{rk}^c(P_{\Omega'}(T^c))-\sigma_{rk+1}^c(T^c) = \sigma_{rk}^c(P_{\Omega'}(T^c)) - \sigma_{rk}^c + \gamma_{rk}^c \sigma_{rk}^c > \gamma_{rk}^c \sigma_{rk}^c /2$.

  By Davis-Kahan $\sin \theta$-theorem \cite{DK} and combining Definition \ref{def:tensor_spectral}, we have that
  \begin{equation}\label{sin_inequelity}
  \begin{split}
  ||\mathcal{W}^\dagger * \mathcal{A}|| &= ||W^{c\dagger} A^{c}|| = \sin \theta_{rk} (U^{c}, A^{c}) \leq \frac{||T^c - P_{\Omega'}(T^c)||}{\sigma_{rk}^c(P_{\Omega'}(T^c))-\sigma_{rk+1}^c(T^c)}  \\
  &\leq \frac{u}{\sigma_{rk}^c(P_{\Omega'}(T^c))-\sigma_{rk+1}^c(T^c)}
  \leq \frac{2u}{\gamma_{rk}^c \sigma_{rk}^c} = \frac{1}{16 \sqrt{r}}.
  \end{split}
  \end{equation}

  Note that the $rk$-th singular value $\sigma_{rk}^c$ of the circular matrix $T^c$ equals to that (i.e.,  $\overline{\sigma}_{rk}$) of the block diagonal matrix $\overline{\mathcal{T}}$. Then the probability formula becomes $p_0 \geq \frac{6144r^2 \mu(\mathcal{U}) (||\mathcal{T}||_F / \overline{\gamma}_{rk} \overline{\sigma}_{rk})^2 \log n}{n} + \frac{64r^{3/2} \mu(\mathcal{U}) (||\mathcal{T}||_F / \overline{\gamma}_{rk} \overline{\sigma}_{rk}) \log n}{n}$.
  \end{proof}

  \begin{lemma}\label{Trancation}
  Let $\mathcal{A} \in \mathbb{R}^{n \times r \times k}$ be any orthonormal basis with $\mathcal{A} \leq \mu$. Then, for a random orthonormal tensor $\mathcal{O} \in \mathbb{R}^{r \times r \times k}$, we have $\mathbb{P}\left\{ \max_{i,j} || [\mathcal{A} * \mathcal{O}](i,j,:)||_F > \sqrt{8\mu\log n/n} \right\} \leq \frac{1}{n^2}$.
  \end{lemma}
  \begin{proof}
  Consider a tube $Z = [\mathcal{A} * \mathcal{O}](i,j,:)$. Note that $||Z||_F$ is distributed like a coordinate of a random vector $\mathbb{R}^r$ of norm at most $\sqrt{\mu r/n}$. By measure concentration, we obtain:
  \begin{equation}
  \mathbb{P}\left\{ ||Z||_F > \epsilon  \sqrt{\mu r/n} \right\} \leq 4 \exp\{- \epsilon^2 r/2\}.
  \end{equation}
  This follows from Levy's Lemma \cite{Mat}, as in \cite{Hard2014FOCS}. Set $\epsilon = \sqrt{8 \log n/r}$, then
  \begin{equation}
  \mathbb{P}\left\{ ||Z||_F > \sqrt{8\mu\log n/n} \right\} \leq 4 \exp\{-4 \log n\} = 4n^{-4}.
  \end{equation}
  Taking a union bound over all $nr \leq n^2/4$ tubes of $\mathcal{W} * \mathcal{O}$, we have that with probability $1 - 1/n^2$,
  \begin{equation}
  \max_{i,j} || [\mathcal{A} * \mathcal{O}](i,j,:)||_F \leq \sqrt{8\mu\log n/n}.
  \end{equation}
  \end{proof}

  \begin{lemma}\label{lemma:incoherence}
  Assume that $||\mathcal{W}^\dagger * \mathcal{A}|| \leq \frac{1}{16 \sqrt{r}}$. Then with probability $1 - 1/n^2$, we have $||\mathcal{W}^\dagger * \mathcal{X}||_F \leq 1/4 $ and $\mu(\mathcal{X}) \leq 32 \mu(\mathcal{U}) \log n$.
  \end{lemma}
  \begin{proof}
  Assume that $||\mathcal{W}^\dagger * \mathcal{A}|| \leq \frac{1}{16 \sqrt{r}}$, then there exists an orthonormal transformation $\mathcal{Q} \in \mathbb{R}^{r \times r \times k}$ such that $||\mathcal{U} * \mathcal{Q} - \mathcal{A}||_F \leq 1/16$. Because of the following three facts: $||\mathcal{U} * \mathcal{Q} - \mathcal{A}||_F \leq \left( \sum_{i=1}^{r} ||\sigma_i||_F^2 \right)^{1/2}$ where $\sigma_i$ denotes the $i$-th eigentube of $(\mathcal{U} * \mathcal{Q} - \mathcal{A})$, $||\sigma_1||_F \geq ||\sigma_2||_F \geq ... \geq ||\sigma_r||_F$, and $||\sigma_1||_F = ||\mathcal{W}^\dagger * \mathcal{A}|| \leq \frac{1}{16 \sqrt{r}}$.

  Since $\mu(\mathcal{U} * \mathcal{Q}) = \mu(\mathcal{U}) \leq \mu$ (the orthonormal transformation $\mathcal{Q}$ does not change the incoherence of $\mathcal{U}$). Therefore, $\mathcal{A}$ is close in Frobenius norm to an orthonormal basis of small coherence. However, it is possible that some tubes of $\mu(\mathcal{U} * \mathcal{Q})$ have Frobenius norm as large as $\sqrt{\mu r/n}$. Rotating $\mu(\mathcal{U} * \mathcal{Q})$ by a random rotation $\mathcal{O}$, Lemma \ref{Trancation} asserts that with probability $1- 1/n^2$, $||[\mathcal{U} * \mathcal{Q} * \mathcal{O}](i,j,:)||_F \leq \mu'= \sqrt{8\mu\log n/n}$, for all $i,j$.
  Moreover, because a rotation does not increase Frobenius norm, then we have $||\mathcal{U} * \mathcal{Q} * \mathcal{O} - \mathcal{A} * \mathcal{O} ||_F \leq 1/16$. Truncating the tubes of $\mathcal{A} * \mathcal{O}$ that has Frobenius norm larger than $\mu'$ to $\mu'$ can therefore only decrease the distance in Frobenius norm to $\mathcal{U} * \mathcal{Q} * \mathcal{O}$, hence,  $||\mathcal{U} * \mathcal{Q} * \mathcal{O} - \mathcal{Z}' ||_F \leq 1/16$.

  Since truncation is a projection onto the set $\{\mathcal{B}: ||\mathcal{B}(i,j,:)||_F \leq \mu' \}$ with respect to Frobenius norm, we have:
  \begin{equation}
  ||\mathcal{A}*\mathcal{O} - \mathcal{Z}'||_F \leq ||\mathcal{U}*\mathcal{Q}*\mathcal{O} - \mathcal{Z}' ||_F \leq \frac{1}{16}.
  \end{equation}
  We can write $\mathcal{X} = \mathcal{Z}'*\mathcal{R}^{-1}$ where $\mathcal{R}$ is an invertible linear transformation with the same eigentubes as $\mathcal{Z}'$ and thus satisfies
  \begin{equation}
  ||\mathcal{R}^{-1}|| = \frac{1}{||\sigma_{1}(\mathcal{Z}')||_F} \leq \frac{1}{||\sigma_1(\mathcal{A}*\mathcal{O})||_F - ||\sigma_1(\mathcal{A}*\mathcal{O} - \mathcal{Z}')||_F} \leq \frac{1}{1- 1/16} \leq 2.
  \end{equation}
  Therefore,
  \begin{equation}
  ||\dot{e}_i^\dagger * \mathcal{X}|| = ||\dot{e}_i^\dagger * \mathcal{T}*\mathcal{R}^{-1}|| \leq ||\dot{e}_i^\dagger * \mathcal{T}|| ||\mathcal{R}^{-1}|| \leq 2||\dot{e}_i^\dagger * \mathcal{T}|| \leq 2\sqrt{8r\mu(\mathcal{U})\log n/n}.
  \end{equation}
  Hence,
  \begin{equation}
  \mu(\mathcal{X}) \leq \frac{n}{r} \frac{32r\mu(\mathcal{U})\log n}{n} \leq 32 \mu(\mathcal{U}) \log n.
  \end{equation}
  \begin{equation}
  \begin{split}
  ||\mathcal{W}^\dagger * \mathcal{X}||_F &= ||\mathcal{W}^\dagger  * \mathcal{T}*\mathcal{R}^{-1}||_F \leq ||\mathcal{W}^\dagger  * \mathcal{T}||_F ||\mathcal{R}^{-1}|| \leq 2 ||\mathcal{W}^\dagger  * \mathcal{T}||_F \\
  &\leq 2||\mathcal{W}^\dagger * \mathcal{A} * \mathcal{O}||_F + 2||\mathcal{A} * \mathcal{O} - \mathcal{T}||_F \leq 2||\mathcal{W}^\dagger * \mathcal{A}||_F + \frac{1}{8} \leq \frac{1}{4}.
  \end{split}
  \end{equation}

  \end{proof}

\subsection{Tensor Least Squares Minimization}
\label{LS_to_NSI}

  Alg. \ref{alg_LS_update} describes a tensor least squares minimization update step, specialized to the case of a symmetric square tensor. Our goal in this section is to express the tensor least squares minimization update step as $\mathcal{Y} = \mathcal{T} * \mathcal{X} + \mathcal{G}$, then we will be able to apply our convergence analysis of the noisy tensor-column subspace iteration in Appendix \ref{sec:noisy_subspace_iteration}. This syntactic transformation is given in Appendix \ref{LS_to_NSI} that is followed by a bound on the norm of the noise term $\mathcal{G}$ in Appendix \ref{bound_G}. With this, we prove in Appendix \ref{sec:median_LS} that the element-wise median process (in Alg. \ref{alg_median_LS_update}) on $t= O(\log n)$ tensor least squares minimizations will result in a much tighter bound of the noise term $\underline{\mathcal{G}}$ that is the average of those $t= O(\log n)$ copies of $\mathcal{G}$.

\subsubsection{From Alternating Least Squares to Noisy Tensor-Column Subspace Iteration}

  We first show that the tensor completion can be analyzed in its circular form in Lemma \ref{LS_circular_gradient}, then give an optimality condition in the circular form that the optimizer $\mathcal{Y}$ satisfies a set of linear equations in Lemma \ref{optimal_condition}. With these constraints, we express the tensor least squares minimization update step as $\mathcal{Y} = \mathcal{T} * \mathcal{X} + \mathcal{G}$ in Lemma \ref{lemma:least_squares}.

  \begin{lemma}\label{LS_circular_gradient}
  The function $f(\mathcal{Y}) = ||\mathcal{P}_{\Omega} (\mathcal{T} - \mathcal{X} * \mathcal{Y}^\dagger)||_F^2$ has an injective mapping to the circular form $f(Y^{c}) =\frac{1}{\sqrt{k}} ||\mathcal{P}_{\Omega'} (T^{c} - X^{c} Y^{c\dagger})||_F^2$. (Note that in Section \ref{whydifferent} we show that those two objective functions are different.)
  \end{lemma}

  \begin{proof}
  For the least squares update $\mathcal{Y} = \argmin_{\mathcal{Y} \in \mathbb{R}^{n \times r \times k} } ||\mathcal{P}_{\Omega} (\mathcal{T} - \mathcal{X} * \mathcal{Y}^\dagger)||_F^2$, it is equivalent to the following optimization problem:
   \begin{equation}\label{LS_equal}
   \min_{\mathcal{Y}}~~ ||\mathcal{G}||_F^2,~~~~
   {\rm s.t.}~~  \mathcal{P}_{\Omega}(\mathcal{T}) = \mathcal{P}_{\Omega}( \mathcal{X} * \mathcal{Y}^\dagger + \mathcal{G}),
   \end{equation}
   where $\mathcal{G}$ is a noise term. The circular form of (\ref{LS_equal}) is:
   \begin{equation}
   \min_{Y}~~ \frac{1}{k}||G^{c}||_F^2,~~~~
   {\rm s.t.}~~  P_{\Omega'} T^{c} = P_{\Omega'} X^{c} Y^\dagger + P_{\Omega'} G^{c}.
   \end{equation}
   Since the addition/subtraction and inverse operations are closed in the circulant algebra \cite{Gleich2013}, then $Y^{\dag} = (P^{\Omega'} X^{c})^{-1} ( P_{\Omega'} T^{c} -  P_{\Omega'} G^{c} )$ is also circular.

   We can see that $f(\mathcal{Y})$ implies $f(Y^{c})$ while the opposite direction does not hold as shown in Section \ref{whydifferent}. Therefore, this mapping is injective.
  \end{proof}

  \begin{lemma}\label{optimal_condition}
  (Optimality Condition). Let $P_i: \mathbb{R}^{nk} \rightarrow \mathbb{R}^{nk}$ be the linear projection onto the coordinates in $\Omega_i' = \{j:(i,j) \in \Omega'\}$ scaled by $p^{-1} = (nk)^2/(\mathbb{E}|\Omega'|)$, i.e., $P_i = p^{-1} \sum_{j \in \Omega_i'} e_j e_j^\dagger$ where $e_i, e_j$ are the standard vector bases and $e_j^\dagger$ is the corresponding row basis. Further, define the matrix $B_i \in \mathbb{R}^{rk \times rk}$ as $B_i = X^{c \dag} P_i X^c$ (note that $B_i$ is invertible as shown in \cite{Hard2014FOCS}). Then, for every $i \in [nk]$, the $i$-th row of $Y^c$ satisfies $e_i^\dagger Y^c = e_i^\dag T^c P_i X^c B_i^{-1}$.
  \end{lemma}

  \begin{proof}
  By Lemma \ref{LS_circular_gradient}, we consider the circular objective function $f(Y^{c}) =\frac{1}{\sqrt{k}} ||\mathcal{P}^{\Omega'} (T^{c} - X^{c} Y^{c\dagger} )||_F^2$.
  For every $i \in [nk], ~j \in [rk]$, we have $\frac{\partial f}{\partial Y_{ij}^{c}} = -\frac{2}{\sqrt{k}} \sum_{s \in \Omega'_i} T^{c}_{is} X_{sj}^{c} + \frac{2}{\sqrt{k}} \sum_{t=1}^{rk} Y^{c}_{it} \sum_{s \in \Omega'_i} X_{sj}^{c} X_{st}^{c}$. Therefore, we know that the optimal $\mathcal{Y}$ must satisfy $e_i^{\dag} T^{c} P_i X^{c} = e_i^{\dag} Y^{c} X^{c \dag} P_i X^{c} = e_i^{\dag} Y^{c}B_i$, hence, $e_i^{\dag} Y^{c} = e_i^{\dag} T^{c} P_i X^{c} B_i^{-1}$.
  \end{proof}

  \begin{lemma}\label{lemma:least_squares}
  Let $E^{c} = (I^{c} - X^{c} X^{c \dag}) U^{c}$, and assume that $\mathcal{T}$ is a noisy tensor (approximately $r$-tubal-rank) that is the superposition of an exact $r$-tubal-rank tensor $\mathcal{M}$ and a noisy tensor $\mathcal{N}$, i.e., $\mathcal{T} = \mathcal{M} + \mathcal{N}$. We express the least squares update as $\mathcal{Y} = \mathcal{T} * \mathcal{X} + \mathcal{G}$ where $\mathcal{G} = \mathcal{G}_{\mathcal{M}} + \mathcal{G}_{\mathcal{N}}$ and their circular matrices $G_{\mathcal{M}}^{c}$ and $G_{\mathcal{N}}^{c}$ satisfy that each row $i \in [nk]$, 
  we have the following expressions:
  \begin{equation}\label{LS_GM_GN}
  \begin{split}
  & e_i^{\dag} G_{\mathcal{M}}^{c} = e_i^{\dag} U^{c} \Lambda_{U^{c}}^{c} E^{c \dag} P_i X^{c} B_i^{-1}, \\
  & e_i^{\dag} G_{\mathcal{N}}^{c} = e_i^{\dag} (N^{c} P_i X^{c} B_i^{-1} - N^{c} X^{c}).
  \end{split}
  \end{equation}
  \end{lemma}

  \begin{proof}
  Since $\mathcal{T} = \mathcal{M} + \mathcal{N}$, we have $\mathcal{M} = \mathcal{U} * \Theta * \mathcal{U}^\dagger$ and $\mathcal{N} = (\mathcal{I} - \mathcal{U} * \mathcal{U}^\dagger) * \mathcal{T}$.
  By Lemma \ref{optimal_condition}, $e_i^\dag Y^c = e_i^\dag T^c P_i X^c B_i^{-1} = e_i^\dag M^c P_i X^c B_i^{-1} + e_i^\dag N^c P_i X^c B_i^{-1}$ since $Y^c = M^c + N^c$. Let $C_i = U^{c \dag} P_i X^{c}$ and $D = U^{c \dag} X^{c}$. We have:
  \begin{equation}
  \begin{split}
  e_i^\dag M^c P_i X^c B_i^{-1} = -e_i^{\dag} U^{c} \Lambda_{U^{c}} C_i B_i^{-1} &= e_i^{\dag}( U^{c} \Lambda_{U^{c}}D - U^{c} \Lambda_{U^{c}}(DB_i - C_i) B_i^{-1})\\
  &= e_i^{\dag} M^{c} X^{c}- e_i^{\dag} U^{c} \Lambda_{U^{c}}(DB_i - C_i) B_i^{-1}, \\
  C_i = U^{c \dag} P_i X^{c} = (X^{c} X^{c \dag} U^{c} + E^{c}) P_i X^{c} &= (U^{c \dag}X^{c}) X^{c \dag} P_i X^{c} + E^{c \dag} P_i X^{c} \\
  &= DB_i + E^{c \dag} P_i X^{c}.
  \end{split}
  \end{equation}
  Then, we have $e_i^\dag M^c P_i X^c B_i^{-1} = e_i^{\dag} M^{c} X^{c} -   e_i^{\dag} U^{c} \Lambda_{U^{c}} E^{c \dag} P_i X^{c}  B_i^{-1}$. From (\ref{LS_GM_GN}) we know that $e_i^\dag N^c P_i X^c B_i^{-1} \\= e_i^{\dag} N^{c}X^{c} + e_i^{\dag} G_{\mathcal{N}}^{c}$. Putting all together, we have that $Y^{c} = M^{c} X^{c} + G_{\mathcal{M}}^{c} + N^{c} X^{c} + G_{\mathcal{N}}^{c} = T^{c} X^{c} + G_{\mathcal{M}}^{c} + G_{\mathcal{N}}^{c} $, therefore, transforming it to the tensor form we have $\mathcal{Y} = \mathcal{T} * \mathcal{X} + \mathcal{G}$.
  \end{proof}

\subsubsection{Bound the Noisy Term $\mathcal{G}$}
\label{bound_G}

  We bound the spectral norm of each horizontal slice of $\mathcal{G}$. An intriguing fact is that the matrix $E^{c}$ appearing in the expression for the error terms (\ref{LS_GM_GN}) satisfies $|| E^{c} || = ||W^{c \dag} X^{c}||$, i.e., $||\mathcal{E}|| = || \mathcal{W}^{\dag} * \mathcal{X}||$ where $\mathcal{W}$ is defined in the t-SVD (Definition \ref{tsvd}). This allows us to obtain a bound through the quantity $||W^{c \dag} X^{c}||$ that equals to $|| \mathcal{W}^{\dag} * \mathcal{X}||$ according to Lemma \ref{lemma:spectral_norm}.

 \begin{lemma}\label{lemma:probability_p}
 Let $\delta \in (0,1)$. Assume that each entry is included in $\Omega$ independently with probability
 \begin{equation}\label{eqn:LS_probability}
 p_{\ell} \geq \frac{r\mu(\mathcal{X})\log{nk}}{\delta^2 n},
 \end{equation}
 then, for $\forall i \in [n]$, $\mathbb{P}\left\{ ||\dot{e}_i^\dagger * \mathcal{G}|| > \delta \left(|| \dot{e}_i^{\dag} * \mathcal{M}|| \cdot ||\mathcal{W}^{\dag} * \mathcal{X}|| +  || \dot{e}_i^{\dag} * \mathcal{N}|| \right) \right\} \leq \frac{1}{5}$.
 \end{lemma}

 \begin{proof}
 The set $\Omega'$ is exactly $k$ replicas of $\Omega$, and the probability $p_{\ell}$ of $\Omega$ corresponds to one replica of those $k$ replicas. According to Lemma 4.3 \cite{Hard2014FOCS}, if the probability $p_{\ell} \geq \frac{r\mu(X^c)\log{nk}}{\delta^2 nk}$ for each replica in $\Omega'$, then we have $\mathbb{P}\{ ||e_i^{\dag} G^{c}|| > \delta (|| e_i^{\dag} M^{c}|| \cdot ||W^{c \dag} X^{c}|| + || e_i^{\dag} N^{c}||) \} \leq \frac{1}{5}$, corresponding to $\mathbb{P}\left\{ ||\dot{e}_i^\dagger * \mathcal{G}|| > \delta \left(|| \dot{e}_i^{\dag} * \mathcal{M}|| \cdot ||\mathcal{W}^{\dag} * \mathcal{X}|| +  || \dot{e}_i^{\dag} * \mathcal{N}|| \right) \right\} \leq \frac{1}{5}$ (Lemma \ref{lemma:spectral_norm}). Note that $\mu(X^{c}) = k\mu(\mathcal{X})$, we have $p_{\ell} \geq \frac{r\mu(\mathcal{X})\log{nk}}{\delta^2 n}$.
 \end{proof}

\subsubsection{Median Tensor Least Squares Minimization}
\label{sec:median_LS}

  Here, we further analyze the element-wise median process in Alg. \ref{alg_median_LS_update}. Given the previous error bound of $\mathcal{G}$ in Lemma \ref{lemma:probability_p}, we can further derive a stronger concentration bound by taking the element-wise median of multiple independent samples of the error term.

 \begin{lemma}\label{lemma:bounded}
 Let $\Omega$ be a sample set in which each element is included independently with probability $p_+$. Let $\mathcal{G}_1,~\mathcal{G}_2,...,\mathcal{G}_t$ be i.i.d. copies of $\mathcal{G}$, and $\underline{\mathcal{G}} = \text{median}(\mathcal{G}_1,~\mathcal{G}_2,...,\mathcal{G}_t)$ be the element-wise median, and assume $p_+$ satisfy (\ref{eqn:LS_probability}). Then, for every $i \in [n]$,
 \begin{equation}
 \mathbb{P}\{
  ||\dot{e}_i^{\dag} * \underline{\mathcal{G}}|| > \delta \left(|| \dot{e}_i^{\dag} * \mathcal{M}|| \cdot ||\mathcal{W}^{\dag} * \mathcal{X}|| +  || \dot{e}_i^{\dag} * \mathcal{N}|| \right) \} \leq \exp(-\Omega(t)),
 \end{equation}
 where $\Omega(t)$ denotes some polynomial of $t$.
 \end{lemma}

  \begin{proof}
  For each $i \in [n]$, let $g_1,g_2,...,g_t \in \mathbb{R}^{1 \times r \times k}$ denote the $i$-th horizontal slice of $\mathcal{G}_1, \mathcal{G}_2, ..., \mathcal{G}_t$. Let $S = \{ j \in [t]: ||g_j|| \leq B \}$ where $B = \frac{\delta}{4} \left(|| \dot{e}_i^{\dag} * \mathcal{M}|| \cdot ||\mathcal{W}^{\dag} * \mathcal{X}|| +  || \dot{e}_i^{\dag} * \mathcal{N}|| \right)$. Applying Lemma \ref{lemma:probability_p} with error parameter $\frac{\delta}{4}$, we have $\mathbb{E}|S| \geq 4t/5$ with $g_j$ being drawn independently. Then we apply a Chernoff bound to argue that $\mathbb{P}(|S| > 2t/3) \geq 1 - \exp(-\Omega(t))$.

  Fixing a coordinate $s \in [r]$. By the median property we have $|\{j£º ||g_j(1,s,:)||_F^2 \geq ||\underline{g}_j||_F^2 \}| \geq t/2$. Since $|S| > 2t/3$, we know that at least $t/3$ horizontal slices  with $j \in S$ have $||g_j(1,s,:)||_F^2 \geq ||\underline{g}_j||_F^2$. Therefore, the average value of $||g_j(1,s,:)||_F^2$ over all $j \in [S]$ much be at least $\frac{t||\underline{g}_j||_F^2 }{3|S|} \geq ||\underline{g}_j||_F^2/3$. This means that the average of $||g_j||_F^2$ over all $j \in [S]$ much be at least $||\underline{g}||_F^2/3$. On the other hand, we also know that the average squared Frobenius norm in $S$ is at most $B^2$ by the definition of $S$. Then, the lemma is proved.
 \end{proof}

  We then provide a strong concentration bound for the median of multiple independent solutions to the tensor least squares minimization step.

  \begin{lemma}
  Let $\Omega$ be a sample set in which each element is included independently with probability $p_{+} \geq \frac{r\mu(\mathcal{X})\log{nk}}{\delta^2 n}$. Let $\mathcal{Y} \leftarrow \text{MedianLS-Y}(\mathcal{P}_{\Omega_{\ell}}(\mathcal{T}), \Omega, \mathcal{X}, r)$. Then, we have with probability $1- 1/n^3$ that $\underline{\mathcal{Y}} = \mathcal{T}*\mathcal{X} + \underline{\mathcal{G}}$ with $\underline{\mathcal{G}}$ satisfying: $||\dot{e}_i^{\dag} * \underline{\mathcal{G}}|| \leq  \delta \left(|| \dot{e}_i^{\dag} * \mathcal{M}|| \cdot ||\mathcal{W}^{\dag} * \mathcal{X}|| +  || \dot{e}_i^{\dag} * \mathcal{N}|| \right)$.
  \end{lemma}

  \begin{proof}
  Using  the $\text{Split}(\Omega, t)$ process, we know that the sample sets $\Omega_1,... \Omega_j, ..., \Omega_t$ are independent and each set $\Omega_j$ includes each element with probability at least $p_+ / t$. The output satisfies $\underline{\mathcal{Y}} = \text{median}(\mathcal{Y}_1, ...,\mathcal{Y}_j, ..., \mathcal{Y}_t)$, where $ \mathcal{Y}_j = \mathcal{T} * \mathcal{X} + \mathcal{G}_j$. Then $\text{median}(\mathcal{Y}_1, ...,\mathcal{Y}_j, ..., \mathcal{Y}_t) = \mathcal{T} * \mathcal{X} + \underline{\mathcal{G}}$.

  Therefore, apply Lemma \ref{lemma:bounded} combining the fact $t = O(\log n)$, we take a union bound over all $n$ horizontal slices of $\mathcal{G}$ to conclude this lemma.
  \end{proof}

\subsection{Convergence of Noisy Tensor-Column Subspace Iteration}\label{sec:noisy_subspace_iteration}

  \begin{algorithm}[h]
  \caption{Noisy Tensor-Column Subspace Iteration}
  \label{alg_noisy_subspace_iteration}
  \begin{algorithmic}
  \STATE \textbf{Input}: Tensor $\mathcal{T} \in \mathbb{R}^{n \times n \times k}$, number of iterations $L$, target dimension $r$
  \STATE Let $\mathcal{X}_0 \in \mathbb{R}^{n \times r \times k}$ be an orthonormal tensor.\\
  \STATE For $\ell=1$ to $L$ \\
  \STATE ~~~~Let $\mathcal{G}_{\ell} \in \mathbb{R}^{n \times r \times k}$ be an arbitrary perturbation.\\
  \STATE ~~~~$\mathcal{Z}_{\ell} \leftarrow \mathcal{T} * \mathcal{X}_{\ell -1} + \mathcal{G}_{\ell}$.\\
  \STATE ~~~~$\mathcal{X}_{\ell} \leftarrow GS(\mathcal{Z}_{\ell})$.\\
  \STATE \textbf{Output}: Tensor $\mathcal{X}_{\ell} \in \mathbb{R}^{n \times r \times k}$.
  \end{algorithmic}
  \vspace{-2pt}
  \end{algorithm}

  Alg. \ref{alg_noisy_subspace_iteration} describes our noisy tensor-column subspace iteration, where $GS(\mathcal{Y}_l)$ denotes the Gram-Schimidt process which orthonormalizes the lateral slices of the tensor $\mathcal{Z}_{\ell}$. The detailed steps of the Gram-Schimidt process for third-order tubal-rank tensor is given in \cite{Kilmer2013}). Note that Alg. \ref{alg_noisy_subspace_iteration} is different from the recently proposed power method \cite{Gleich2013}: 1) we simultaneously compute multiple top-$k$ eigenslices while the power method considered only the top-$1$ eigenslice, 2) in each iteration $\ell$, the computation in Alg. \ref{alg_noisy_subspace_iteration} is perturbed by a tensor $\mathcal{G}_{\ell}$ which can be adversarially and adaptively chosen, and 3) the Gram-Schimidt process for third-order tubal-rank tensor is introduced to manipulate $\mathcal{G}_{\ell}$.

  For the matrix case, an important observation of \cite{Jain2013STOC,Hardt2014COLT,Hard2014FOCS} is that the least squares minimization can be analyzed as a noisy update step of the well known {\em subspace iteration} (or power method). Therefore, the convergence of the alternating minimization iteration is equivalent to the convergence of the noisy subspace iteration. The corresponding convergence analysis exploits the tangent function of the largest principle angle between the subspace $U$ spanned by the first $r$ singular vectors of the input matrix and the $r$-dimensional space spanned by the columns of the iterate $X_\ell$.

  To show the convergence results of noisy tensor column subspace iteration, we use the largest principal angle between two tensor-column subspaces as the potential function.
  Borrowing idea from \cite{Gleich2013}, we show that the noisy tensor-column subspace iteration can be transformed to {\em $k$ parallel} noisy subspace iterations in the frequency domain.

  \begin{lemma}\label{lemma:noisy_tensor}
  The noisy tensor-columns subspace iteration in Alg. \ref{alg_noisy_subspace_iteration} converges at a geometric rate\footnote{We do not explicitly state the convergence rate because the one in \cite{Hard2014FOCS} depends on the condition number, while we encounter a constant convergence rate from our experiments.}.
  \end{lemma}
  \begin{proof}
  The key iterative operations in Alg. \ref{alg_noisy_subspace_iteration} are
  \begin{equation}\label{noisy_subspace_iteration}
  \begin{split}
  \mathcal{Z}_{\ell} \leftarrow &\mathcal{T} * \mathcal{X}_{\ell -1} + \mathcal{G}_{\ell},\\
  \mathcal{X}_{\ell} \leftarrow & \text{GS}(\mathcal{Z}_{\ell}).
  \end{split}
  \end{equation}
  Introducing the $\text{cft}(\cdot)$ opertion in (\ref{eqn_to_freq}), we know that (\ref{noisy_subspace_iteration}) be represented as follows:
  \begin{equation}
  \text{cft}(\mathcal{Y}_{\ell}) \leftarrow \text{cft}(\mathcal{X}) \text{cft}(\mathcal{X}_{\ell - 1}) + \text{cft}(\mathcal{G}_{\ell}).
  \end{equation}
  This implies that (\ref{noisy_subspace_iteration}) equals to $k$ parallel standard noisy subspace iteration in the frequency domain. Therefore, combining the convergence results of \cite{Hard2014FOCS} that  noisy subspace iteration converges at a geometric rate, our noisy tensor-columns subspace iteration in Alg. \ref{alg_noisy_subspace_iteration} will also converge at a geometric rate.

  \end{proof}

  In the following, we first provide the definitions of principal angles and corresponding inequalities for the matrix case \cite{Hard2014FOCS}. Then, we need to establish explicit inequalities along the iteration process, so that we will be able to bound the recovery error of Alg. \ref{alg_AM}.

  \begin{definition}
  \textbf{Largest principal angle}. Let $X,Y \in \mathbb{R}^{n \times r}$ be orthonormal bases for subspaces $\mathcal{S}_X, \mathcal{S}_Y$, respectively. Then, the sine of the {\em largest principal angle} between $\mathcal{S}_X$ and $\mathcal{S}_Y$ is defined as $\sin \theta (\mathcal{S}_X, \mathcal{S}_Y) \doteq || (I - XX^{\dagger}) Y||$.
  \end{definition}

  \begin{lemma}\label{matrix_local_convergence}
  (Matrix Local Convergence) Let $0 \leq \epsilon \leq 1/4$, $\Delta = \max_{1 \leq \ell \leq L} ||G_{\ell}||$, and $\gamma_r = 1 - \sigma_{r+1} / \sigma_r$.
  Assume that $|| W^\dagger X_0 || \leq 1/4$ and $\sigma_r \geq 8 \Delta / \gamma_r \epsilon$. Then
  \begin{equation}
  || W^\dagger X_{\ell} || \leq \max \{ \epsilon, 2 || W^\dagger X_0 ||  \exp(-\gamma_r \ell/2) \}
  \end{equation}
  \end{lemma}

  Similarly we can prove the following lemma for our tensor case.
  \begin{lemma}\label{lemma:tensor_local_convergence}
  (Tensor Local Convergence) Let $0 \leq \epsilon \leq 1/4$, $\Delta = \max_{1 \leq \ell \leq L} ||\mathcal{G}_{\ell}||$, and $\gamma_{rk} = 1 - \overline{\sigma}_{rk+1} / \overline{\sigma}_{rk}$.
  Assume that $|| \mathcal{W}^\dagger * \mathcal{X}_0 || \leq 1/4$ and $\overline{\sigma}_{rk} \geq 8 \Delta / \gamma_{rk} \epsilon$. Then
  \begin{equation}\label{epsilon_4}
  || \mathcal{W}^\dagger * \mathcal{X}_{\ell} || \leq \max \{ \epsilon, 2 || \mathcal{W}^\dagger * \mathcal{X}_0 || \cdot  \exp(-\gamma_{rk} {\ell}/2) \}
  \end{equation}
  \end{lemma}

  \begin{proof}
  According to Lemma \ref{lemma:spectral_norm}, we have
  \begin{equation}\label{eqn:tensor_local_convergence}
  ||\mathcal{W}^{\dagger} * \mathcal{X}_{\ell}|| = || \overline{\mathcal{W}^{\dagger} * \mathcal{X}_{\ell}  }|| = || \overline{\mathcal{W}^{\dagger}} ~ \overline{\mathcal{X}_{\ell}  }||.
  \end{equation}
  This means that the largest principle angle between $\mathcal{W}^{\dagger}$ and $\mathcal{X}_{\ell}$ equals to that of these two tensor-column subspaces in the frequency domain.

  Note that $|| \mathcal{W}^\dagger * \mathcal{X}_0 || \leq 1/4$ will be provided in Lemma \ref{proof:initialization}, thus $|| \overline{\mathcal{W}^\dagger}~ \overline{\mathcal{X}_0} || \leq 1/4$. Let $\Delta = \max_{1 \leq \ell \leq L} ||\mathcal{G}_{\ell}||$, $\gamma_{rk} = 1 - \overline{\sigma}_{rk+1} / \overline{\sigma}_{rk}$, and $\overline{\sigma}_{rk} \geq 8 \Delta / \gamma_{rk} \epsilon$. From Definition \ref{def:block-diagonal} and \ref{tsvd}, we know that a tensor $\mathcal{T}$ with tubal-rank $r$ has a corresponding block diagonal matrix $\overline{\mathcal{T}}$ with rank $rk$. Therefore, applying Lemma \ref{matrix_local_convergence} we get
  \begin{equation}
  \begin{split}
  || \overline{\mathcal{W}^{\dagger}} ~ \overline{\mathcal{X}_L}|| &\leq \max \{ \epsilon, 2 || \overline{\mathcal{W}^\dagger}~ \overline{\mathcal{X}_0} || \cdot   \exp(-\gamma_{rk} L/2) \}\\
  &= \max \{ \epsilon, 2 || \mathcal{W}^\dagger * \mathcal{X}_0 || \cdot  \exp(-\gamma_{rk} L/2) \}.
  \end{split}
  \end{equation}
  Combining with (\ref{eqn:tensor_local_convergence}), the lemma is proof.
  \end{proof}

  To prove the convergence of the noisy tensor-column subspace iteration, we show that the error term $||\mathcal{G}_{\ell}||$ decrease as $\ell$ increases and Alg. \ref{alg_noisy_subspace_iteration} starts to converge. We define the following condition as a convergence bound for this type of shrinking error.

  \begin{definition}\label{def:tensor_admissible}
  (Tensor $\epsilon$-Admissible). Let $\gamma_{rk} = 1 - \overline{\sigma}_{rk+1} / \overline{\sigma}_{rk}$. We say that the pair of tensors $(\mathcal{X}_{\ell-1}, \mathcal{G}_{\ell})$ is $\epsilon$-admissible for noisy tensor-column subspace iteration if
  \begin{equation}
  ||\mathcal{G}_{\ell}|| \leq \frac{1}{32} \gamma_{rk} \overline{\sigma}_{rk} ||\mathcal{W}^{\dagger} * \mathcal{X}_{\ell-1}|| + \frac{\epsilon}{32}\gamma_{rk} \overline{\sigma_{rk}}.
  \end{equation}
  \end{definition}

  One can say that a sequence of tensors $\{(\mathcal{X}_{\ell-1}, \mathcal{G}_{\ell})\}$ is $\epsilon$-admissible for noisy tensor-column subspace iteration if each element of this sequence is $\epsilon$-admissible. In the following we will use the notation $\{(\mathcal{G}_{\ell}$ as a shorthand for $\{(\mathcal{X}_{\ell-1}, \mathcal{G}_{\ell})\}_{\ell=1}^{L}$.

  With Lemma \ref{lemma:tensor_local_convergence} and Definition \ref{def:tensor_admissible}, we are able to get the following convergence guarantee for admissible noise tensors.

  \begin{theorem}\label{theorem:addmissible_noisy_tensors}
  Let $\gamma_{rk} = 1 - \overline{\sigma}_{rk+1} / \overline{\sigma}_{rk}$, and $\epsilon \leq 1/2$. Assume that the sequence of noisy tensors $\{\mathcal{G}_{\ell}\}$ is $(\epsilon/2)$-admissible for the noisy tensor-columns subspace iteration and that $|| \mathcal{W}^\dagger * \mathcal{X}_0 || \leq 1/4$. Then, we have $|| \mathcal{W}^\dagger * \mathcal{X}_L || \leq  \epsilon$ for any $L \geq 4 \gamma_{rk}^{-1} \log(1/\epsilon)$.
  \end{theorem}

  \begin{proof}
  We prove by induction that for every integer $t \geq 0$ after $L_t = 4t\gamma_{rk}^{-1}$ steps, we have $|| \mathcal{W}^\dagger * \mathcal{X}_{L_t} || \leq  \max\{2^{-(t+1)}, \epsilon\}$. For the base case $t=0$, the lemma holds because of the assumption that $|| \mathcal{W}^\dagger * \mathcal{X}_0 || \leq 1/4$. For $t \geq 1$, we assume that $|| \mathcal{W}^\dagger * \mathcal{X}_{L_t} || \leq  \max\{2^{-(t+1)}, \epsilon\}$. Apply Lemma \ref{lemma:tensor_local_convergence} with $\mathcal{X}_0 = \mathcal{X}_{L_t}$, error parameter $\max\{2^{-t+2}, \epsilon\}$ and $L = L_{t+1} - L_t = 4/\gamma_{rk}$. The conditions of the lemma are satisfied due to the assumption that $\{\mathcal{G}_{\ell}\}$ is $\epsilon/2$-admissible. Therefore, we get
  \begin{equation}
  || \mathcal{W}^\dagger * \mathcal{X}_{L_{t+1}} || \leq \max \{ \epsilon, 2  \max\{2^{-(t+1)}, \epsilon\} \cdot  \exp(\gamma_{rk} (L_{t+1} - L_t )/2) \} \leq \max\{\epsilon, 2^{-(t+2)}\}
  \end{equation}

  \end{proof}


\subsection{Incoherence via the SmoothQR Procedure}

  As a requirement for our proof in Appendix \ref{proof:theorem}, we need to show that each intermediate solution $\mathcal{Y}_{\ell}$ (accordingly $\mathcal{X}_{\ell}$) has small coherence. Lemma \ref{lemma:incoherence_bound} states that applying the SmoothQR factorization (in Alg. \ref{alg_smoothQR}) on $\mathcal{Y}_{\ell}$ will return a tensor $\mathcal{Z}$ satisfying this coherence requirement. Note that before orthonormalizing $\mathcal{Y}_{\ell}$, a small Gaussian perturbation $\mathcal{H}$ is added to $\mathcal{Y}_{\ell}$. There exists such noisy term that will cause little effect as long as its norm is bounded by that of $\mathcal{G}_{\ell}$.

  \begin{lemma}\label{V_1}
  Let $\mathcal{G} \in \mathbb{R}^{n \times r \times k}$ be any tensor with $||\mathcal{G}|| \leq 1$, $\mathcal{W} \in \mathbb{R}^{n \times (n-r) \times k}$ be a $(n - r)$ dimensional tensor-column subspace with orthogonal projection $\mathcal{P}_{\mathcal{W}}$, and $\mathcal{H} \in \mathbb{R}^{n \times r \times k} \sim \mathcal{N}(0, \tau^2/n)$ be a random Gaussian tensor. Assume that $r = o(n/\log n)$ where $o(n)$ denotes an order that is lower than $n$. Then, with probability $1 - \exp (-\Omega(n))$, we have $\sigma_{rk}(\overline{\mathcal{P}_{\mathcal{W}}(\mathcal{G} + \mathcal{H})}) \geq \Omega(\tau)$.
  \end{lemma}

  \begin{proof}
  Consider a tensor $\mathcal{X} \in \mathbb{R}^{r \times 1 \times k}$ with $||\mathcal{X}||_F = 1$, we have
  \begin{equation}
  ||\mathcal{P}_{\mathcal{W}}(\mathcal{G} + \mathcal{H}) * \mathcal{X}||^2 > ||\mathcal{P}_{\mathcal{W}}*\mathcal{H}*\mathcal{X}||^2 - |\langle \mathcal{P}_{\mathcal{W}}*\mathcal{G}*\mathcal{X}, \mathcal{P}_{\mathcal{W}}*\mathcal{H}*\mathcal{X}\rangle |.
  \end{equation}
  Note that $g = \mathcal{H} * \mathcal{X} \in \mathbb{R}^{n \times 1 \times k}$ follows the distribution $N(0, \tau^2/n)^{n \times 1 \times k}$, $y = \mathcal{P}_{\mathcal{W}}*\mathcal{G} * \mathcal{X}$ has spectral norm at most $1$, and $\mathcal{W}$ is a $n-k$ dimensional tensor-column subspace, and $h = \mathcal{P}_{\mathcal{W}} * \mathcal{H} * \mathcal{X}$ follows the distribution $N(0, \tau^2/n)^{n \times 1 \times k}$. Then, we need to lower bound $||h||^2 - | \langle y, h\rangle|$. Since $\mathbb{E}||h||^2 > \tau^2/2$, by standard concentration bounds for the norm of a Gaussian variable, we get
  \begin{equation}
  \mathbb{P}\{ ||h||^2 \leq \tau^2/4 \} \leq \exp{\Omega(n)}.
  \end{equation}
  On the other hand, $\langle y,h\rangle$ is distributed like a one-dimensional Gaussian variable of variance at most $\tau^2/n$. By Gaussian tail bounds, $\mathbb{P}\{ \langle y, h\rangle^2 > \tau^2/8 \} \leq \exp{-\Omega(n)}$. Therefore, with probability $1 - \exp{\Omega(n)}$, we have $||\mathcal{P}_{\mathcal{W}}(\mathcal{G} + \mathcal{H}) * \mathcal{X}|| > \Omega(\tau)$.

  Taking a union bound over a set of the unit sphere in $\mathbb{R}^{r \times 1 \times k}$ of size $\exp{O(r \log r)}$, we have that with probability $1 - \exp(O(r \log r)) \exp(-\Omega(n))$, $||\mathcal{P}_{\mathcal{W}}(\mathcal{G} + \mathcal{H}) * \mathcal{X}|| > \Omega(\tau)$ for all unit tensors $\mathcal{X} \in \mathbb{R}^{r \times 1 \times k}$, i.e., $\sigma_{rk}(\overline{\mathcal{P}_{\mathcal{W}}(\mathcal{G} + \mathcal{H})}) > \Omega(\tau)$.

  Note that $\exp{O(r \log r)} = \exp(o(n))$, hence this event occurs with probability $1 - \exp(\Omega(n))$.

  \end{proof}

  We introduce a variant of $\mu$-coherence, i.e., $\rho$-coherence, that applies to tensors rather than tensor-column subspaces.   The next lemma (Lemma \ref{lemma:incoherence_bound}) show that adding a Gaussian noise term leads to a bound on the coherence after applying the QR-factorization.

  \begin{definition}
  ($\rho$-coherence). Given a tensor $\mathcal{G} \in \mathbb{R}^{n \times r \times k}$ we let $\rho(\mathcal{G}) \doteq \frac{n}{r} ||\dot{e}_i^\dagger * \mathcal{G}||^2$.
  \end{definition}

  \begin{lemma}\label{lemma:incoherence_bound}
  Let $r = \Omega(n/\log n)$ and $\tau \in (0,1)$. Let $\mathcal{U} \in \mathbb{R}^{n \times r \times k}$ be an orthonormal tensor, and $\mathcal{G} \in \mathbb{R}^{n \times r \times k}$ be a tensor such that $||\mathcal{G}|| \leq 1$. Let $\mathcal{H} \sim \mathcal{N}(0, \tau^2/n)^{n \times r \times k}$ be a random Gaussian tensor. Then, with probability $1 - \exp (-\Omega(n)) - n^{-5}$, there exists an orthonormal tensor $\mathcal{Q} \in \mathbb{R}^{n \times 2r \times k}$ such that
  \begin{itemize}
  \item $\mathcal{R}(\mathcal{Q}) = \mathcal{R}([ \mathcal{U} ~|~ \mathcal{G} + \mathcal{H}])$ where $\mathcal{R}(\mathcal{Q})$ denotes the range of $\mathcal{Q}$;
  \item $\mu(\mathcal{Q}) \leq O(\frac{1}{\tau} (\rho(\mathcal{Q}) + \mu(\mathcal{U}) + \log n))$.
  \end{itemize}
  \end{lemma}

  \begin{proof}
  First, $\mathcal{R}([ \mathcal{U} ~|~ \mathcal{G} + \mathcal{H}]) = \mathcal{R}([ \mathcal{U} ~|~ (\mathcal{I} - \mathcal{U}*\mathcal{U}^\dagger) *(\mathcal{G} + \mathcal{H})])$. Let $\mathcal{B} = (\mathcal{I} - \mathcal{U}*\mathcal{U}^\dagger) *(\mathcal{G} + \mathcal{H})])$. Applying the QR-factorization to $[\mathcal{U}~|~\mathcal{B}]$, we can find two orthonormal tensors $\mathcal{Q}_1, \mathcal{Q}_2 \in \mathbb{R}^{n \times r \times k}$ such that $[\mathcal{Q}_1~|~\mathcal{Q}_2] = [\mathcal{U}~|~\mathcal{B}*\mathcal{R}^{-1}]$ where $\mathcal{R} \in \mathbb{R}^{r \times r \times k}$. Since $\mathcal{U}$ is already orthonormal, we can have $\mathcal{Q}_1 = \mathcal{U}$. Furthermore, the lateral slices of $\mathcal{B}$ are orthogonal to $\mathcal{U}$ and thus we apply the QR-factorization to $\mathcal{U}$ and $\mathcal{B}$ independently.

  Applying Lemma \ref{V_1} to the $(n - r)$-dimensional tensor-column subspace $\mathcal{U}^{\bot}$ and the tensor $\mathcal{G} + \mathcal{H}$, we get that with probability $1 - \exp(-\Omega(n))$,~ $\sigma_{rk}(\mathcal{B}) \geq \Omega(\tau)$. Assume that this hold in the following.

  We verify the second condition. We have
  \begin{equation}
  \frac{n}{r} ||\dot{e}_i^{\dagger}*\mathcal{Q}||^2 = \frac{n}{r} ||\dot{e}_i^{\dagger}*\mathcal{U}||^2 + \frac{n}{r} ||\dot{e}_i^{\dagger}*\mathcal{B}*\mathcal{R}^{-1}||^2 = \mu(\mathcal{U})+ \frac{n}{r} ||\dot{e}_i^{\dagger}*\mathcal{B}*\mathcal{R}^{-1}||^2.
  \end{equation}
  On the other hand, we also have
  \begin{equation}
  \frac{n}{r} ||\dot{e}_i^{\dagger}*\mathcal{B}*\mathcal{R}^{-1}||^2 \leq \frac{n}{r} ||\dot{e}_i^{\dagger}*\mathcal{B}||^2||\mathcal{R}^{-1}||^2 \leq O\left(\frac{n}{r\tau^2} ||\dot{e}_i^{\dagger}*\mathcal{B}||^2\right),
  \end{equation}
  where we used the fact that $||\mathcal{R}^{-1}|| = 1/\sigma_{rk}(\mathcal{R}) = O(1/\tau)$.

  Moreover, we have
  \begin{equation}
  \begin{split}
  \frac{n}{r} ||\dot{e}_i^{\dagger}*\mathcal{B}*\mathcal{R}^{-1}||^2 &\leq 2\frac{n}{r} ||\dot{e}_i^{\dagger} * (\mathcal{I} - \mathcal{U}*\mathcal{U}^\dagger)||^2 + 2 \rho((\mathcal{I} - \mathcal{U}*\mathcal{U}^\dagger) * \mathcal{H}) \\
  &\leq 2\rho(\mathcal{G}) + 2\rho(\mathcal{U}*\mathcal{U}^\dagger*\mathcal{G}) + 2 \rho((\mathcal{I} - \mathcal{U}*\mathcal{U}^\dagger) * \mathcal{H}).
  \end{split}
  \end{equation}
  Note that $\rho(\mathcal{U}*\mathcal{U}^\dagger*\mathcal{G}) \leq \mu(\mathcal{U}) ||\mathcal{U}^\dagger * \mathcal{G}||^2 \leq \mu(\mathcal{G})$.

  Combining the Lemma \ref{V_4} (in the following), we have $(\mathcal{I} - \mathcal{U}*\mathcal{U}^\dagger)*\mathcal{H} \leq O(\log n)$ with probability $1 - n^{-5}$. Summing up the probability concludes the lemma.
  \end{proof}

  \begin{lemma}\label{V_4}
  Let $\mathcal{P}$ be the projection onto an $(n-r)$-dimensional tensor-column subspace. Let $\mathcal{H} \sim N(0, 1/n)^{n \times r \times k}$. Then, $\rho(\mathcal{P}*\mathcal{H}) \leq O(\log n)$ with probability $1 - 1/n^5$.
  \end{lemma}
  \begin{proof}
  Let $\mathcal{P} = (\mathcal{I} - \mathcal{U}*\mathcal{U}^\dagger)$ for some $r$-dimensional tensor-column basis $\mathcal{U}$. Then,
  \begin{equation}
  \rho(\mathcal{P} * \mathcal{H}) \leq \rho(\mathcal{H}) + \rho(\mathcal{U}*\mathcal{U}^\dagger * \mathcal{H} ).
  \end{equation}
  Using concentration bounds for the norm of each horizontal slice of $\mathcal{H}$ and a union bound over all horizontal slices, it follows that
  \begin{equation}
  \begin{split}
  \rho(\mathcal{H}) &\leq O(\log n) ~~~\text{with probability}~~~1 - \frac{1}{2}n^{-5}, \\
  \rho(\mathcal{U}*\mathcal{U}^\dagger * \mathcal{H} ) &\leq \rho(\mathcal{U})||\mathcal{U}^\dagger * \mathcal{H}||^2.
  \end{split}
  \end{equation}
  Note that $\mathcal{U}^\dagger * \mathcal{H}$ is a Gaussian tensor following the distribution $N(0, 1/n)^{r \times r \times k}$, and its largest singular value satisfies $||\mathcal{U}^\dagger * \mathcal{H}||^2 \leq O(r\log n/n)$ with probability $1 - \frac{1}{2}n^{-5}$. Summing up the probability concludes the lemma.
  \end{proof}

  The next theorem states that when SmoothQR is called on an input of the form $\mathcal{T}*\mathcal{X} +\mathcal{G}$ with suitable parameters, the algorithm outputs a tensor of the form $\mathcal{X}' = \text{QR}(\mathcal{T}*\mathcal{X} +\mathcal{G} + \mathcal{H})$ whose coherence is bounded in terms of $\mathcal{G}$ and $\rho(\mathcal{G})$, and $\mathcal{H}$ satisfies a bound on its norm.

  \begin{theorem}\label{thorem_incoherence}
  Let $\tau > 0$, $r = \Omega(n / \log n)$, $\mathcal{G} \in \mathcal{R}^{n \times r \times k}$, and $\mathcal{X} \in \mathbb{R}^{n \times r \times k}$ be an orthonormal tensor such that $\upsilon \geq \max\{ ||\mathcal{G}||, ||\mathcal{N} * \mathcal{X}||\}$. There exist a constant $C > 0$, assume that
  \begin{equation}
  \mu \geq \frac{C}{\tau^2} \left( \mu(\mathcal{U}) + \frac{\rho(\mathcal{G}) + \rho(\mathcal{N}*\mathcal{X})}{\upsilon^2}  + \log n \right),
  \end{equation}
  then, for every $\epsilon \leq \tau \upsilon$ satisfying $\log (n / \epsilon) \leq n$ and every $\mu \leq n$ we have with probability $1 - O(n^{-4})$, the algorithm SmoothQR (Alg. \ref{alg_smoothQR}) terminates in $\log (n / \epsilon)$ steps and outputs $(\mathcal{X}', \mathcal{H})$ such that $\mu(\mathcal{X}') \leq \mu$ and $||\mathcal{H}|| \leq \tau \upsilon$.
  \end{theorem}

  \begin{proof}
  If the algorithm SmoothQR (Alg. \ref{alg_smoothQR}) terminates in an iteration where $\varsigma \leq \tau^2\upsilon^2/4$ (proved in Lemma \ref{V_5}), we claim that in this case, with probability $1 - \exp(-\Omega(n))$ we must have that $||\mathcal{H}|| \leq \tau \upsilon$. Assume that the algorithm SmoothQR (Alg. \ref{alg_smoothQR}) terminates in an iteration where $\varsigma \leq \tau^2\upsilon^2/r$, then the algorithm takes at most $t = O(\log(n/\epsilon)) \leq O(n)$ steps.

  Let $\mathcal{H}_1, ..., \mathcal{H}_t$ denote the random Gaussian tensors generated in each step. We claim that each of them satisfies $\mathcal{H} \leq \tau \upsilon$.  Note that for all $t$ we have $\mathbb{E}||\mathcal{H}_t||^2 \leq \tau^2 \upsilon^2/4$. The claim therefore follows directly from tail bounds for the Frobenius norm of Gaussian random tensors and holds with probabilities $ 1- \exp(-\Omega(n))$.

  \end{proof}

  \begin{lemma}\label{V_5}
  With probability $ 1 - O(n^{-4})$, Alg. \ref{alg_smoothQR} terminates in an iteration where $\varsigma \leq \tau^2\upsilon^2/4$.
  \end{lemma}

  \begin{proof}
  Consider the first iteration in which $\varsigma \leq \tau^2\upsilon^2/8$. Let us define $\mathcal{G}' = (\mathcal{N} * \mathcal{X} + \mathcal{G}) / 2\upsilon$.  Apply Lemma \ref{lemma:incoherence_bound} to the tensor $\mathcal{G}'$ which satisfies the required assumption that $||\mathcal{G}'|| \leq 1$.  Lemma \ref{lemma:incoherence_bound} states that with probability $ 1 - O(n^{-4})$, there is an orthonormal $n \times 2r \times k$ tensor $\mathcal{Q}$ such that
  \begin{equation}
  \begin{split}
  \mathcal{R}(\mathcal{Q}) &= \mathcal{R}([\mathcal{U} ~|~\mathcal{G}' + \mathcal{H}]) = \mathcal{R}([\mathcal{U} ~|~\mathcal{G} + \mathcal{N}*\mathcal{X} + \mathcal{H}]),\\
  \mu(\mathcal{Q}) &\leq O(\frac{1}{\tau^2}(\rho(\mathcal{G} + \mu(\mathcal{U} + \log n)))).
  \end{split}
  \end{equation}
  On one hand, we have
  \begin{equation}
  \mathcal{R}(\mathcal{X}') = \mathcal{R}(\mathcal{T}*\mathcal{X} + \mathcal{G} + \mathcal{H}) = \mathcal{R}(\mathcal{M}*\mathcal{X} + \mathcal{N}*\mathcal{X}  + \mathcal{G} +\mathcal{H}) \subset \mathcal{R}([ \mathcal{U}~|~\mathcal{N}*\mathcal{X}  + \mathcal{G} +\mathcal{H}]) = \mathcal{R}(\mathcal{W}),
  \end{equation}
  where we use the fact that $\mathcal{U}$ is an orthonormal basis for the range of $\mathcal{M}*\mathcal{X} = \mathcal{U}*\Theta *\mathcal{U}^{\dagger}*\mathcal{X}$. On the other hand, $\rho(\mathcal{G}') = O(\rho(\mathcal{\mathcal{G}/\upsilon}) + \rho(\mathcal{N}*\mathcal{X}/\upsilon'))$.

  Therefore, combining Lemma \ref{B_2} and the fact that $\text{dim}(\mathcal{Q}) \leq 2 \text{dim}(\mathcal{X}')$ where $\text{dim}(\cdot)$ denotes the dimension, we have $\mu (\mathcal{X}') \leq 2 \mu(\mathcal{Q}) \leq \mu $. This lemma is proved as long as $C$ is large enough.

  \end{proof}

  \begin{lemma}\label{B_2}
  Let $\mathcal{X}, \mathcal{Y}$ be $r$ and $r'$ dimensional tensor-column subspaces, respectively, such that $\mathcal{X} \subset \mathcal{Y}$. Then, we have $\mathcal{X} \leq \frac{r'}{r}\mu(\mathcal{Y})$.
  \end{lemma}

  \begin{proof}
  We know that $\mu(\mathcal{Y})$ is rotationally invariant. Therefore, without loss of generality we assume that $\mathcal{Y} = [\mathcal{X} ~|~\mathcal{X}']$ for some orthonormal tensor $\mathcal{X}'$. Here, we identify $\mathcal{X}$ and $\mathcal{Y}$
  with orthonormal bases. Therefore,
  \begin{equation}
  \mu(\mathcal{X}) = \frac{n}{r} \max\limits_{i\in[n]} ||\dot{e}_i^\dagger * \mathcal{X}||^2 \leq \frac{n}{r} \max\limits_{i\in[n]} \left( ||\dot{e}_i^\dagger * \mathcal{X}||^2 + ||\dot{e}_i^\dagger * \mathcal{X}'||^2 \right) = \frac{n}{r}\max\limits_{i\in[n]} ||\dot{e}_i^\dagger * \mathcal{Y}||  = \frac{r'}{r}\mu(\mathcal{Y}).
  \end{equation}

  \end{proof}


\end{document}